\documentclass[twoside,11pt]{article}

\usepackage{jmlr2e}
\usepackage{amsmath}
\usepackage{amssymb}
\usepackage{natbib}
\usepackage{dsfont}
\usepackage{tikz, float}
\usetikzlibrary{arrows}
\usepackage{pgfplots}
\usepgfplotslibrary{patchplots}
\usepackage{wrapfig}
\usepackage{xparse}
\usepackage{tabularx}
\usepackage{lipsum}
\usepackage[small]{caption}

\usepackage[capitalise]{cleveref}

\usepackage[textsize=tiny]{todonotes}
\newcommand{\todot}[2][]{\todo[size=\scriptsize,color=red!20!white,#1]{Tor: #2}}

\newcommand\blfootnote[1]{%
  \begingroup
  \renewcommand\thefootnote{}\footnote{#1}%
  \addtocounter{footnote}{-1}%
  \endgroup
}

\newcommand{\cX}{\mathcal{X}}
\DeclareMathOperator*{\argmin}{arg\,min}
\newcommand{\geo}{\text{\sc geo}}

\newcommand{\Prob}[1]{\mathbb P\left( #1 \right)}
\newcommand{\ind}[1]{\mathds{1}\{#1\}}
\newcommand{\norm}[1]{\left\Vert #1 \right\Vert}
\newcommand{\E}{\mathbb E}
\newcommand{\ip}[1]{\left< #1 \right>}
\newcommand{\logit}{\operatorname{logit}}
\newcommand{\cdbar}{\; | \;}

\newcommand{\mnist}{{\sc mnist}}

\newcommand{\supp}{\operatorname{Supp}}
\newcommand{\sign}{\operatorname{sign}}

\newcommand{\cI}{\mathcal{I}}
\newcommand{\cC}{\mathcal{C}}

\newcommand{\cB}{\mathcal B}
\newcommand{\cF}{\mathcal F}
\newcommand{\cG}{\mathcal G}
\newcommand{\cD}{\mathcal D}

\newcommand{\cW}{\mathcal W}
\newcommand{\cR}{\mathcal R}

\newcommand{\cS}{\mathcal S}
\newcommand{\cZ}{\mathcal Z}
\newcommand{\cM}{\mathcal M}

\newcommand{\N}{\mathbb N}
\newcommand{\R}{\mathbb R}

\newcommand{\Q}{\mathbb Q}
\newcommand{\cN}{\mathcal N}

\let\empty\epsilon
\let\epsilon\varepsilon

\allowdisplaybreaks 

\newtheorem{myprop}{Proposition}

\ShortHeadings{Online Learning with Gated Linear Networks}{Veness, Lattimore, Bhoopchand, Grabska-Barwinska, Mattern, Toth}
\firstpageno{1}

\begin{document}

\title{Online Learning with Gated Linear Networks}

\author{\name Joel Veness$^\dagger$ \email aixi@google.com \\
		\name Tor Lattimore$^\dagger$ \email lattimore@google.com \\
		\name Avishkar Bhoopchand$^\dagger$ \email avishkar@google.com\\
		\name Agnieszka Grabska-Barwinska \email agnigb@google.com\\
		\name Christopher Mattern \email cmattern@google.com \\
		\name Peter Toth \email petertoth@google.com \\
       \addr Google, DeepMind
}

\editor{}

\maketitle

\begin{abstract}
This paper describes a family of probabilistic architectures designed for online learning under the logarithmic loss.
Rather than relying on non-linear transfer functions, our method gains representational power by the use of data conditioning.
We state under general conditions a learnable capacity theorem that shows this approach can in principle learn any bounded Borel-measurable function on a compact subset of euclidean space; the result is stronger 
than many universality results for connectionist architectures because we provide both the model and the learning procedure for which convergence is guaranteed.
\end{abstract}


\blfootnote{$\dagger$ denotes joint first authorship.}

\section{Introduction}

This paper explores the use of techniques from the online learning and data compression communities for the purpose of high dimensional density modeling, with a particular emphasis on image density modeling.
The current state of the art is almost exclusively dominated by various deep learning based approaches from the machine learning community.
These methods are trained offline, can generate plausible looking samples, and generally offer excellent empirical performance. But they also have some drawbacks. Notably, they require multiple passes through
the data (not online) and do not come with any kind of meaningful theoretical guarantees.
Of course density modeling is not just a machine learning problem; in particular, it has received considerable study from the statistical data compression community. 
Here the emphasis is typically on \emph{online} density estimation, as this in conjunction with an adaptive arithmetic encoder \citep{Witten1987} obviates the need to encode the model parameters along with the encoded data when performing compression.
Our main contribution is to show that a certain family of neural networks, composed of techniques originating from the data compression and online learning communities, can in some circumstances match the performance 
of deep learning based approaches in just a single pass through the data, while also enjoying universal source coding guarantees.

\subsection{Related Work}

\subsubsection{Online, neural, density estimation}

Although neural networks have been around for a long time, their application to data compression has been restricted until relatively recently due to hardware limitations.
To the best of our knowledge, the first promising empirical results were presented by \cite{schmidhuber96}, who showed that an offline trained 3-layer neural network could outperform Lempel-Ziv based approaches on text prediction; interestingly online prediction was noted as a promising future direction if the computational performance issues could be addressed.

Building on the work of \cite{schmidhuber96}, \cite{Mahoney2000} introduced a number of key algorithmic and architectural changes tailored towards computationally efficient online density estimation.
Rather than use a standard fully connected MLP, neurons within a layer were partitioned into disjoint sets; given a single data point, hashing was used to select only a single neuron from each set in each layer, implementing a kind of hard gating mechanism. 
Since only a small subset of the total weights were used for any given data point, a speed up of multiple orders of magnitude was achieved.
Also, rather than using backpropagation, the weights in each neuron were adjusted using a local learning rule, which was found to work better in an online setting.

The next major enhancement was the introduction of \emph{context-mixing}, culminating in the PAQ family of algorithms \citep{Mahoney2013}. 
At a high level, the main idea is to combine the predictions of multiple history-based models using a network architecture similar in spirit to the one described above.
Common choices of models to combine include $n$-grams, skip-grams \citep{Guthrie06}, match models \citep{Mahoney2013}, Dynamic Markov Coding \citep{Cormack87}, as well as specialized models for commonly occurring data sources.
Many variants along this theme were developed, including improvements to the network architecture and weight update mechanism, with a full history given in the book by \cite{Mahoney2013}.
At the time of writing, context mixing variants achieve the best compression ratios on widely used benchmarks such as the Calgary Corpus \citep{bell90}, 
Canterbury Corpus \citep{bell97}, and Wikipedia \citep{hutterprize}.

The excellent empirical performance of these methods has recently motivated further investigation into understanding why these techniques seem to perform so well and whether they have applicability beyond data compression.
\cite{knoll12} investigated the PAQ8 variant \citep{Mahoney2013} from a machine learning perspective, providing a useful interpretation of the  as a kind of \emph{mixture of experts} \citep{Jacobs1991} architecture, as well as providing an up-to-date summary of the algorithmic developments since the tech report of \cite{Mahoney2005}.
They also explored the performance of PAQ8 in a variety of standard machine learning settings including text categorization, shape recognition and classification, 
showing that competitive performance was possible.
One of the key enhancements in PAQ7 was the introduction of logistic mixing, which applied a logit based non-linearity to a neurons input; \cite{Mattern12, Mattern13} generalized 
this technique to non-binary alphabets and justified this particular form of input combination via a KL-divergence minimization argument.
Furthermore he showed that due to the convexity of the instantaneous loss (in the single neuron case), the local gradient based weight update used within PAQ7 and later versions 
is justified in a regret-minimizing sense with respect to piece-wise stationary sources.
Subsequent work \citep{Mattern16} further analyzed the regret properties of many of the core building blocks used to construct many of the base models used within PAQ models, 
but a theoretical understanding of the mixing network with multiple interacting geometric mixing neurons remained open.

\subsubsection{Batch Neural Density Estimation}

Generative image modeling is a core topic in unsupervised learning and
has been studied extensively over the last decade from a deep learning perspective. The topic is sufficiently developed that there are now practical techniques for generating 
realistic looking natural images \citep[and others]{larochelle2011, icml2015_gregor15, VanDenOord2016, ChenKSDDSSA16, GulrajaniKATVVC16}.
Modeling images is necessarily a high dimensional problem, which precludes the use of many traditional density estimation techniques on both computational and statistical grounds.
Although image data is high dimensional, it is also quite structured, and is well suited to deep learning approaches which can easily incorporate appropriate notions of locality and translation invariance into the architecture.

The aforementioned methods all require multiple passes over the dataset to work well, which makes it significantly more challenging to incorporate them into naturally online settings such as reinforcement learning, where density modeling is just starting to see promising applications \citep{VenessBHCD15,OstrovskiBOM17}. 

\subsubsection{Miscellaneous}

Our work shares some similarity to that of \cite{Balduzzi16}, who used techniques from game theory and online convex programming to analyze some convergence properties of deep convolutional ReLU neural networks by introducing a notion of gated regret.
While extremely general, this notion of gated regret is defined with respect to an oracle that knows which ReLU units will fire for any given input; while this is sufficient to guarantee convergence, it does not provide any guarantees on the quality of the obtained solution.
Our work fundamentally differs in that we explicitly construct a family of gated neural architectures where it is easy to locally bound the usual notion of regret, allowing us to provide far more meaningful theoretical guarantees in our restricted setting.

\cite{foerster17a} recently introduced a novel recurrent neural architecture whose modeling power was derived from using a data-dependent affine transformation as opposed to a non-linear activation function. 
As we shall see later, this is similar in spirit to our approach; we use a product of data-dependent weight matrices to provide representation power instead of a non-linear activation function.
Our work differs in that we consider an online setting, and use a local learning rule instead of backpropagation to adjust the weights.

\subsection{Contribution}

Our main contributions are:
\begin{itemize}
\item to introduce a family of neural models, Gated Linear Networks, which consist of a sequence of data dependent linear networks coupled with an appropriate choice of gating function; we show that the high performance PAQ models are a special case of this framework, opening up many potential avenues for improvement; 
\item to theoretically justify the local weight learning mechanism in these architectures, and highlight how the structure of the loss allows many other techniques 
to give similar guarantees. Interestingly, while the original motivation for introducing gating was computational \citep{Mahoney2000,knoll12}, we show that this technique also adds meaningful 
representational power as well;
\item introduce an adaptive regularization technique which allows a Gated Linear Network to have competitive loss guarantees with respect to all possible sub-networks obtained via pruning the original network; 
\item to provide an \emph{effective} capacity theorem showing that with an appropriate choice of gating function, a sufficiently large network size and a Hannan consistent/no-regret \citep{Cesa-Bianchi2006} local learning method, Gated Linear Networks \emph{will} learn any continuous density function to an arbitrary level of accuracy. 
\end{itemize}

\subsection{Paper Outline}

After introducing the required notation, we proceed by describing a notion of ``neuron'' suitable for our purposes, which we will call a \emph{gated geometric mixer}; this will form 
the elementary building block for all of our subsequent work.
We then proceed to describe \emph{gated linear networks}, which are feed-forward architectures composed of multiple layers of such gated geometric mixing neurons.
Our theoretical results are then presented, which include a capacity analysis of this family of models. 
A technique for adaptive regularization, \emph{sub-network switching}, is then introduced, along with its associated theoretical properties.
Next we describe an architecture built from these principles, reporting empirical results on \mnist, both classification and density modelling.

\section{Gated Geometric Mixing}

This section introduces the basic building block of our work, the \emph{gated geometric mixer}.
We start with some necessary notation, before describing (ungated) \emph{geometric mixing} and some of its properties.
We then generalize this technique to the gated case, giving rise to the gated geometric mixer.

\subsection{Notation}
\label{sec:notation}

Let $\Delta_d = \{x \in [0,1]^d : \norm{x}_1 = 1\}$ be the $d$ dimensional probability simplex embedded in $\R^{d+1}$ and
$\cB = \{ 0, 1 \}$ the set of binary elements.
The indicator function for set $A$ is $\mathds{1}_A$ and satisfies $\mathds{1}_A(x) = 1$ if $x \in A$ and $\mathds{1}_A(x) = 0$ otherwise.
For predicate $P$ we also write $\mathds{1}[P]$, which evaluates to $1$ if $P$ is true and $0$ otherwise.
The scalar element located at position $(i,j)$ of a matrix $A$ is $A_{ij}$, with the $i$th row and $j$ column denoted by $A_{i*}$ and $A_{*j}$ respectively. 
For functions $f:\R \to \R$ and vectors $x \in \R^d$ we adopt the convention of writing $f(x) \in \R^d$ for the coordinate-wise image of $x$ under $f$ so that $f(x) = \left( f(x_1), \dots, f(x_d) \right)$.
If $p,q \in [0,1]$, then $\cD(p,q) = p \log p/q + (1-p) \log(1-p)/(1-q)$ is the Kullback-Leibler (KL) divergence between Bernoulli distributions with parameters $p$ and $q$ respectively.
Let $\cX$ be a finite, non-empty set of symbols, which we call the alphabet.
A string of length $n$ over $\cX$ is a finite sequence $x_{1:n} = x_1x_2 \ldots x_n \in \cX^n$ with $x_t \in \cX$ for all $t$. 
For $t \leq n$ we introduce the shorthands $x_{<t} = x_{1:t-1}$ and $x_{\leq t} = x_{1:t}$. 
The string of length zero is $\empty$ and the set of all finite strings is $\cX^* = \{\empty\} \cup \bigcup_{i=1}^{\infty} \cX^i$.
The concatenation of two strings $s,r \in \cX^*$ is denoted by $sr$.
A sequential, probabilistic model $\rho$ is a probability mass function $\rho : \cX^* \to [0,1]$, satisfying the constraint that 
$\rho(x_{1:n}) = \sum_{y\in\cX} \rho(x_{1:n}y)$
for all $n \in \mathbb{N}$, $x_{1:n} \in \cX^n$, with $\rho(\empty) = 1$.
Under this definition, the conditional probability of a symbol $x_n$ given previous data $x_{<n}$ is defined as $\rho(x_n \, | \, x_{<n}) = \rho(x_{1:n}) / \rho(x_{<n})$ 
provided $\rho(x_{<n}) > 0$, with the familiar chain rules $\rho(x_{1:n}) = \prod_{i=1}^n \rho(x_i \, | \, x_{<i})$ and $\rho(x_{i:j}\, |\, x_{<i}) = \prod_{k=i}^j \rho(x_k \, | \, x_{<k})$ applying as usual.

\subsection{Geometric Mixing}

\begin{figure}[t!]
\centering
\vspace{-3.5em}
\includegraphics{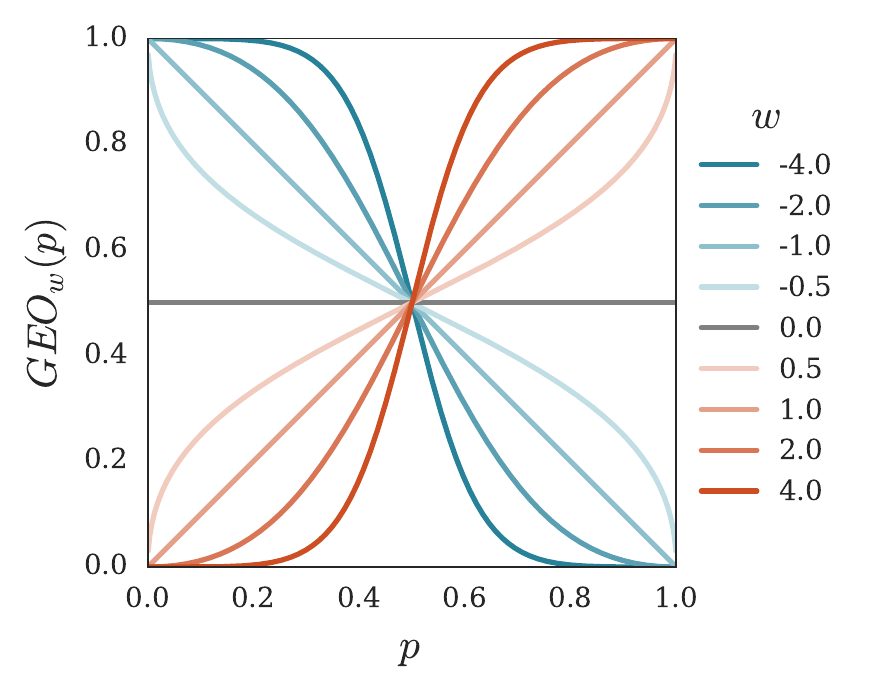}
\vspace{-1.5em}
\caption{Role of $w$ in a one-dimensional geometric mixer with input probability $p$.}
\label{pic:geo}
\end{figure}

We now describe geometric mixing, an adaptive online ensemble technique for obtaining a single conditional probability estimate from the output of multiple models. 
We provide a brief overview here, restricting our attention to the case of binary sequence prediction for simplicity of exposition; for more information see the work of \cite{Mattern12,Mattern13,Mattern16}. 

\subsubsection{Model}

Given $m$ sequential, probabilistic, binary models $\rho_1, \dots, \rho_m$, Geometric Mixing provides a principled way of combining the $m$ associated conditional probability distributions into a single conditional probability distribution, giving rise to a probability measure on binary sequences that has a number of desirable properties.
Let $x_t \in \{ 0, 1 \}$ denote the Boolean target at time $t$.
Furthermore, let $p_t = \left( \, \rho_1(x_t=1 | x_{<t}), \dots, \rho_m(x_t =1| x_{<t}) \, \right)$.
Given a convex set $\cW \subset \mathbb{R}^m$ and parameter vector $w \in \cW$, the Geometric Mixture is defined by
\begin{equation}\label{eq:geomix}
\geo_{w}(x_t = 1; p_t) = \frac{\prod_{i=1}^m p_{t,i}^{w_i}}{ \prod_{i=1}^m p_{t,i}^{w_i} + \prod_{i=1}^m (1 - p_{t,i})^{w_i} },
\end{equation}
with $\geo_{w}(x_t = 0; p_t) = 1 - \geo_{w}(x_t = 1; p_t)$. 

\paragraph{Properties.}

A few properties of this formulation are worth discussing.
Setting $w_i = 1/m$ for $i \in [1,m]$ is equivalent to taking the geometric mean of the $m$ input probabilities.
As illustrated in Figure \ref{pic:geo}, higher absolute values of $w_i$ translate into an increased belief into model $i$'s prediction; for negative values of $w_i$, the prediction needs to be reversed (see blue lines in Fig.\,\ref{pic:geo}).
If $w = 0$ then $\geo_{w}(x_t = 1 ; p_t) = 1/2$; and in the case where $w_i = 0$ for $i \in \cS$ where $\cS$ is a proper subset of $[1,m]$, the contributions of the models in $\cS$ are essentially ignored (taking $0^0$ to be $1$).
Due to the product formulation, every model also has ``the right of veto", in the sense that a single $p_{t,i}$ close to 0 coupled with a $w_i > 0$ drives $\geo_{w}(x_t = 1; p_t)$ close to zero.
These properties are graphically depicted in Figure \ref{pic:geo}.

\paragraph{Alternate form.}

Via simple algebraic manipulation, one can also express Equation \ref{eq:geomix} as 
\begin{equation}\label{eq:geomix2}
\geo_{w}(x_t = 1; p_t) = \sigma \left(w \cdot \logit(p_t)\right)\,,
\end{equation}
where $\sigma(x) = 1 / (1 + e^{-x})$ denotes the sigmoid function, and $\logit(x) = \log(x/(1-x))$ is its inverse. This form is best suited for numerical implementation.
Furthermore, the property of having an input non-linearity that is the inverse of the output non-linearity is the reason why a linear network is obtained when layers of geometric mixers are stacked on top of each other.

\paragraph{Remarks}
This form of combining probabilities is certainly not new; for example it is discussed by \cite{Genest86} under the name \emph{logarithmic opinion pooling} and is closely related to the Product of Experts model of \cite{hinton2002}.
While one can attempt to justify this choice of probability combination via appealing to notions such as its \emph{externally Bayesian} properties \citep{Genest86}, 
or as the solution to various kinds of weighted redundancy minimization problems \citep{Mattern13,Mattern16}, the strongest case for this form of model mixing is 
simply its superior empirical performance when $w$ is adapted over time via online gradient descent \citep{zinkevich03} 
compared with other similarly complex alternates such as linear mixing \citep{Mattern12,Mahoney2013}. 

\subsubsection{Properties under the Logarithmic Loss}\label{sec:loss-prop}

We assume a standard online learning setting, whereby at each round $t \in \mathbb{N}$ a predictor outputs a binary distribution $q_t~:~\cB~\to~[0,1]$, with the environment responding with an observation $x_t \in \cB$.
The predictor then suffers the logarithmic loss 
\begin{align*}
\ell_t(q_t, x_t) = -\log q_t(x_t),
\end{align*}
before moving onto round $t+1$.
The loss will be close to 0 when the predictor assigns high probability to $x_t$, and large when low probability is assigned to $x_t$; in the extreme cases, a zero loss is obtained when $q_t(x_t)=1$, and an infinite loss is suffered when $q_t(x_t)=0$. 
In the case of geometric mixing, which depends on both the $m$ dimensional input predictions $p_t$ and the parameter vector $w \in \cW$, we abbreviate the loss by defining
\begin{align}
\ell^{\geo}_t(w) = \ell_t(\geo_w( \cdot \, ;p_t), x_t). \label{eq:geo-loss}
\end{align}
The following proposition, proven in Appendix~\ref{app:loss-prop}, establishes some useful properties about the logarithmic loss when applied to Geometric Mixing.

\begin{myprop}\label{prop:loss-prop}
For all $t \in \mathbb{N}$, $x_t \in \cB$, $p_t \in (0,1)^m$ and $w \in \cW$ we have:
\begin{enumerate}
\item $\nabla \ell^\geo_t(w) = \left( \geo_w( 1 ; p_t) - x_t \right) \logit(p_t)$.
\item $\norm{\nabla \ell^\geo_t(w)}_2 \leq \norm{\logit(p_t)}_2$.
\item $\ell^\geo_t(w)$ is a convex function of $w$.
\item If $p_t \in [\epsilon, 1 - \epsilon]^{m}$ for some $\epsilon \in (0,1/2)$, then:
 \begin{enumerate}
	\item $\ell^\geo_t:\cW \to \R$ is $\alpha$-exp-concave with 
	$\alpha = \sigma \left( \log\left( \frac{\epsilon}{1 - \epsilon} \right)^{\max_{w \in \cW} \norm{w}_1} \right) ;$ 
	\item $\norm{\nabla \ell^\geo_t(w)}_2 \leq \sqrt{m} \log \left( \frac{1}{\epsilon} \right)$.
 \end{enumerate}
\end{enumerate}
\end{myprop}
Note also that Part 4.(a) of Proposition \ref{prop:loss-prop} implies that $\alpha = \Theta \left( \left( \epsilon / (1 - \epsilon) \right)^{\max_{w \in \cW} \norm{w}_1} \right)$. 

\paragraph{Remarks}
The above properties of the sequence of loss functions make it straightforward to apply one of the many different online convex programming techniques to adapt $w$ at the end of each round.
The simplest to implement and most computationally efficient is Online Gradient Descent \citep{zinkevich03} with $\cW$ equal to some choice of hypercube.
The convexity of $\ell^\geo_t(w)$ allows one to derive an $O(\sqrt{T})$ regret bound using standard techniques with respect to the best $w^* \in \cW$ chosen in hindsight 
provided an appropriate schedule of decaying learning rates is used, where $T$ denotes the total number of rounds.
Furthermore, when a fixed learning rate is used one can show $O(s \sqrt{T})$ regret guarantees with respect to data sequences composed of $s$ pieces \citep{Mattern13}.
More refined algorithms based on coin betting allow the $s$ to be moved inside the square root as shown by \cite{JOWW17}. 
The $\alpha$-exp-concavity also allows second order techniques such as Online Newton Step \citep{HazanAK07} and its sketched variants \citep{LuoACL16} to be applied. 
These techniques allow for $O(\log{T})$ regret guarantees with respect to the best $w^* \in \cW$ chosen in hindsight at the price of a more computationally demanding 
weight update procedure and further dependence on $\epsilon$ as given in Part 4.(a) of Proposition \ref{prop:loss-prop}.

\subsection{Gated Geometric Mixing}

We are now in a position to define our notion of a single neuron, a \emph{gated geometric mixer}, which we obtain by adding a contextual gating procedure to geometric mixing.
Here, contextual gating has the intuitive meaning of mapping particular examples to particular sets of weights.
More precisely, associated with each neuron is a context function $c : \cZ \to \cC$, where $\cZ$ is the set of possible \textit{side information} and $\cC = \{0,\dots,k-1 \}$ for some $k \in \mathbb{N}$ is the \textit{context space}.
Given a convex set $\cW \subset \mathbb{R}^d$, each neuron is parametrized by a matrix
\begin{equation*}
W = \begin{bmatrix}
w_0\\
\vdots \\
w_{k-1}
\end{bmatrix}
\end{equation*}
with each row vector $w_i \in \cW$ for $0 \leq i < k$.
The context function $c$ is responsible for mapping a given piece of side information $z_t \in \cZ$ to a particular row $w_{c(z_t)}$ of $W$, which we then use with standard geometric mixing.
More formally, we define the gated geometric mixture prediction as
\begin{equation}\label{eq:gated_geomix}
\geo^c_{W}(x_t = 1 \,;\, p_t, z_t) = \geo_{w_{c(z_t)}}(x_t = 1 \,;\, p_t),
\end{equation}
with $\geo^c_{W}(x_t = 0 \,;\, p_t, z_t) := 1 - \geo^c_{W}(x_t = 1 \,;\, p_t, z_t)$. 
Once again we have the following equivalent form
\begin{equation}\label{eq:gated_geomix2}
\geo^c_{W}(x_t = 1; p_t, z_t) =  \sigma \left(w_{c(z_t)} \cdot \logit(p_t)\right).
\end{equation}
The key idea is that our neuron can now specialize its weighting of the input predictions based on some property of the side information $z_t$.
The side information can be arbitrary, for example it could be some additional input features, or even functions of $p_t$. 
Ideally the choice of context function should be informative in the sense that it simplifies the probability combination task.

\subsubsection{Context Functions}

Here we introduce several classes of general purpose context functions that have proven useful empirically, theoretically, or both. 
All of these context functions take the form of an indicator function $\mathds{1}_{\cS}(z) : \cZ \to \cB$ on a particular choice of set $\cS \subseteq \cZ$, with $\mathds{1}_{\cS}(z) := 1$ if $z \in \cS$ and $0$ otherwise.
This list, while sufficient for understanding the content in this paper, is by no means exhaustive; practitioners can and should choose context functions that make sense for the given domain of interest.

\paragraph{Half-space contexts}

This choice of context function is useful for real-valued side information.
Given a normal $v \in \R^d$ and offset $b \in \R$, consider the associated affine hyperplane $\{x \in \mathbb{R}^d : x \cdot v = b \}$.
This divides $\mathbb{R}^d$ in two, giving rise to two half-spaces, one of which we denote $H_{v,b} = \{x \in \mathbb{R}^d : x \cdot v \geq b\}$.
The associated half-space context is then given by $\mathds{1}_{H_{v,b}}(z)$.

\paragraph{Skip-gram contexts}

The following type of context function is useful when we have multi-dimensional binary side information and can expect single components of $\cZ$ to be informative.
If $\cZ = \cB^d$, given an index $i \in [1,d]$, a skip-gram context is given by the function
$\mathds{1}_{\cS_i}(z)$
where $S_i := \{ z \in \cZ : z_i = 1 \}$.
One can also naturally extend this notion to categorical multi-dimensional side information or real valued side information by using thresholding. 

\subsubsection{Context Function Composition}
\label{sec:context_composition}

Richer notions of context can be created from composition.
In particular, any finite set of $d$ context functions $\{ c_i : \cZ \to \cC_i \}_{i=1}^d$ with associated context spaces $\cC_1, \dots, \cC_d$ can be composed into a single higher order context function $c : \cZ \to \cC$, where $\cC = \left\{0,1,\ldots, -1 + \prod_{i=1}^d |\cC_i| \right\}$ by defining
\begin{equation*}
c(z) = \sum_{i=1}^d c_i(z) \left ( \prod_{j=i+1}^d |\cC_j| \right).
\end{equation*}

For example, we can combine four different skip-gram contexts into a single context function with a context space containing 16 elements.
The combined context function partitions the side information based on the values of the four different binary components of the side information.

\section{Gated Linear Networks}

We now introduce \emph{gated linear networks} (GLNs), which are feed-forward networks composed of many layers of gated geometric mixing neurons.
Each neuron in a given layer outputs a gated geometric mixture over the predictions from the previous layer, with the final layer consisting of just a single neuron that determines the output of the entire network.

\subsection{Model}

\begin{figure}[t!]
\centering
\vspace{-2em}
\includegraphics[scale=0.7]{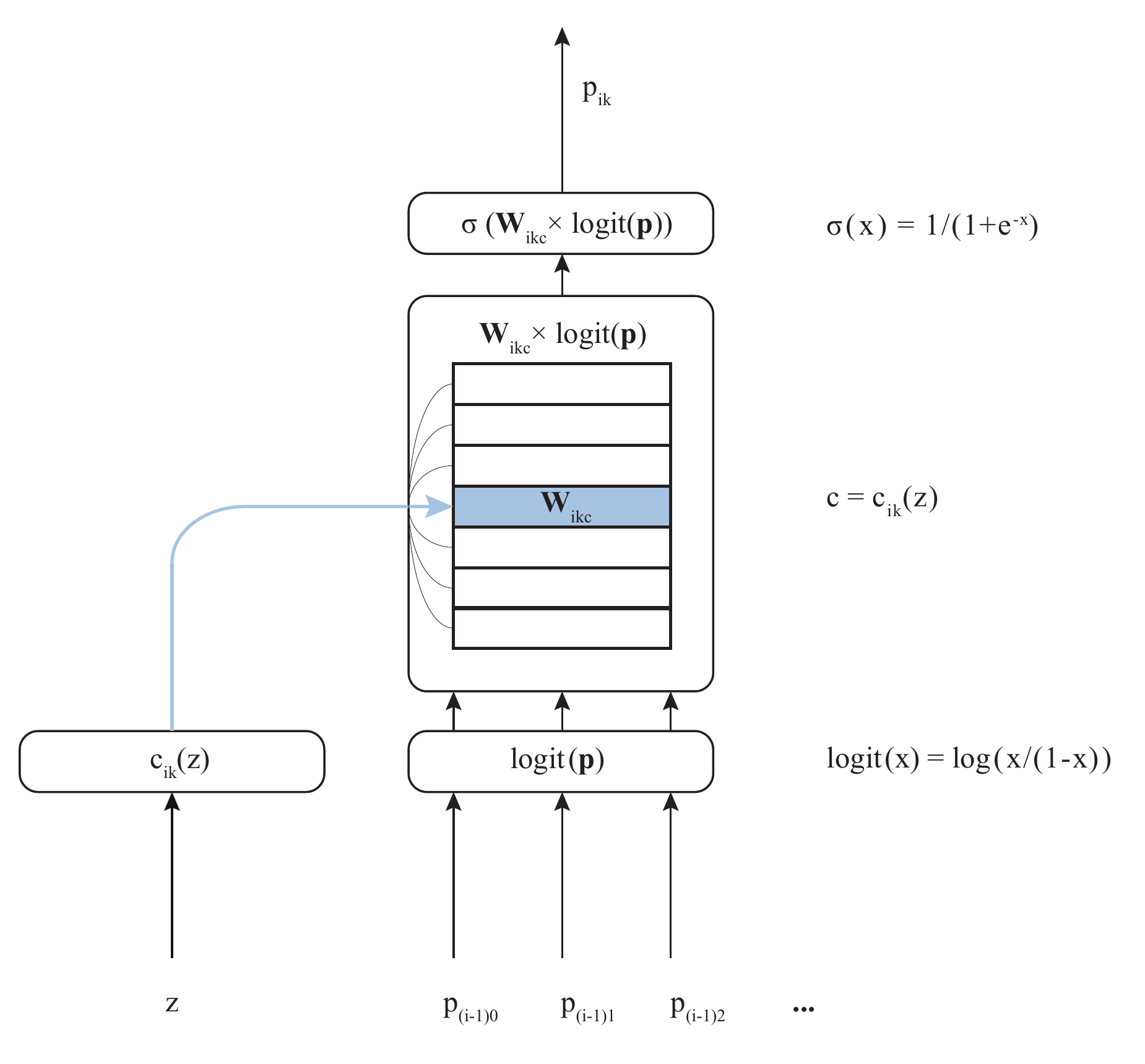}
\caption{Behaviour of the $k$th neuron in layer $i > 0$ of a Gated Linear Network.}\label{fig:neuron}
\end{figure}

Once again let $\cZ$ denote the set of possible side information and $\cC \subset \mathbb{N}$ be a finite set called the context space.
A GLN is a network of sequential, probabilistic models organized in $L+1$ layers indexed by $i \in \{0,\ldots,L\}$, with $K_i$ models in each layer.
Models are indexed by their position in the network when laid out on a grid; for example, $\rho_{ik}$ will refer to the $k$th model in the $i$th layer.
The zeroth layer of the network is called the \textit{base layer} and is constructed from $K_0$ probabilistic base models $\left \{ \rho_{0k} \right \}_{k=0}^{K_0-1}$ of the form given in Section \ref{sec:notation}.
Since each of their predictions is assumed to be a function of the given side information and all previously seen examples, these base models can essentially be arbitrary.
The nonzero layers are composed of gated geometric mixing neurons. 
Associated to each of these will be a fixed context function $c_{ik} : \cZ \to \cC$ that determines the behavior of the gating. 
In addition to the context function, for each context $c \in \cC$ and each neuron $(i,k)$ there is an associated weight vector $w_{ikc} \in \R^{K_{i-1}}$ which is used to geometrically mix the inputs.
We also enforce the constraint of having a non-adaptive bias model on every layer, which will be denoted by $\rho_{i0}$ for each layer $i$.
Each of these bias models will correspond to a Bernoulli Process with parameter $\beta$.
These bias models play a similar role to the bias inputs in MLPs.

\paragraph{Network Output}

\begin{figure}[t!]
\centering
\vspace{-3.5em}
\includegraphics[scale=0.5]{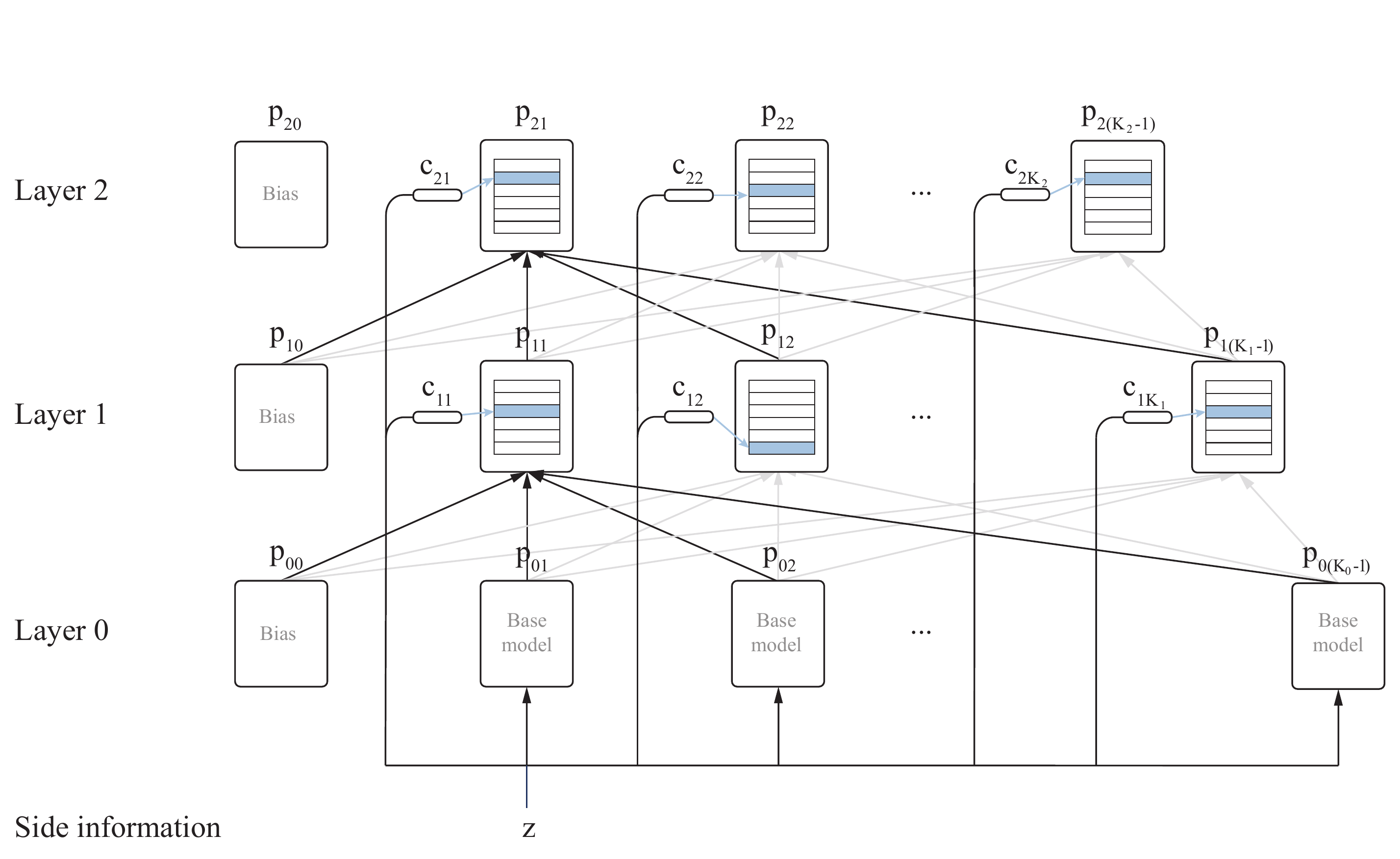}
\caption{A graphical illustration of a Gated Linear Network.}
\label{pic:network}
\end{figure}

Given a $z \in \cZ$, a weight vector for each neuron is determined by evaluating its associated context function.
The output of each neuron can now be described inductively in terms of the outputs of the previous layer. 
To simplify the notation, we assume an implicit dependence on $x_{<t}$ and let $p_{ij}(z) = \rho_{ij} ( x_t = 1 \, | \, x_{<t} ; z)$ denote the output of the $j$th neuron in the $i$th layer, and $p_i(z) = \left( p_{i0}(z), p_{i1}(z), \dots, p_{i K_{i-1}} (z) \right)$ 
the output of the $ith$ layer.
The bias output for each layer is defined to be $p_{i0}(z) = \beta$ for all $z \in \cZ$, for all $0 \leq i \leq L+1$, where $\beta \in (0,1) \setminus \{  1/2 \}$.\footnote{The seemingly odd technical constraint $\beta \neq 0.5$ is required to stop the partial derivatives of the loss with respect to the bias weight being 0 under geometric mixing.}
From here onwards we adopt the convention of setting $\beta = e / (e+1)$ so that $\logit(\beta) = 1$.
For layers $i \geq 1$, the $k$th node in the $i$th layer receives as input the vector of dimension $K_{i-1}$ of predictions of the preceding layer (see Figure \ref{pic:network}).
The output of a single neuron is the geometric mixture of the inputs with respect to a set of weights that depend on its context, namely
\begin{align*}
p_{ik}(z) = \sigma \left( w_{ikc_{ik}(z)} \cdot \logit \left( p_{i-1} \left( z \right) \right) \right), 
\end{align*}
as illustrated by Figure \ref{fig:neuron}.
The output of layer $i$ can be re-written in matrix form as
\begin{align}\label{eq:neuron}
p_i(z) = \sigma(W_i(z) \logit(p_{i-1}(z)))\,,
\end{align}
where $W_i(z) \in \R^{K_i \times K_{i-1}}$ is the matrix with $k$th row equal to $w_{ik}(z) = w_{ikc_{ik}(z)}$.
Iterating Equation \ref{eq:neuron} once gives 
\begin{equation*}
p_i(z) = \sigma\left( W_i \left(z \right) \logit \left( \sigma \left( W_{i-1} \logit(p_{i-2}\left(z \right) \right) \right) \right).
\end{equation*}
Since $\logit$ is the inverse of $\sigma$, the $i$th iteration of Equation \ref{eq:neuron} simplifies to
\begin{align}
\label{eq:linear}
p_i(z)= \sigma\Bigl( W_{i}(z) W_{i-1}(z) \dots W_{1}(z) \logit(p_0(z)) \Bigr).
\end{align}
Equation~(\ref{eq:linear}) shows the network behaves like a linear network \citep{Baldi1989,SaxeMG13}, but with weight matrices that are data-dependent.
Without the data dependent gating, the product of matrices would collapse to single linear mapping, giving the network no additional modeling power over a single neuron \citep{minsky69perceptrons}.

\subsection{Learning in Gated Linear Networks}

We now describe how the weights are learnt in a Gated Linear Network.
While architecturally a GLN appears superficially similar to the well-known multilayer perception (MLP), what and how it learns is very different.
The key difference is that every neuron in a GLN probabilistically predicts the target. 
This allows us to associate a loss function to each neuron. 
This loss function will be defined in terms of just the parameters of the neuron itself; thus, unlike backpropagation, \emph{learning will be local}.
Furthermore, this loss function will be convex, which will allow us to avoid many of the difficulties associated with training typical deep architectures.
For example, we can get away with simple deterministic weight initializations, which aids the reproducibility of empirical results.
The convexity allows us to learn from correlated inputs in an online fashion without suffering significant degradations in performance.
And as we shall see later, GLNs are extremely data efficient, and can produce state of the art results in a single pass through the data.
One should think of each layer as being responsible for trying to directly improve the predictions of the previous layer, rather than a form of implicit non-linear feature/filter construction as is the case with MLPs trained offline with back-propagation \citep{Rumelhart1988}.

\paragraph{Weight space and initialization}
For our subsequent theoretical analysis, we will require that the weights satisfy the following mild technical constraints:
\begin{enumerate}
\item $w_{ijc0} \in [a,b] \subset \mathbb{R}$ for some real $a < 0$ and $b > 0$;
\item $w_{ijc} \in \cS \subset \R^{K_{i-1}}$ where $\cS$ is a compact, convex set such that $\Delta_{K_{i-1}-1} \subset \cS$.
\end{enumerate}
One natural way to simultaneously meet these constraints is to restrict each neurons contextual weight vectors to lie within some (scaled) hypercube:
$w_{ijc} \in [-b,b]^{K_{i-1}}$, where $b \geq 1$. 
For what follows we will need notation to talk about weight vectors at specific times.
From here onwards we will use $w^{(t)}_{ijc}$ to denote the weight vector $w_{ijc}$ at time $t$.
As each neuron will be solving an online convex programming problem, initialization of the weights is straightforward and is non-essential to the theoretical analysis.
Choices found to work well in practice are:
\begin{itemize}
\item \emph{zero initialization:} $w^{(1)}_{ikc} = 0$ $\quad$ for all $i,k,c$.
\item \emph{geometric average initialization:} $w^{(1)}_{ikc} = 1/K_{i-1}$ $\quad$ for all $i,k,c$.
\end{itemize}
The zero initialization can be seen as a kind of sparsity prior, where each input model is considered a-priori to be unimportant, which has the effect of making the geometric mixture rapidly adapt to incorporate the predictions of the best performing models.
The geometric average initialization forces the geometric mixer to (unsurprisingly) initially behave like a geometric average of its inputs, which makes sense if one believes that the predictions of each input model are reasonable.
One could also use small random weights, as is typically done in MLPs, but we recommend against this choice because it makes little practical difference and has a negative impact on reproducibility.

\paragraph{Weight update}

Learning in GLNs is straightforward in principle; as each neuron probabilistically predicts the target, one can treat the current input to any neuron as a set of expert predictions and apply a single step of local online learning using one of the no-regret methods discussed in Section \ref{sec:loss-prop}.
The particular choice of optimization method will not in any way affect the capacity results we present later.
For practical reasons however we focus our attention on Online Gradient Descent \citep{zinkevich03} with $\cW_{ik}$ a hypercube.
This allows the weight update for any neuron at layer $i$ to be done in time complexity $O(K_{i-1})$, which permits the construction of large networks.

More precisely, let $\ell^{ij}_t(w_{ijc})$ denote the loss of the $j$th neuron in layer $i$. 
Using Equation~(\ref{eq:geo-loss}) we have
\begin{equation}
\ell^{ij}_t(w_{ijc}) 
= \ell_t(\geo_{w}( \cdot \, ; p_{i-1}(z_t), x_t). 
\end{equation} 
Now, for all $i \in [1,L]$, $j\in K_i$, and for all $c = c_{ij}(z_t)$, we set
\begin{equation}\label{eq:gln_weight_update}
w^{(t+1)}_{ijc} = \Pi_i \left( w^{(t)}_{ijc} - \eta_t \nabla \ell^{ij}_t(w^{(t)}_{ijc}) \right),
\end{equation}
where 
$\Pi_i$ is the projection operation onto hypercube $[-b,b]^{K_{i-1}}$:
\begin{align*}
\Pi_i(x) = \argmin_{y \in [-b,b]^{K_{i-1}} } \norm{y - x}_2\,.
\end{align*}
The projection is efficiently implemented by clipping every component of $w^{(t)}_{ijc}$ to the interval $[-b,b]$.
The learning rate $\eta_t \in \mathbb{R}_+$ can depend on time; we will revisit the topic of appropriately setting the learning rate in subsequent sections.

\paragraph{Performance guarantees for individual neurons}
Let $(i,k)$ refer to a non-bias neuron. 
We now state a performance guarantee for a single neuron when $\cW_{ik} = [-b,b]^{K_{i-1}}$.
Given $a \in \cC$ consider the index set $N_{ika}(n) = \{t \le n : c_{ik}(z_t) = a\}$, the set of rounds when context $a$ is observed by neuron $(i,k)$. 
The regret experienced by the neuron on $N_{ika}(n)$ is defined as the difference between the losses suffered by the neuron and the losses it would have suffered using the best choice of weights in hindsight.
\begin{align}
R_{ika}(n) = \sum_{t \in N_{ika}(n)} \ell^{ik}_t(w_{ika}^{(t)}) - \min_{w \in \cW_{ik}} \sum_{t \in N_{ika}(n)}  \ell^{ik}_t(w) \,.
\label{def:regret-context}
\end{align}
\cite{zinkevich03} showed that if the learning rates are set to $\eta_t = \frac{D}{G \sqrt{t}}$, then the regret of gradient descent is at most 
\begin{align}
R_{ika}(n) \leq \frac{3DG}{2} \sqrt{|N_{ika}(n)|}\,, 
\label{eq:regret-guarantee}
\end{align}
where $D = \max_{x,y \in \cW} \norm{x - y}_2$ is the diameter of $\cW_{ik}$ and $G$ is an upper bound on the $2$-norm of the gradient of the loss.
In our case we have $D = 2b\sqrt{K_{i-1}}$, with a bound $G \leq \sqrt{K_{i-1}} \log(\frac{1}{\epsilon})$ following from Proposition~\ref{prop:loss-prop} Part 4.(b) by ensuring the input to each neuron is within $[\epsilon, 1-\epsilon]$. 
The regret definition given in \cref{def:regret-context} is on the subsequence of rounds when $c_{ik}(z_t) = a$. The total regret for neuron $(i,k)$ is
\begin{align}
R_{ik}(n) = \sum_{a \in \cC} R_{ika}(n)\,.
\label{def:regret}
\end{align}
Combining the above and simplifying, we see that
\begin{align*}
R_{ik}(n) &\leq \frac{3DG}{2} \sum_{a \in \cC} \sqrt{|N_{ika}(n)|} 
= \frac{3DG|\cC|}{2} \sum_{a \in \cC} \frac{1}{|\cC|} \sqrt{|N_{ika}(n)|}\\
&\leq \frac{3DG|\cC|}{2} \sqrt{ \sum_{a \in \cC} \frac{1}{|\cC|} |N_{ika}(n)|}
= \frac{3DG}{2} \sqrt{ |\cC| \sum_{a \in \cC}  |N_{ika}(n)|},
\end{align*}
where the first and second inequalities follow from \cref{eq:regret-guarantee} and Jensen's inequality respectively.
Noting that $\bigcup_{a \in \cC} N_{ika}(n) = \{1, \dots, n \}$ allows us to further simplify the bound to
\begin{align*}
R_{ik}(n) \leq \frac{3DG}{2} \sqrt{|\cC| n}\,.
\end{align*}
In the particular case where each input probability is bounded between $[\epsilon, 1-\epsilon]$ and the weights reside in the hypercube $[-b,b]^{K_{i-1}}$, the above bound becomes
\begin{align}
R_{ik}(n) \leq 3 b K_{i-1} \sqrt{|\cC| n} \, \log \left(\frac{1}{\epsilon} \right).
\label{eq:gated_bound}
\end{align}
The cumulative loss suffered by the omniscient algorithm is:
\begin{align*}
\sum_{a \in \cC} \min_{w \in \cW_{ik}} \sum_{t \in N_{ika}(n)} \ell^{ik}_t(w)\,.
\end{align*}
The regret is the difference between the losses suffered by the learning procedure and the above quantity. Since the latter usually grows like $\Omega(n)$, the
$O(\sqrt{n})$ additional penalty suffered by the algorithm is asymptotically negligible.

How should we interpret Equation \ref{eq:gated_bound}? 
First of all, note the linear dependence on the number of inputs; this is the price a single neuron pays in the worst case for increasing the width of the preceding layer.
There is also a linear dependence on $b$, which is expected as when the competitor set gets larger it should be more difficult to compete with the optimal solution in hindsight.
Clipping inputs is common implementation trick for many deep learning systems; here this has the effect of stopping the 2-norm of the gradient exploding, which would adversely affect the worst case regret.
Furthermore, by assuming the unit basis vector $e_j \in \cW_{ik}$ we guarantee that the cumulative loss of neuron $(i,k)$ is never much larger than the cumulative loss of neuron $(i-1,j)$; to see this, note that if the weights were set to $e_j$ then geometric mixing simply replicates the output of the $j$th neuron in the preceding layer. 
This is the bare minimum we should expect, but the guarantee actually provides much more than this by 
ensuring that the cumulative loss of neuron $(i,k)$ is never much worse than the best geometric mixture of its inputs. As we will observe in Section~\ref{sec:capacity}, 
if the context functions are sufficiently rich, then the layering of geometric mixtures combined with \cref{eq:regret-guarantee} leads to a network that 
learns to approximate arbitrary functions.

\subsection{Computational properties}

Some computational properties of Gated Linear Networks are now discussed.

\paragraph{Complexity of a single online learning step}

Generating a prediction requires computing the contexts from the given side information for each neuron, and then performing $L$ matrix-vector products.
Under the assumption that multiplying a $m \times n$ by $n \times 1$ pair of matrices takes $O(mn)$ work, the total time complexity to generate a single prediction is $O \left( \sum_{i=1}^L K_i K_{i-1} \right)$ for the matrix-vector products, which in typical cases will dominate the overall runtime.
Using online gradient descent just requires updating the rows of the weight matrices using Equation \ref{eq:gln_weight_update}; this again takes time $O \left( \sum_{i=1}^L K_i K_{i-1} \right)$.

\paragraph{Parallelism}

When generating a prediction, parallelism can occur within a layer, similar to an MLP.
The local training rule however enables all the neurons to be updated simultaneously, as they have no need to communicate information to each other.
This compares favorably to back-propagation and significantly simplifies any possible distributed implementation.
Furthermore, as the bulk of the computation is primarily matrix multiplication, large speedups can be obtained straightforwardly using GPUs.

\paragraph{Static prediction with pre-trained models}

In the case where no online updating is desired, prediction can be implemented potentially more efficiently depending on the exact shape of the network architecture.
Here one should implement Equation \ref{eq:linear} by first solving a Matrix Chain Ordering problem \citep{hu1982} to determine the optimal way to group the matrix multiplications.

\section{Effective Capacity of Gated Linear Networks}\label{sec:capacity}

Neural networks have long been known to be capable of approximating arbitrary continuous functions with almost any reasonable
activation function \cite[and others]{Hor91}. We will show that provided the contexts are chosen sufficiently richly, then GLNs also have the capacity to approximate large classes of functions. More than this, the capacity is \textit{effective} in the
sense that gradient descent will eventually find the approximation. In contrast, similar results for neural networks show the existence
of a choice of weights for which the neural network will approximate some function, but do not show that gradient descent (or any other single algorithm) will converge to these weights.
Of course we do not claim that gated linear networks are the only model with an effective capacity result. For example, $k$-nearest-neighbour with appropriately chosen $k$ is also universal
for large function classes. We view universality as a good first step towards a more practical theoretical understanding of a model, but the ultimate goal is to provide meaningful finite-time
theoretical guarantees for \textit{interesting} and \textit{useful} function classes. Gated linear networks have some advantages over other architectures in the sense that they are constructed
from small pieces that are well-understood in isolation and the nature of the training rule eases the analysis relative to neural networks.

Our main theorem is the following abstract result that connects the richness of the context functions used in each layer to the ability of a network with no base predictors
to make constant improvements. After the theorem statement we provide intuition on simple synthetic examples and analyse the richness of certain classes of context functions.

\begin{theorem}[Capacity]\label{thm:capacity}
Let $\mu$ be a measure on $(\R^d, \cF)$ with $\cF$ the Lebesgue $\sigma$-algebra and let $z_t \in \R^d$ be a sequence of independent random variables sampled from $\mu$.
Furthermore, let $x_1,x_2,\ldots$ be a sequence of independent Bernoulli random variables with $\Prob{x_t = 1|z_t} = f(z_t)$ for some $\cF$-measurable function $f:\R^d \to [0,1]$.
Consider a gated linear network and for each layer $i$ let $\cG_i = \{c_{ik} : 1 \leq k \leq K_i\}$ be the set of context functions in that layer.
Assume there exists a $\delta > 0$ such that for each non-bias neuron $(i,k)$ the weight-space $\cW_{ik}$ is compact and $(-\delta, \delta) \times \Delta^{K_{i-1}-2} \subseteq \cW_{ik}$.
Provided the weights are learned using a no-regret algorithm, then
the following hold with probability one:
\begin{enumerate}
\item For each non-bias neuron $(i,k)$ there exists a non-random function $p^*_{ik}: \R^d \to (0,1)$ such that
\begin{align*}
\lim_{n\to\infty} \frac{1}{n} \sum_{t=1}^n \sup_{z \in \supp(\mu)} \left|p_{ik}(z;w^{(t)}) - p^*_{ik}(z)\right| = 0\,.
\end{align*}
\item The average of $p^*_{ik}$ converges to the average of $f$ on the partitions of $\R^d$ induced by the contexts:  
\begin{align*}
\sum_{i=1}^\infty \max_k \max_{c \in \cG_{i+1}} \sum_{a \in \cC} \left(\int_{c^{-1}(a)} \left(f(z) - p^*_{ik}(z)\right) d\mu(z)\right)^2 \leq \max\left\{2, \frac{4}{\delta}\right\}\log\left(\frac{1}{\beta}\right)\,,
\end{align*}
where $\beta = e/(1+e)$ is the output of bias neurons $p_{i0}(z)$.
\item There exists a non-random $\cF$-measurable function $p_\infty^*: \R^d \to (0,1)$ such that 
\begin{align*}
\lim_{i\to\infty} \max_k \int_{\R^d} (p^*_{ik}(z) - p_\infty^*(z))^2 d\mu(z) = 0\,.
\end{align*}
\end{enumerate}
\end{theorem}

The assumption that the sequence of side information is independent and identically distributed 
may be relaxed significantly. A sufficient criteria is that $z_1,z_2,\ldots$ follows a Markov processes that mixes to stationary measure $\mu$. 
Unfortunately this requires a benign assumption on the stability of the learning algorithm, which is satisfied by most online learning algorithms including
online gradient descent with appropriately decreasing learning rates.

\begin{definition}
A learning procedure for GLNs is stable if there exist constants $r < 0$ and $C > 0$ such that for all data sequences $(x_t)_t$ and $(z_t)_t$ it holds that: 
\begin{align*}
\norm{w^{(t)}_{ika} - w^{(t+1)}_{ika}}_\infty \leq \mathds{1}_{c^{-1}_{ik}(a)}(z_t) C t^r \qquad \text{ for all } t\,.
\end{align*}
\end{definition}

The proofs of Theorems~\ref{thm:capacity} and \ref{thm:capacity-markov} are given in Appendices~\ref{sec:thm:capacity} and \ref{sec:thm:capacity-markov} respectively.

\begin{theorem}\label{thm:capacity-markov}
Suppose the learning procedure is stable and that all conditions of Theorem~\ref{thm:capacity} hold except that $(z_t)_t$ is sampled from an aperiodic and $\phi$-irreducible Markov chain with stationary distribution $\mu$.
Then (1), (2) and (3) from Theorem~\ref{thm:capacity} hold.
\end{theorem}

Theorem \ref{thm:capacity-markov} suggests that GLNs could be well suited reinforcement learning applications.
For example, one could combine a GLN-based density model with the policy evaluation approach of \cite{VenessBHCD15}.

\paragraph{Interpretation of Theorems~\ref{thm:capacity} and \ref{thm:capacity-markov}}
The theorem has three parts. The first says the output of each neuron converges almost surely in Cesaro average to some nonrandom function on the support of $\mu$. 
Outside the support of $\mu$ the algorithm will never observe data, so provided there is no covariate shift we should not be concerned about this.
The second part gives a partial characterisation of what this function is. Let
\begin{align*}
\epsilon_i = \max_k \max_{c \in \cG_{i+1}} \left(\int_{c^{-1}(a)} (f(z) - p^*_{ik}(z)) d\mu(z) \right)^2\,. 
\end{align*}
The theorem says that $\sum_{i=1}^\infty \epsilon_i = O(1)$, which since $\epsilon_i$ is clearly positive implies that $\lim_{i\to\infty} \epsilon_i \to 0$. The convergence is fast in the
sense that $\epsilon_i$ is approximately subharmonic. Suppose that $\cG_i = \cG$ is the same for all layers and $\epsilon_i$ is small, then
for all $c \in \cG$ and $a \in \cC$ we have
\begin{align*}
\int_{c^{-1}(a)} f(z) d\mu(z) \approx \int_{c^{-1}(a)} p^*_{ik}(z) d\mu(z)\,,
\end{align*}
which means that $p^*_{ik}$ is approximately equal to $f$ on \text{average} on all sets $c^{-1}(a) \subset \R^d$. We will shortly see that if  
 $\{c^{-1}(a) : c \in \cG, a \in \cC\}$ is sufficiently rich, then $f(z) \approx p^*_{ik}(z)$ also holds.
Finally, the last part of the theorem shows that as the number of layers tends to infinity, the output of neurons in later layers converges to some single function.
Note that all results are about what happens in the limit of infinite data. In principle one could apply the regret guarantees from gradient descent to calculate a rate
of convergence, a task which we leave for the future.

\paragraph{Intuition on synthetic data}
We now illustrate the behaviour of a small network with various contexts in the simple case of $d = 1$. 
Figure \ref{fig:network_at_convergence} shows the chosen architecture and the problem setup. Each box relates to a non-bias neuron of the network, comparing the 
ground truth, $f(z)$ (black line) to the output of each neuron, $p_{ik}(x|z)$ (coloured line). Axes are identical for all boxes, and labelled on bottom left: 
The single-dimensional side information $z \in [-3, 3]$ is used to derive the Bernoulli parameters according to a Gaussian transformation $f(z) = e^{ -z^2 / 2}$. 
Thus, the training data consist of samples $z_t$ drawn uniformly over the range $[-3, 3]$, and their labels drawn from Bernoulli distributions parametrised by $f(z_t)$.

\begin{figure}[b!]
\begin{center}
\includegraphics[width=4.5cm]{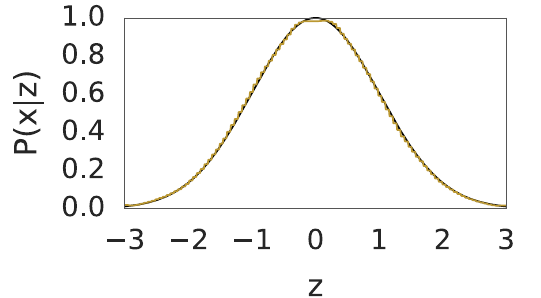}
\caption{Output of two-layer network with evenly spaced half-space contexts and $K_1 = 100$}
\label{fig:even_half_spaced_ctxs}
\end{center}
\end{figure}

The GLN has six layers, with 3/2/2/2/2/1 non-bias neurons in each layer reading from bottom to top (Fig. \ref{fig:network_at_convergence}), and a single base predictor $p_{00}(z) = \alpha$ 
for all $z$ (not shown). All bias neurons are also set to $p_{0k} = \beta$ (not shown). The half-space contexts of the non-bias neurons are parametrized by an offset, 
chosen randomly (coloured dashed lines in Fig. \ref{fig:network_at_convergence}).
The bottom row depicts the output of the network in the first layer after training. Each neuron learns to predict the average of $f$ on
the partitions of $[-3, 3]$ induced by $c_{1k}^{-1}(a)$ for each $a \in \{0, 1\}$.
Already in the second layer, the neurons' outputs gain in complexity, combining information not only from the neuron's own contexts, but also from the estimates of the neurons below. In effect, every output has 4 discontinuities (albeit two of them are very close, due to two contexts in the first layer having similar offsets).
The top panel shows the output of the non-bias neuron in the sixth layer, which even for this small network is predicting with reasonable accuracy.

\begin{figure}[t]
\begin{center}
\includegraphics{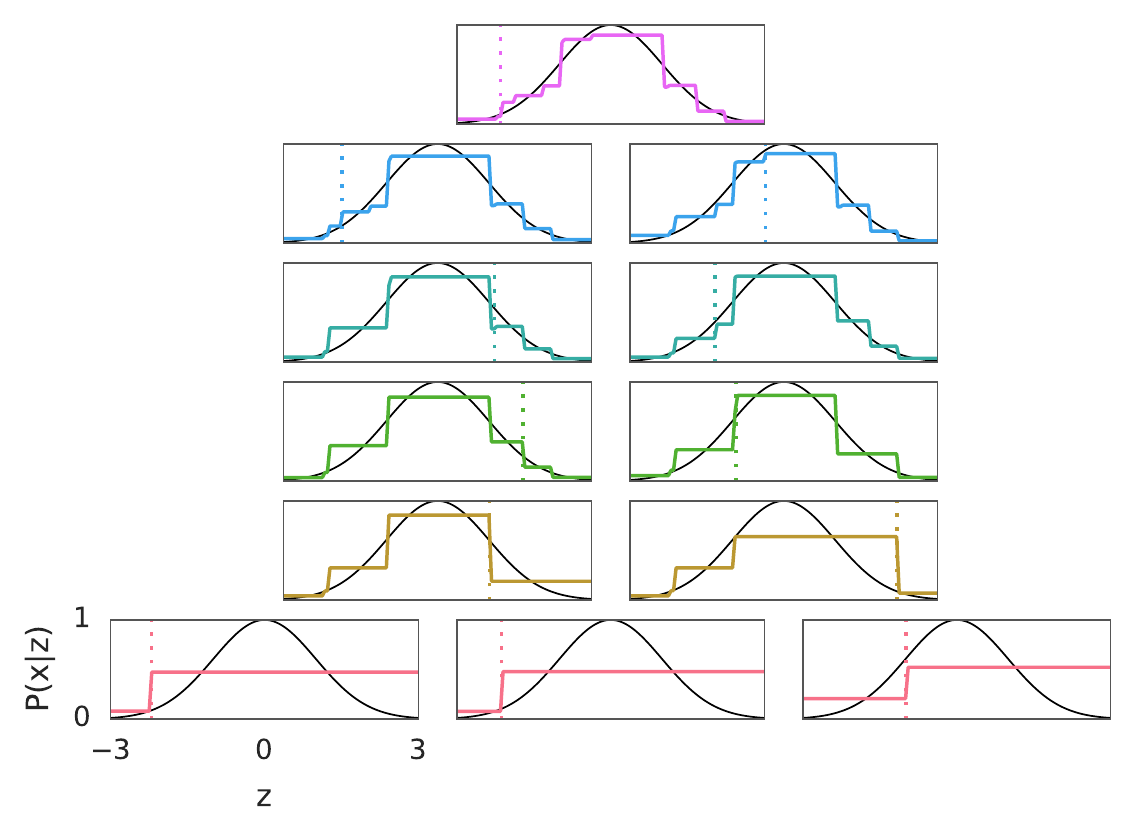}
\end{center}
\caption{Output of a six-layer network with half-space contexts. Each box represents a non-bias neuron in the network, the function to fit is shown in black, and the output distribution learnt by each neuron 
is shown in colour (for example, red for the first layer and purple for the top-most neuron). All axes are identical, as labeled on bottom left. The dashed coloured lines 
represent the choice of hyperplane for each neuron.}
\label{fig:network_at_convergence}
\end{figure}

The performance of the network is easily improved by increasing the number of neurons in the first layer, as shown in Figure \ref{fig:even_half_spaced_ctxs}.
If $K_1 = 100$, then the output of the non-bias neuron in the second layer is a near-perfect predictor.
As it happens, a two-layer network with half-space contexts and sufficiently many neurons in the first layer can
learn to approximate any continuous function, which explains the performance of such a shallow network on this problem.
Note that the base predictor is completely uninformative in this case and the half-space contexts are not tuned to the data.

\paragraph{Applications of Theorem~\ref{thm:capacity}}
We now see several applications of Theorem~\ref{thm:capacity} for different contexts.
If we assume that $\cG_i = \cG$ for all $i$, then the result shows that in the limit of infinite data a sufficiently deep network predicts like the average of $f$ on $c^{-1}(a)$ for all context functions $c \in \cG$ and
all $a \in \cC$. When the space of context functions is sufficiently rich this implies that the network indeed predicts like $f$.
Let $\cG$ be a (possibly infinite) set of context functions such that 
\begin{align*}
\norm{h}_{\cG} = \sup_{c \in \cG} \sum_{a \in \cC} \left|\int_{c^{-1}(a)} h(z) d\mu(z) \right|
\end{align*}
is a norm on the space of bounded measurable functions. Note that homogeneity, non-negativity and the triangle inequality are trivial, which leaves one only to check
that $\norm{h}_{\cG} = 0$ implies that $h = 0$ $\mu$-almost-everywhere.
Now suppose that $\cG_1 \subset \cG_2 \subset \cG_3 \subset \cdots$ is a sequence of finite subsets of $\cG$ such that
for any bounded measurable $h$
\begin{align}
\label{eq:approx-norm}
\lim_{i \to \infty} \norm{h}_{\cG_i} = \norm{h}_{\cG}\,,
\end{align}
which is true in the special case that $\cG$ is countable and $\cG_i \uparrow \cG$.
Then by Parts (2) and (3) of Theorem~\ref{thm:capacity} we have
\begin{align*}
0 = \lim_{i \to \infty} \max_{c \in \cG_i} \sum_{a \in \cC} \left|\int_{c^{-1}(a)} (f(z) - p_\infty^*) d\mu(z)\right|
= \lim_{i\to \infty} \norm{f - p_\infty^*}_{\cG_i}
= \norm{f - p_\infty^*}_{\cG}\,.
\end{align*}
Therefore $\mu(z : f(z) - p_\infty(z) = 0) = 1$, which means that almost surely the network asymptotically predicts like $f$ $\mu$-almost-everywhere.
We now give some example countable classes of contexts, which by \cref{eq:approx-norm} can be used to build a universal network. All the claims in the following enumeration
are proven in Appendix~\ref{app:topology}.

\begin{itemize}
\item (\textit{Covers})\,\, Let $B_r(x) = \{y : \norm{x - y}_2 < r\}$ be the open set of radius $r > 0$. 
If $\cG = \{ \mathds{1}_{B_r(x)} : x \in \Q^d, r \in \Q \cap (0, \infty) \}$, then $\norm{\cdot}_{\cG}$ is a norm on
the space of bounded measurable functions. 
\item (\textit{Covers cont.})\,\, The above result can easily be generalised to the situation where $\cG$ is the set of indicator functions 
on any countable base of $\R^d$.
\item (\textit{Hyperplanes})\,\, 
Now assume that $\mu$ is absolutely continuous with respect to the Lebesgue measure.
If $\cG$ is the space of context functions that are indicators on half-spaces $\cG = \{ \mathds{1}_{H_{v,c}} : v \in \Q^d, \norm{v} = 1, c \in \Q\}$, then $\norm{\cdot}_{\cG}$
is a norm on the spaces of bounded $\cF$-measurable functions.
\end{itemize}

We conclude this section with three remarks.

\begin{remark}[Asymptotic non-diversity]
Part (3) of Theorem~\ref{thm:capacity} shows that asymptotically in the depth of the network there is vanishingly little diversity in the predictors in the bottom layer. 
The intuition is quite simple. If the contexts are sufficient that the network learn the true function, then eventually all neurons will predict perfectly (and hence be the same).
On the other hand, if the class of contexts is not rich enough for perfect learning, then there comes a time when all neurons in a given layer cannot do better than copying the output of their
best parent. Once this happens, all further neurons make the same predictions.
\end{remark}

\begin{remark}[Universality of shallow half-spaces]
If $\mu$ is supported on a compact set and $f$ is continuous, then a two-layer network is sufficient to approximate $f$ to arbitrary precision with half-space contexts. 
A discussion of this curiosity is deferred to Appendix~\ref{app:2-layer-1-dim}. In practice we observe that except when $d = 1$ it is beneficial to use
more layers. And note that arbitrary precision is only possible by using arbitrarily large weights and wide networks.
\end{remark}

\begin{remark}[Effective capacity $\neq$ capacity]
Theorem~\ref{thm:capacity} shows that if the context functions are sufficiently rich and the network is 
deep, then the network can learn any well-behaved function using online gradient descent.
This is an effective capacity result because we have provided a model (the network) and a learning procedure (OGD) that work together
to learn any well-behaved function. 
For a fixed architecture, the capacity of a GLN can be much larger than the effective capacity. In particular, there exist functions that can be modelled 
by a particular choice of weights that online gradient descent will never find. (This does not contradict Theorem~\ref{thm:capacity}; one can always 
construct a larger GLN whose effective capacity is sufficiently rich to model such functions.)
\end{remark}

\vspace{-0.38cm} 

\begin{wrapfigure}[7]{r}{3cm}
\vspace{-0.4cm}
\begin{tikzpicture}[font=\tiny]
\draw[fill=blue,opacity=0.2] (0,0) rectangle (1,1);
\draw[fill=blue,opacity=0.2] (-1,-1) rectangle (0,0);
\draw[fill=red,opacity=0.2] (0,0) rectangle (-1,1);
\draw[fill=red,opacity=0.2] (0,0) rectangle (1,-1);
\draw[<->] (-1.5,0) -- (1.5,0);
\draw[<->] (0,-1.5) -- (0,1.5);
\node at (0.5,0.5) {$f(z) \!=\! 1$};
\node at (-0.5,-0.5) {$f(z)\! =\! 1$};
\node at (0.5,-0.5) {$f(z)\! =\! 0$};
\node at (-0.5,0.5) {$f(z) \!= \! 0$};
\end{tikzpicture}
\end{wrapfigure}
\itshape
To construct an example demonstrating this failure we exploit the fact that training each neuron against the target means that neurons do not coordinate to solve the following `exclusive-or' problem. 
Let $d = 2$ and $\mu$ be uniform on $[-1,1]^2$ and $f:\R^2 \to [0,1]$ be given by $f(z) = \ind{z_1z_2 \geq 0}$. Next suppose $\cG = \{c_1, c_2\}$ where $c_i(z) = \ind{z_i \geq 0}$.
Notice that $z_1$ or $z_2$ alone is insufficient to predict $f(z)$.
If a gated linear network with $K_i = 2$ for all $i \geq 2$ and contexts from $\cG$ is trained in the iid setting of Theorem~\ref{thm:capacity}, then
it is easy to check that $\lim_{i\to\infty} p_{ik}^*(z) = 1/2$ almost surely.
And yet there \textit{exists} a choice of weights such that $p_{ik}(z;w) = f(z)$ for all $z$. The problem is the optimal weights have neurons in the same layer coordinating so that a single neuron
in the next layer has access to $z_1$ and $z_2$ as inputs. But online gradient descent has neurons selfishly optimizing their own performance, which does not lead to coordination in this case.
\normalfont

\section{Adaptive Regularization via Sub-network Switching}

While GLNs can have almost arbitrary capacity in principle, large networks are susceptible to a form of the \emph{catch-up phenomenon} \citep{ErvenGR07}.
During the initial stages of learning neurons in the lower layers typically have better predictive performance than neurons in the higher layers. 
This section presents a solution to this problem based on \emph{switching} \citep{VSHB12}, a fixed share \citep{Herbster1998} variant tailored to the logarithmic loss.
The main idea is simple: as each neuron predicts the target, one can construct a switching ensemble across all neurons predictions.
This guarantees that the predictions made by the ensemble are not much worse than the predictions made by the best sparsely changing sequence of neurons.
We now describe this process in detail.

\subsection{Model}

\begin{figure}[t!]
\centering
\includegraphics[scale=.75]{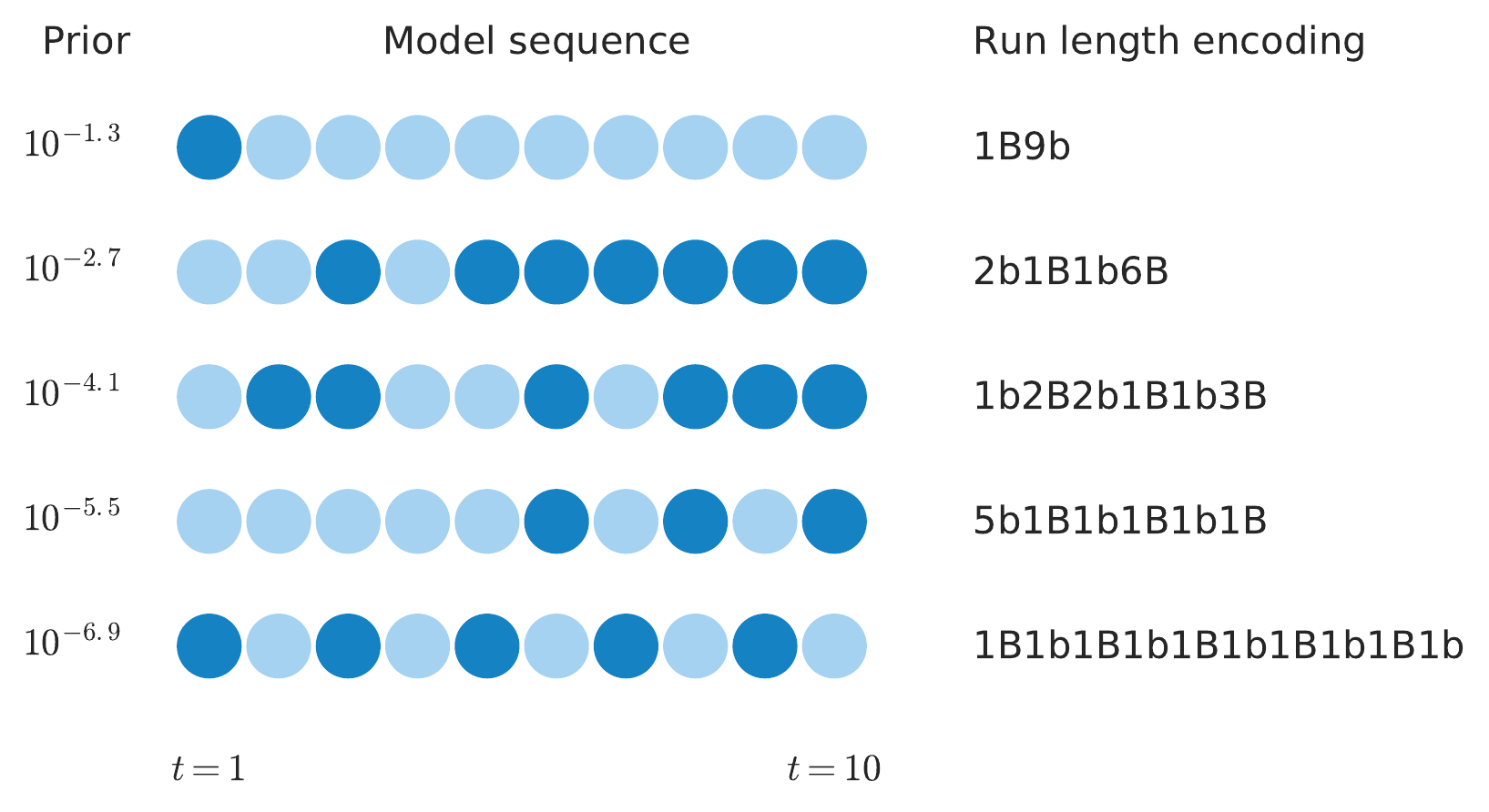}
\caption{A visual depiction of the run-length encoding prior.}
\label{pic:switch_prior}
\end{figure}

Let $\cM = \left \{\rho_{ij} : i \in [1,L], j \in [0, K_i-1] \right \}$ denote the model class consisting of all neurons that make up a particular GLN with $L$ layers and $K_i$ neurons in each layer.
Now for all $n \in \mathbb{N}$, for all $x_{1:n} \in \cX^n$, consider a Bayesian (non-parametric) mixture that puts a prior $w_\tau(\cdot)$ over all \emph{sequences} of neurons inside the index set $\cI_n(\cM) = \cM^n$, namely
\begin{equation}\label{eq:switch_distribution}
\tau(x_{1:n}) = \sum\limits_{\nu_{1:n} \in \cI_n(\cM)} w_\tau(\nu_{1:n}) \; \nu_{1:n}(x_{1:n})
\end{equation}
where $\nu_{1:n}(x_{1:n}) = \prod_{k=1}^n \nu_k(x_k \cdbar x_{<k})$.
As $w_\tau(\cdot)$ is a prior, it needs to be non-negative and satisfy $\sum_{\nu_{1:n}} w_\tau(\nu_{1:n}) = 1$ for all $n \in \mathbb{N}$. 
From the dominance property of Bayesian mixtures it immediately follows that for any $\nu^*_{1:n} \in \cI_n(\cM)$ we have
\begin{equation}\label{eq:switching_dominance}
-\log \tau(x_{1:n}) \leq -\log \left( w_\tau\left(\nu^*_{1:n}\right) \; \nu^*_{1:n}\left(x_{1:n}\right) \right) \leq -\log  w_\tau(\nu^*_{1:n})  -\log \nu^*_{1:n}(x_{1:n}). 
\end{equation}
Thus the regret $\cR_n = -\log \left( \frac{\tau(x_{1:n}) }{ \nu^*_{1:n}(x_{1:n})} \right)$ with respect to a sequence of models $\nu^*_{1:n}$ is upper bounded by $-\log  w_\tau(\nu^*_{1:n})$.
Putting a uniform prior over all neuron sequences would lead to a vacuous regret of $n \log(|\cM|)$, so we are forced to concentrate our prior mass on a 
smaller set of neuron sequences we think a-priori are likely to predict well.

Empirically we found that when the number of training examples is small neurons in the lower layers usually predict better than those in higher layers, but this reverses as more data becomes available.
Viewing the sequence of best-predicting neurons over time as a string, we see that
a run-length encoding gives rise to a highly compressed representation with a length linear in the number of times the best-predicting neuron changes.
Run-length encoding can be implemented probabilistically by using an arithmetic encoder with the following recursively defined prior:
\begin{equation}\label{eq:switch_prior_rec}
w_\tau(\nu_{1:n}) = \left\{
     \begin{array}{lr}
       1 \text{~~~~~~if~~~~} n = 0\\
       \tfrac{1}{|\cM|} \text{~~~if~~~~} n = 1 \\
       w_\tau(\nu_{<n}) \times  \left( \frac{n-1}{n} \mathds{1}[\nu_n = \nu_{n-1}] + \frac{1}{n (|\cM|-1)} \mathds{1}[\nu_n \neq \nu_{n-1}] \right) \text{ otherwise\,.}  
     \end{array}
   \right.
\end{equation}
A graphical depiction of this prior is shown in Figure \ref{pic:switch_prior}.
Multiple sequences of models are depicted by rows of light blue and dark blue circles, with each colour depicting a different choice of model.
A run length encoding of the sequence is given by a string made up of pairs of symbols, the first being an integer representing the length of the run and the second a character ($B$ for dark blue, $b$ for light blue) depicting the choice of model.
One can see from the figure that high prior weight is assigned to model sequences which have shorter run length encodings.
More formally, one can show (see Appendix \ref{app:switch_prior_proof}) that
\begin{equation}\label{eq:switching_prior_bound}
-\log w_{\tau}(\nu_{1:n}) \leq \left( s(\nu_{1:n})+1 \right) \left( \log |\cM| + \log n \right),
\end{equation}
for all $\nu_{1:n} \in \mathcal{I}_n(\mathcal{M})$, where $s(\nu_{1:n}) := \sum_{t=2}^n \mathbb{I}[\nu_t \neq \nu_{t-1}]$ denotes the number of switches from one neuron to another in $\nu_{1:n}$. 
Combining Equations \ref{eq:switching_dominance} and \ref{eq:switching_prior_bound} allows us to upper bound the regret $\cR_n$ with respect to an arbitrary $\nu^*_{1:n} \in \mathcal{I}_n(\mathcal{M})$ by $$\left( s(\nu^*_{1:n})+1 \right) \left( \log |\cM| + \log n \right).$$
Therefore, when there exists a sequence of neurons with a small number $s(\nu^*_{1:n}) \ll n$ of switches that performs well, only logarithmic regret is suffered, and one can expect the switching ensemble to predict almost as well as if we knew what the best performing sparsely changing sequence of neurons was in advance.
Note too that the bound holds in the case where we just predict using a single choice of neuron; here only $\log |\cM| + \log n$ is suffered, so there is essentially no downside to using this method if such architecture adaptivity is not needed.

\subsection{Algorithm}

A direct computation of Equation \ref{eq:switch_distribution} would require $|\cM|^n$ additions, which is clearly intractable. 
The switching algorithm of \cite{VSHB12} runs in time $O(n |\cM|)$, but was originally presented in terms of log-marginal probabilities, which requires a log-space implementation to avoid numerical difficulties.
An equivalent numerically robust formulation is described here, which incrementally maintains a weight vector that is used to compute a convex combination of model predictions at each time step.

Let $u^{(t)}_{ik} \in (0,1]$ denote the switching weight associated with the neuron $(i, k)$ at time $t$.
The weights will satisfy the invariant $\sum_{i=1}^{|L|} \sum_{k=0}^{K_i -1} u^{(t)}_{ik} = 1$ for all $t$.
At each time step $t$, the switching mixture outputs the conditional probability

\begin{equation*}
\tau(x_t  |  x_{<t}) = \sum_{i=1}^{|L|} \sum_{k=0}^{K_i -1} u^{(t)}_{ik}  \rho_{ik}(x_t  |  x_{<t}),
\end{equation*}
with the weights defined, for all $1 \leq i \leq L$ and $0 \leq k < K_i$, by $u^{(1)}_{ik} = 1 / |\cM|$ and
\begin{equation*}
u^{(t+1)}_{ik} = \frac{1}{(t+1)(|\cM|-1)} +
\left( \frac{t |\cM| - t -1}{(t+1)(|\cM|-1)} \right) \frac{u^{(t)}_{ik} \; \rho_{ik}(x_t | x_{<t})}{\tau(x_t  |  x_{<t})}.
\end{equation*}
This weight update can be straightforwardly implemented in $O(|\cM|)$ time per step.
To avoid numerical issues, we recommend enforcing the weights to sum to 1 by explicitly dividing by $\sum_{i=1}^{|L|} \sum_{k=0}^{K_i -1} u^{(t+1)}_{ik}$ after each weight update.

\subsection{Illustration}

\begin{figure}[t!]
\includegraphics{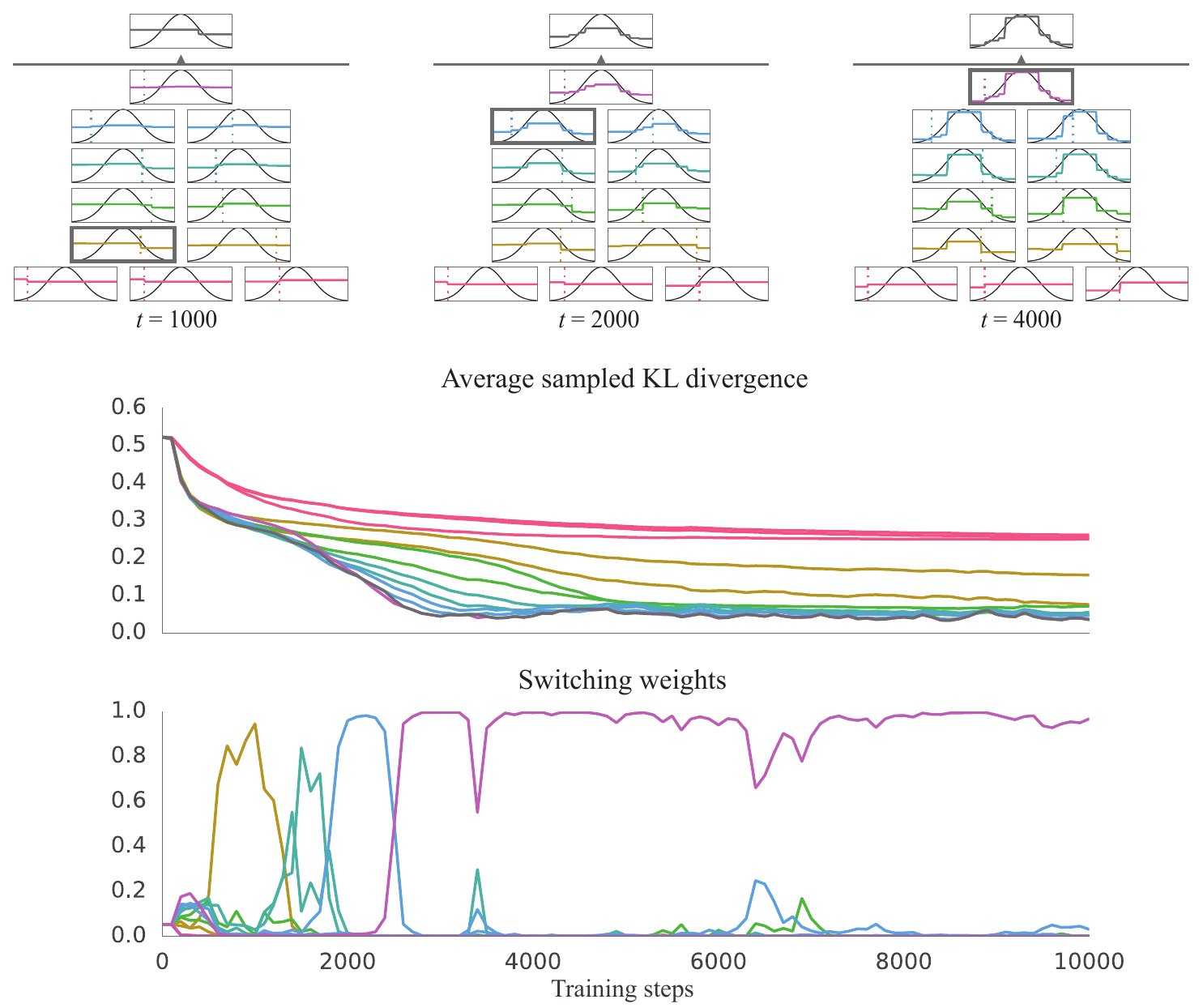}
\caption{Evolution of model fit over time}
\label{pic:switch_fitting}
\end{figure}

Figure \ref{pic:switch_fitting} shows an example of a small 6 layer GLN fitting a family of Bernoulli distributions parametrized 
by a Gaussian transformation $g(z) = e^{ -z^2 / 2}$ of a single dimensional side information value $z \in [-3,3]$,
which is the same setup depicted in \cref{fig:network_at_convergence} without switching.
The GLN has 6 layers, with 3/2/2/2/2/1 neurons on each layer reading from bottom to top.
The output distribution learnt by each neuron is shown in colour (for example, red for the bottom layer and purple for the top neuron).
The dashed vertical line in each distribution denotes the choice of half-space which controls the gating of input examples. 
Each example was generated by first sampling a $z$ uniformly distribution between $[-3,3]$ and then sampling a label from a Bernoulli distribution parameterized by $g(z)$.
The topmost box in each of the four subfigures shows the output distribution (in black) obtained using subnetwork switching.
The top-left subfigure shows the model fit at $t=1000$; the top-middle at $t=2000$; and the top-right the fit at $t=4000$.
One can clearly see the model fit improving over time, and that with sufficient training the higher level neurons better model the target family of distributions.
The bottom graph in the figure shows the evolution of the switching weights over time; here we see that initially a neuron on the 3rd layer dominates the predictions, then one in the 4th layer, 
then the 5th and subsequently settling down on the predictions made by the top-most neuron.

\section{Empirical Results}

We now present a short series of case studies to shed some insight into the capabilities of GLNs.
When describing our particular choice of architecture, we adopt the convention of describing the number of neurons on each layer by a dashed list of natural numbers going from the most shallow layer to the deepest layer.

\subsection{Non-linear decision boundaries with half-spaces}

\begin{figure}[t!]
\centering
\includegraphics[scale=0.16]{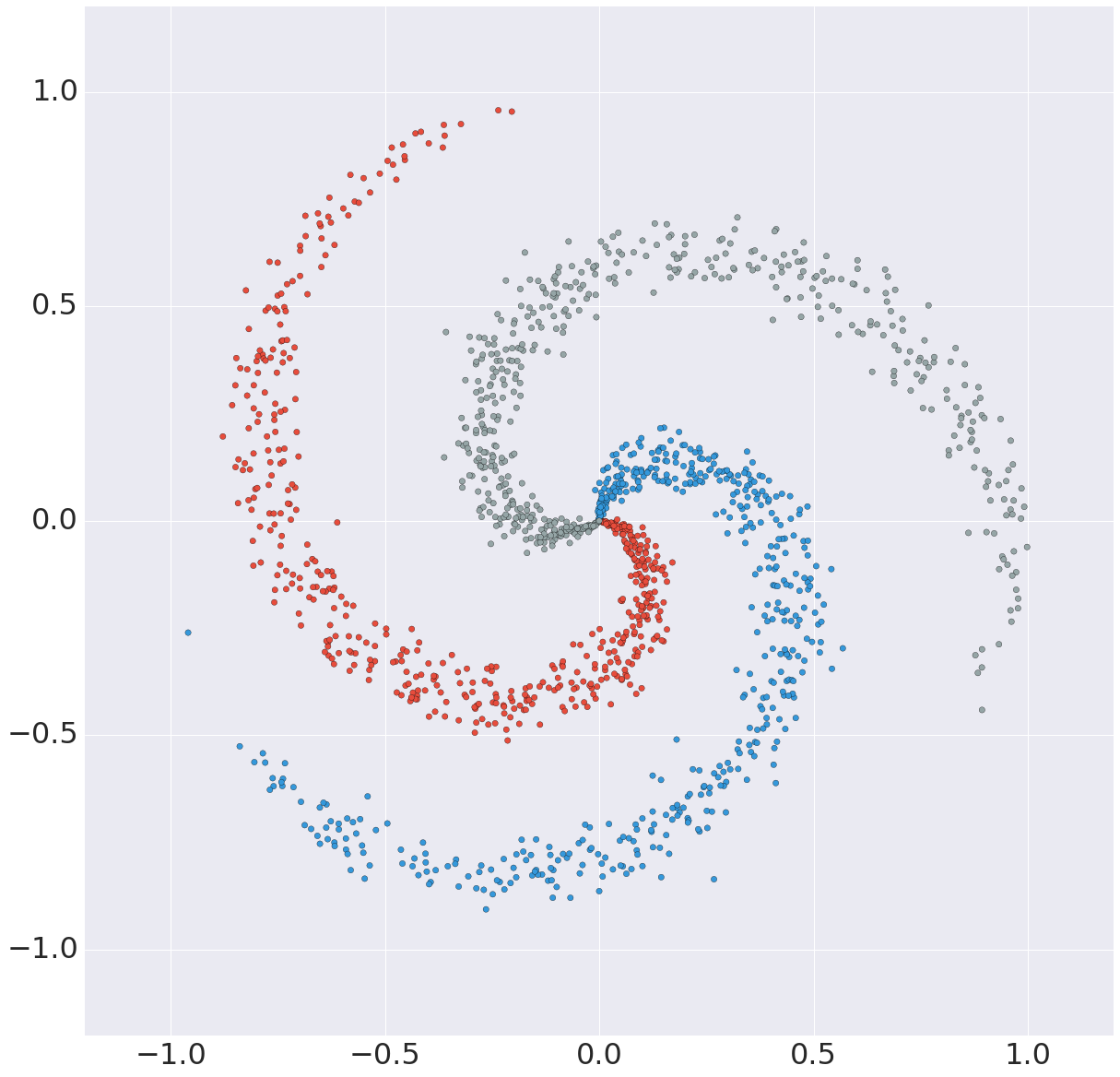}
\caption{The spiral dataset. The horizontal and vertical axes describe the values of the two dimensional side information, with the class label being indicated by either a red, grey or blue point.}
\label{pic:spiral_dataset}
\end{figure}

Our first experiment is on a synthetic `spiral' ternary classification task, which we use to demonstrate the ability of GLNs to model non-linear decision boundaries. 
A visualization of the dataset is shown in Figure \ref{pic:spiral_dataset}.
The horizontal and vertical axes describe the values of the two dimensional side information, with the class label being indicated by either a red, grey or blue point.
To address this task with a GLN, we constructed an ensemble of 3 GLNs to form a standard one-vs-all classifier. 
Each member of the ensemble was a 3 layer network consisting of 50-25-1 neurons which each used a half-space context for gating.
The half-space contexts for each neuron were determined by sampling a 2 dimensional normal vector whose components were distributed according to $\cN(0,36)$, and a bias weight distributed according to $\cN(0,9)$.
Each component of all weight vectors were constrained to lie within $[-200, 200]$.
The side information was the 2-dimensional x/y values, and the input to the network was the component-wise sigmoid of these x/y values (as GLNs require the input to be within [0,1]). 
The learning rate was set to $0.01$.
Figure \ref{pic:decision_boundaries} shows a plot of the resultant decision boundaries for neurons of differing depths after training on a single pass through all the examples.

\begin{figure}[t!]
\centering
\includegraphics[scale=0.08]{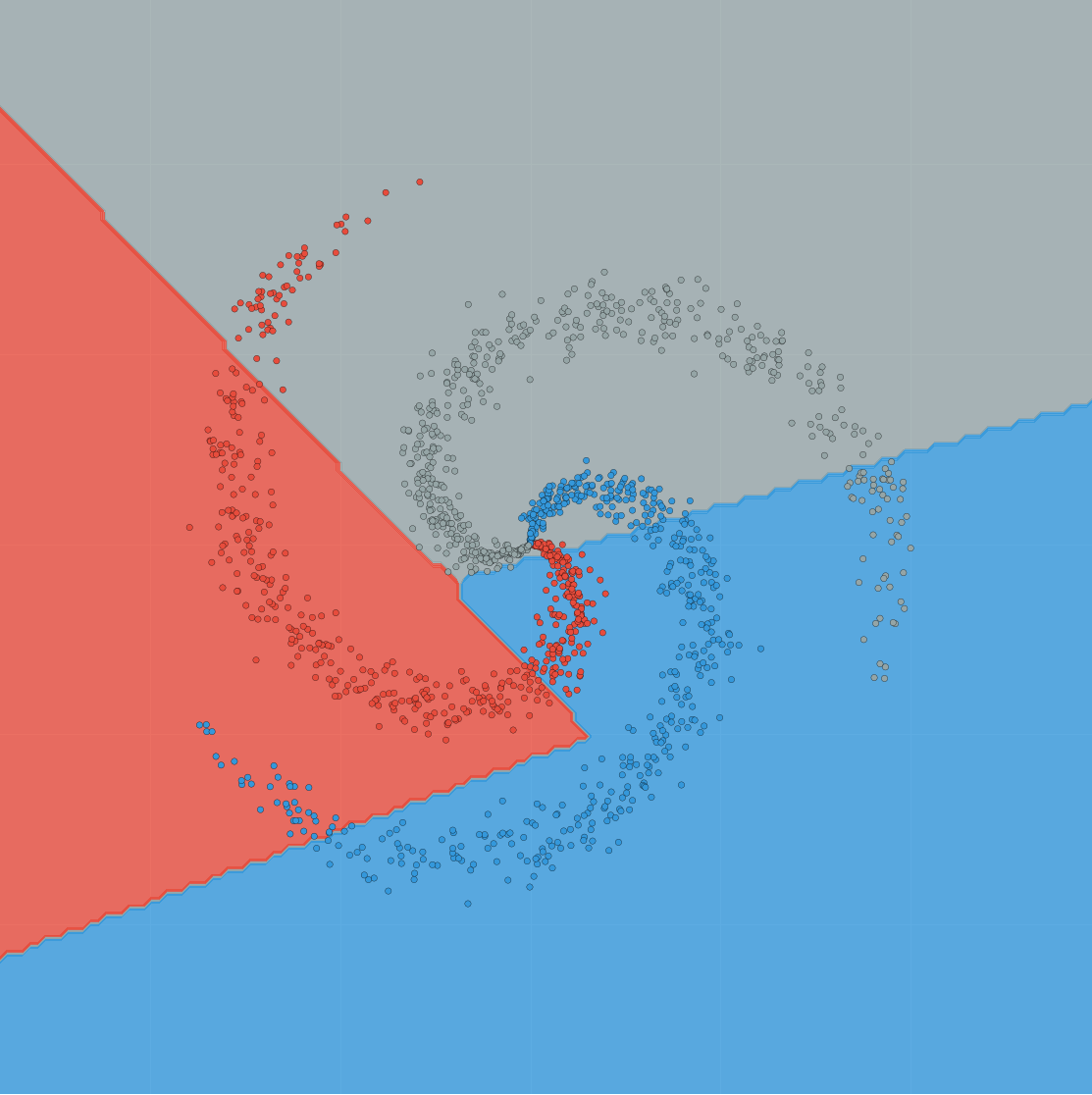}
~
\includegraphics[scale=0.08]{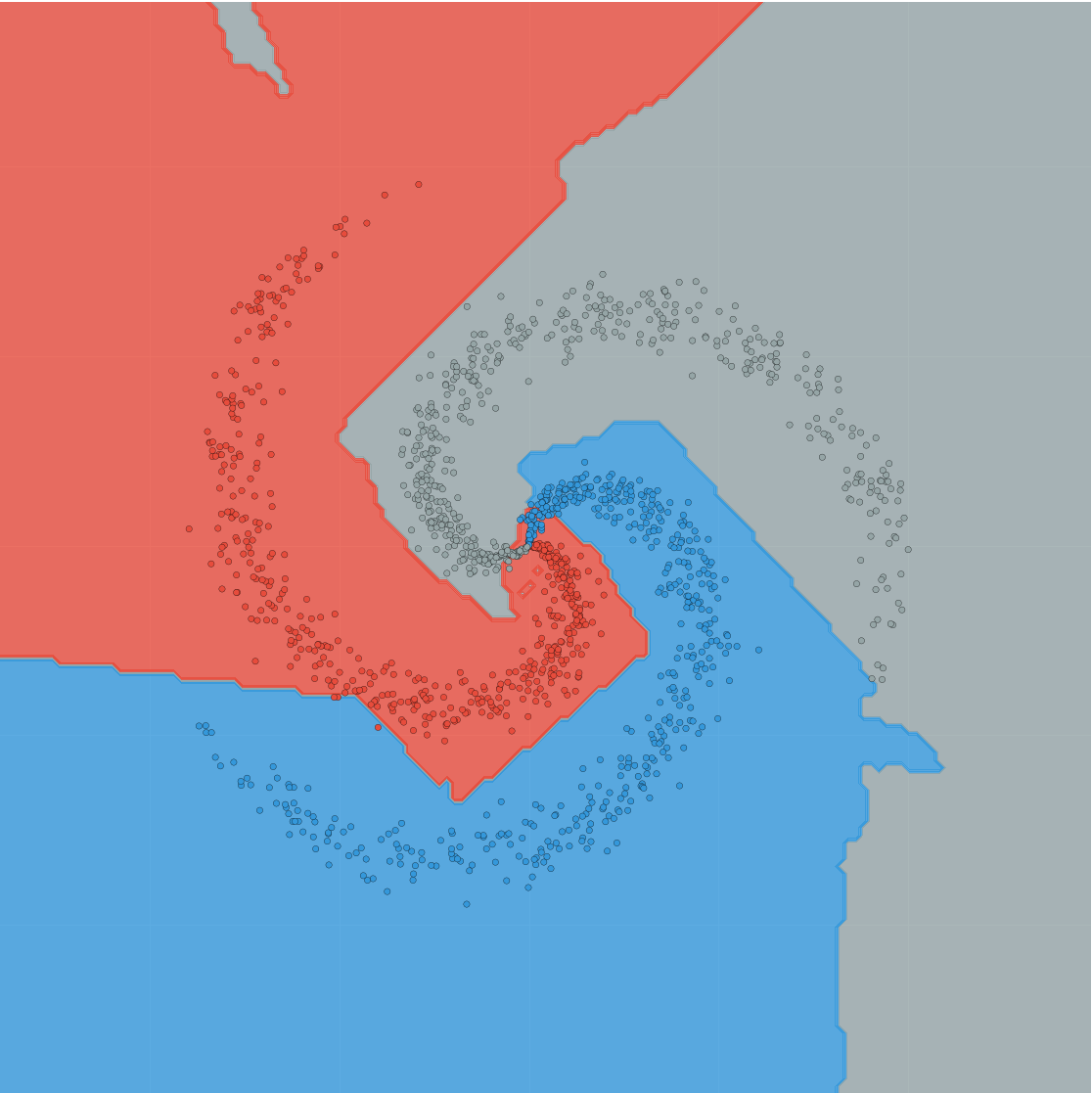}
~
\includegraphics[scale=0.08]{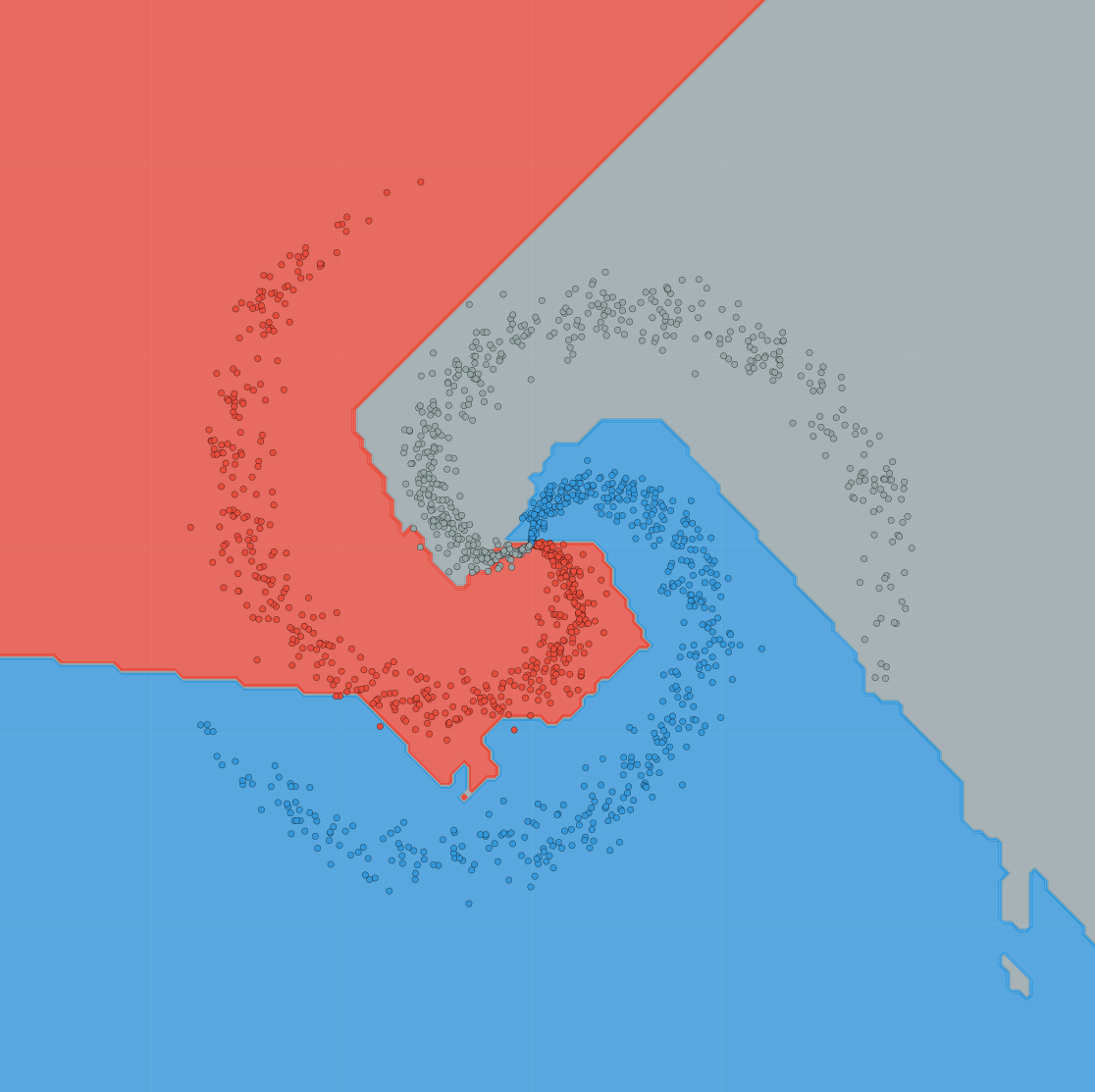}
~
\includegraphics[scale=0.08]{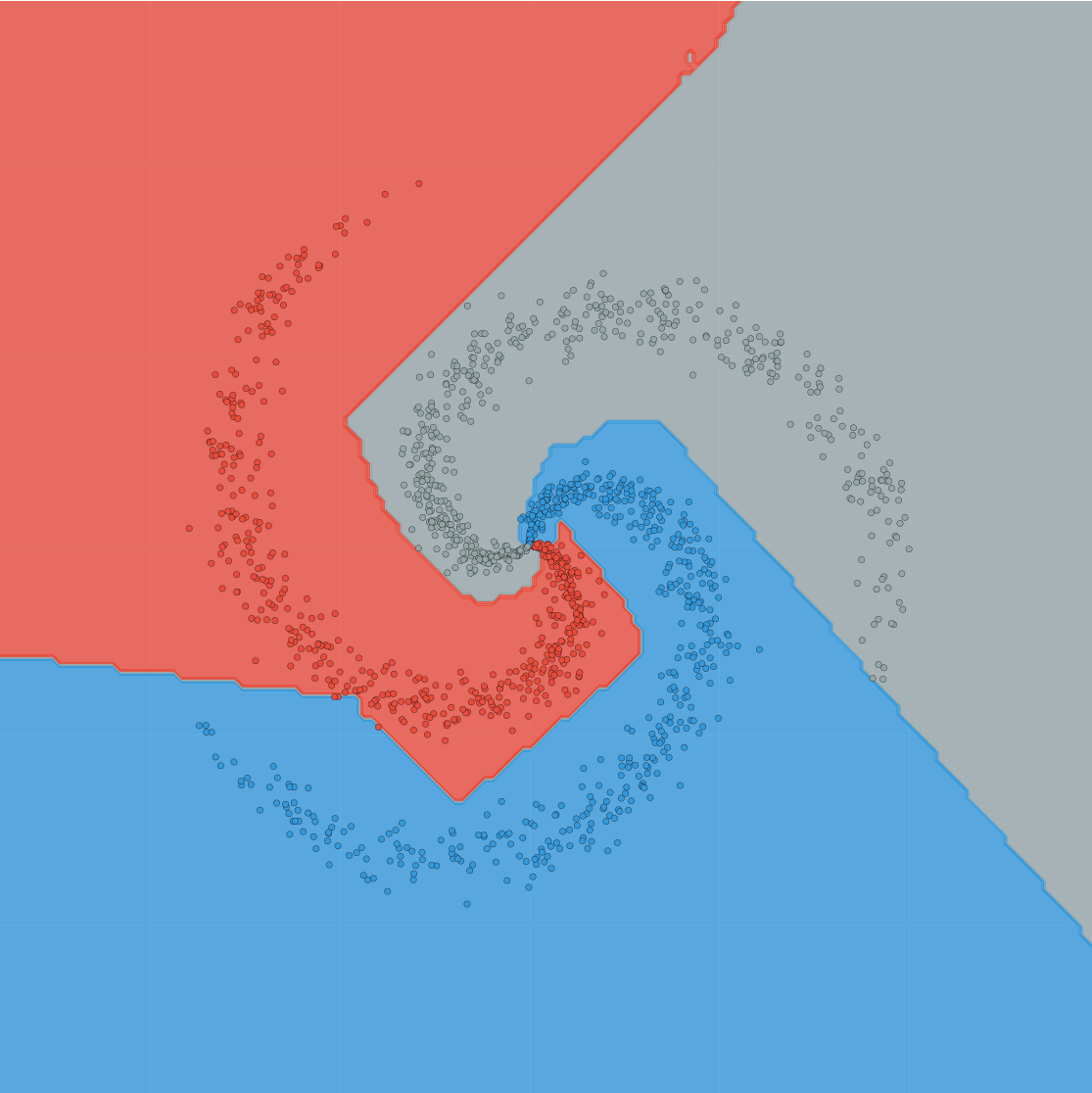}
\caption{From left to right, the decision boundary determined by a neuron on the bottom layer, the middle layer, and the final layer, and the top-level switching mixture. The fit of the bottommost neuron is quite poor. The middle neuron improves significantly, making mistakes only at the center of the spiral. The topmost neuron has learnt to fit the data almost perfectly. The final decision boundary of the switching ensemble resembles a smoothed version of the deepest neuron.}
\label{pic:decision_boundaries}
\end{figure}

\subsection{Online Half-space Classification: MNIST}

Next we explore the use of GLNs for classification on the well known MNIST dataset \citep{Lecun98}. 
Once again we used an ensemble of 10 GLNs to construct a one-vs-all classifier.
Each member of the ensemble was a 3 layer network consisting of 1500-1500-1 neurons, each of which used 6 half-space context functions (using the context composition technnique of Section \ref{sec:context_composition}) as the gating procedure, meaning that each neuron contained 64 different possible weight vectors. 
The half-space contexts for each neuron were determined by sampling a $784$ dimensional normal vector whose components were distributed according to $\cN(0,0.01)$, and a bias weight of $0$.
The learning rate for an example at time $t$ was set to $\min\{8000/t, 0.3\}$.
These parameters were determined by a grid search across the training data.
Each component of all weight vectors were constrained to lie within $[-200, 200]$.
The input features were pre-processed by first applying mean-subtraction and a de-skewing operation \citep{deskew}.

Running the method purely online across a single pass of the data gives an accuracy on the test set of 98.3\%.
If weight updating during the testing set is disabled, the test set accuracy drops slightly to 98.1\%.
Without the de-skewing operation, the accuracy drops to 96.9\%.
It is worth nothing that our classifier contains no image specific domain knowledge; in fact, one could even apply a (fixed across all images) permutation to each image input and get exactly the same result.

\subsection{Online Density Modeling: MNIST}

Our final result is to use GLNs and image specific gating to construct an online image density model for the binarized MNIST dataset \citep{larochelle11a}, a standard benchmark for image density modeling.
By exploiting the chain rule $$\mathbb{P}(X_1, X_2, \dots, X_d)=\prod_{i=1}^{d} \mathbb{P}(X_i \, | \, X_{<i})$$ of probability, we constructed a so-called autoregressive density model over the $28\times28$ dimensional binary space by using 784 GLNs to model the conditional distribution for each pixel; a row-major ordering was used to linearize the two dimensional pixel locations.

\paragraph{Base Layer}

The base layer for each GLN was determined by taking the outputs of a number of skip-gram models \citep{Guthrie06}.
The context functions used by these skip-gram models included a number of specific geometric patterns that are known to be helpful for lossless image compression \citep{Witten1999}, plus a large number of randomly sampled pixel locations.
\begin{figure}[h!]
\centering
\includegraphics[scale=0.08]{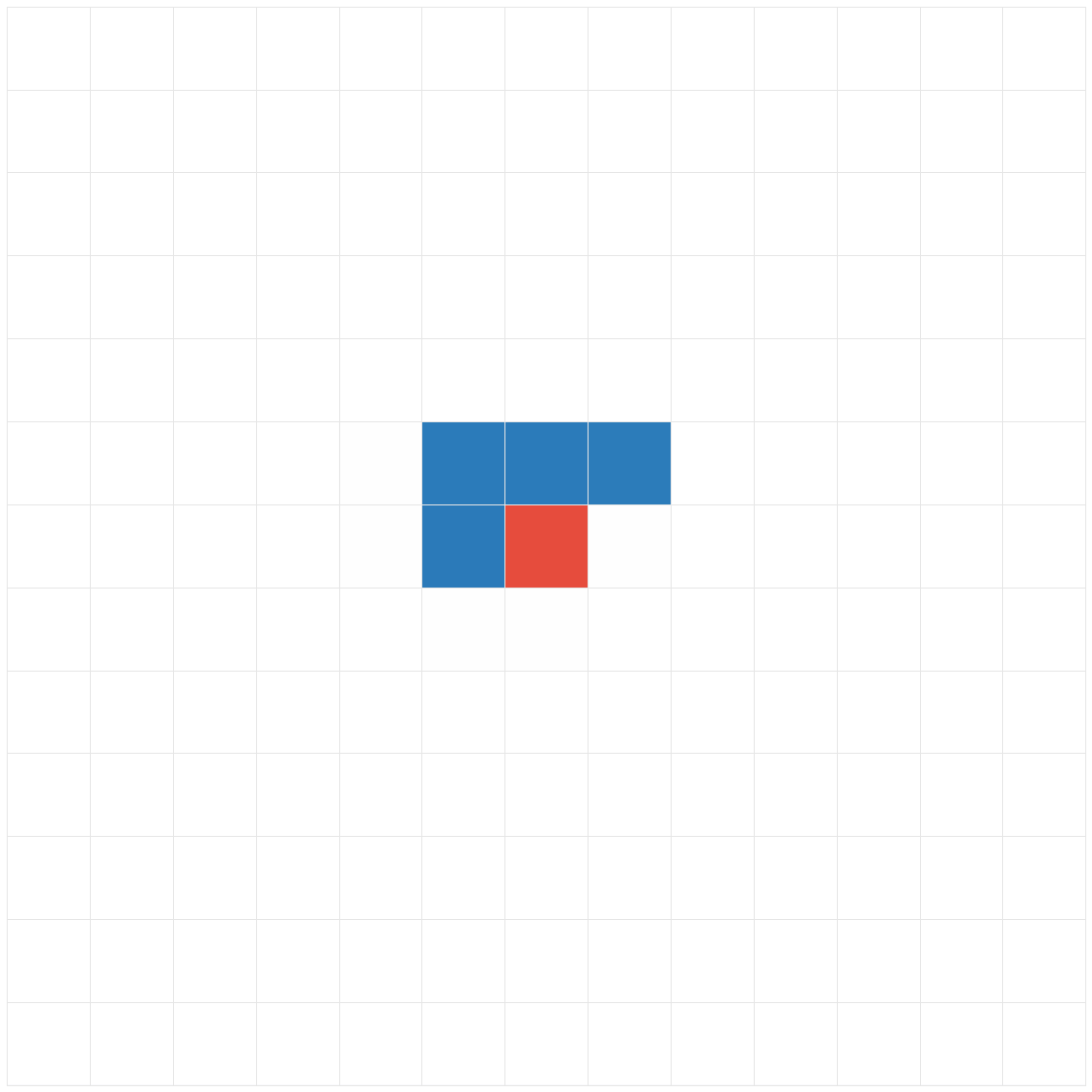}
~
\includegraphics[scale=0.08]{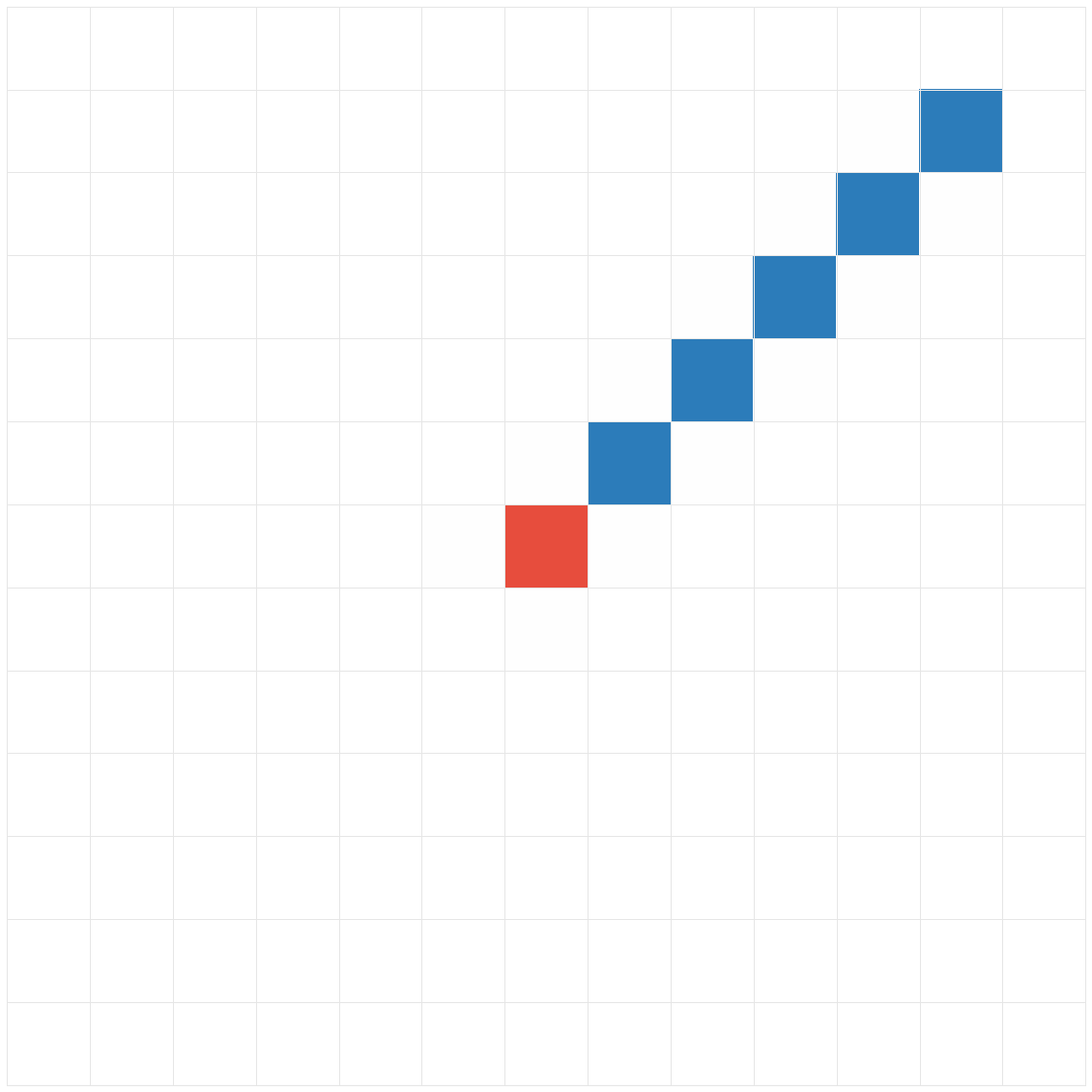}
~
\includegraphics[scale=0.08]{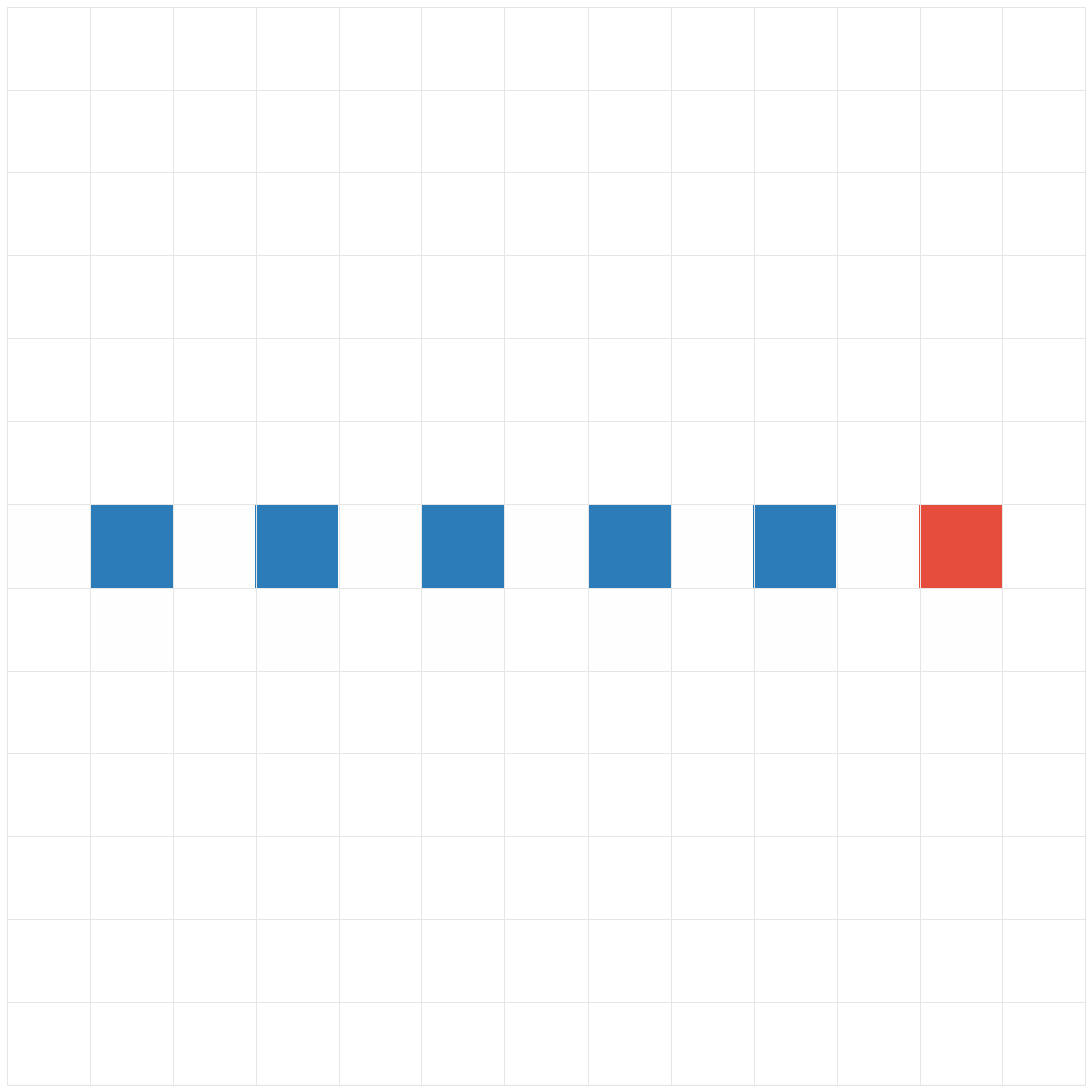}
~
\includegraphics[scale=0.08]{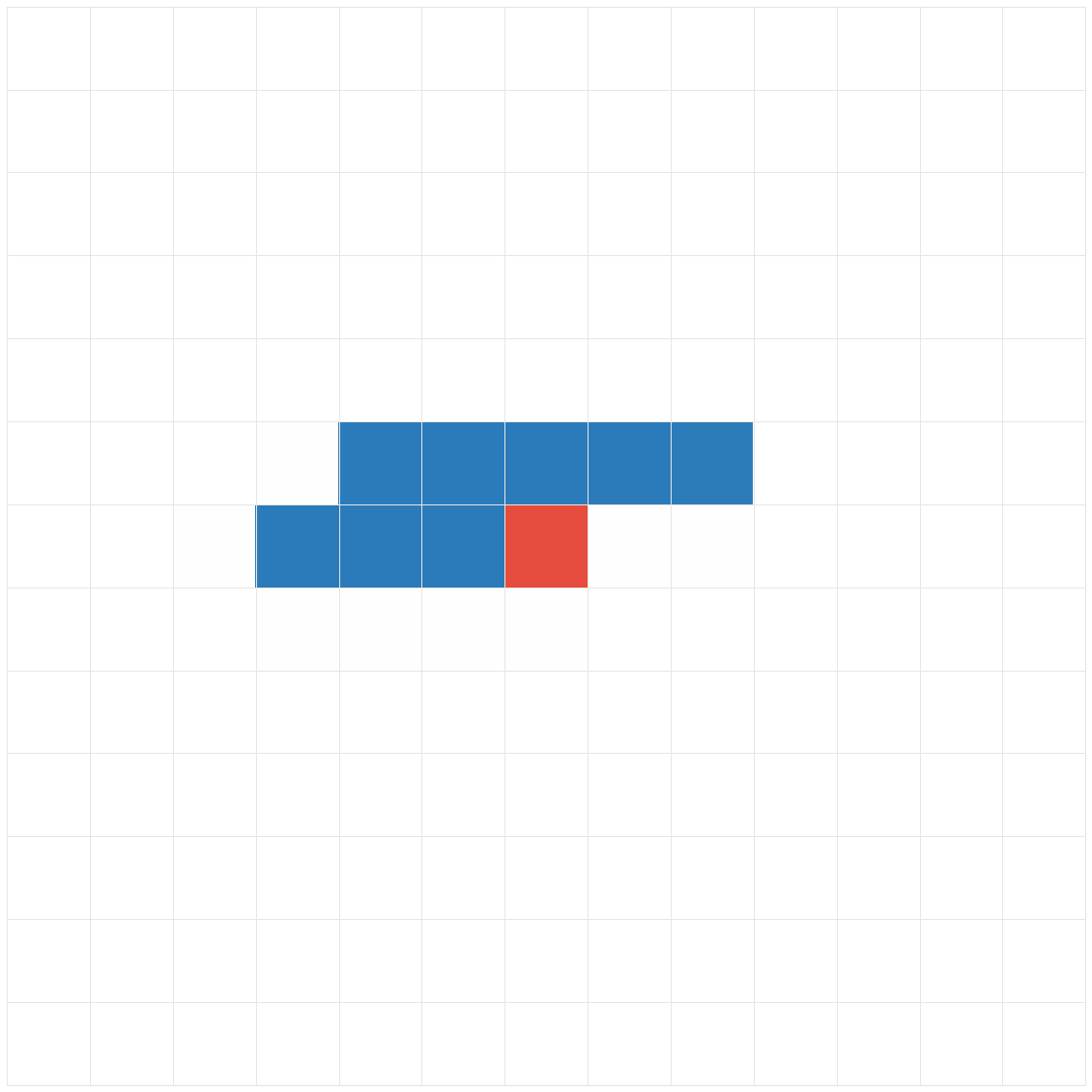}
\caption{Example skip-gram context functions. }
\label{pic:skip_gram_ctx}
\end{figure}
Some example skip-grams are shown in Figure \ref{pic:skip_gram_ctx}; the red box indicates the location of the pixel in the image, and the blue boxes indicate the pixels which are used to determine the skip-gram context.
For each pixel model, the base layer would consist of up to 600 skip-gram predictions; the exact number of predictions depends on the location of the pixel in the linear ordering; the set of permitted skip-gram models for pixel $i$ are those whose context functions do not depend on pixels greater than or equal $i$ within the linear ordering.
The Zero-Redundancy estimator \citep{Willems97} was used to estimate the per-context skip-gram probability estimates online. 

\paragraph{Image specific context functions}

Empirically we found that we could significantly boost the performance of our method by introducing two simple kinds of image specific context functions.

The first were \emph{max-pool} context functions, some examples of which are depicted in Figure \ref{pic:max_pool_ctx}.
Once again the red square denotes the location of the current pixel.
For each shaded region of a particular colour, the maximum of all binary pixel values was computed.
The max-pool context function returns the number represented by the binary representation obtained by concatenating these max-pooled values together (in some fixed order).

\begin{figure}[h!]
\centering
\includegraphics[scale=0.08]{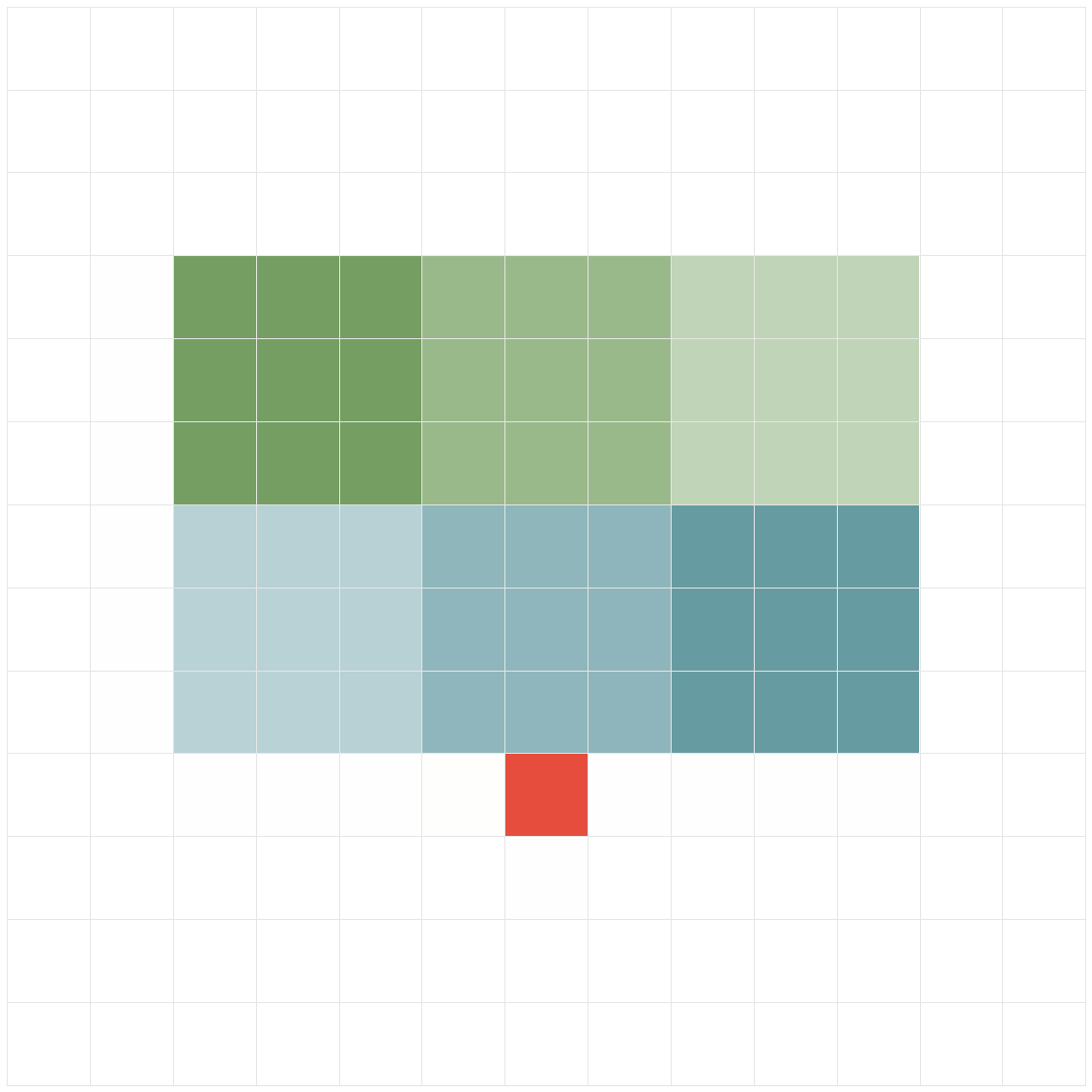}
~
\includegraphics[scale=0.08]{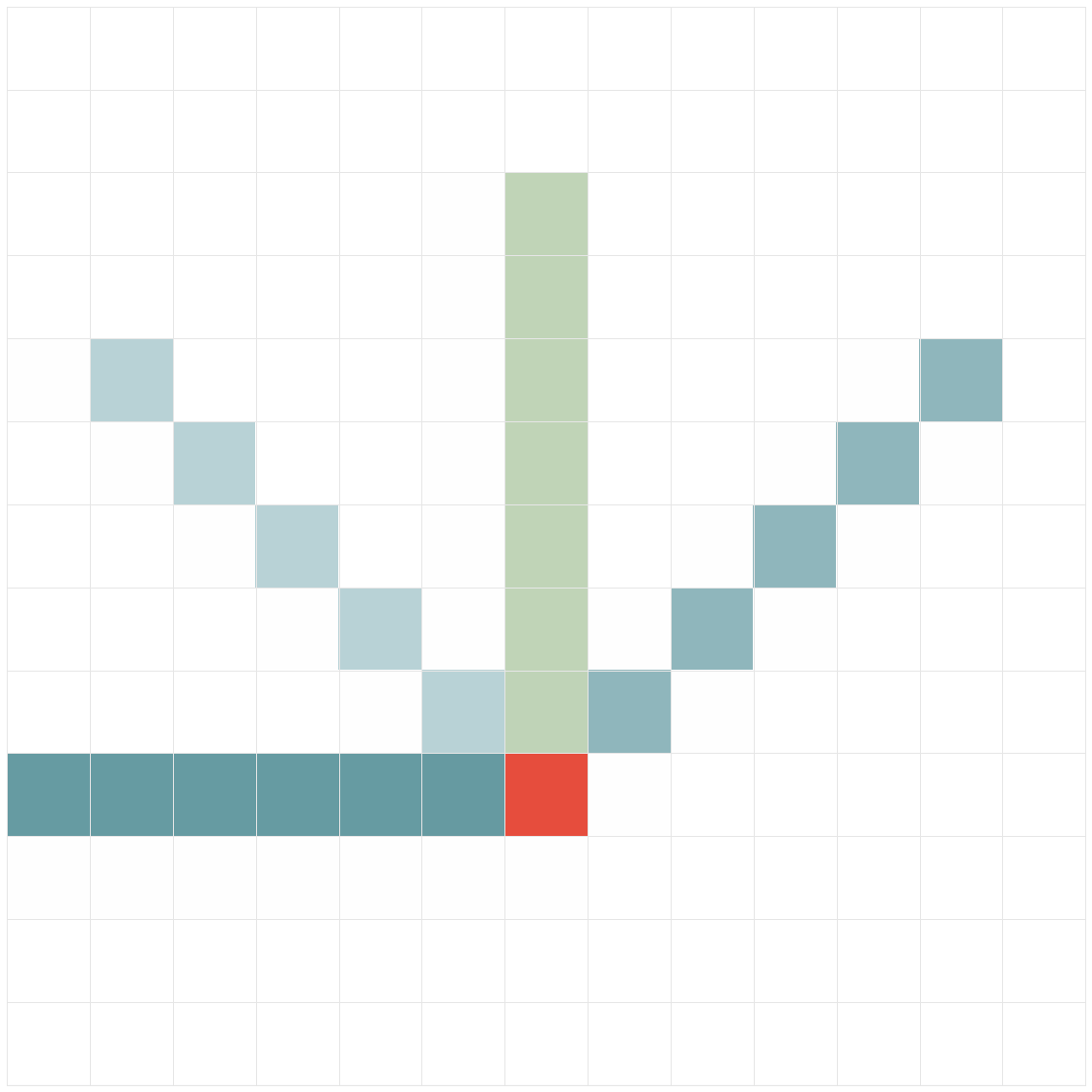}
~
\includegraphics[scale=0.08]{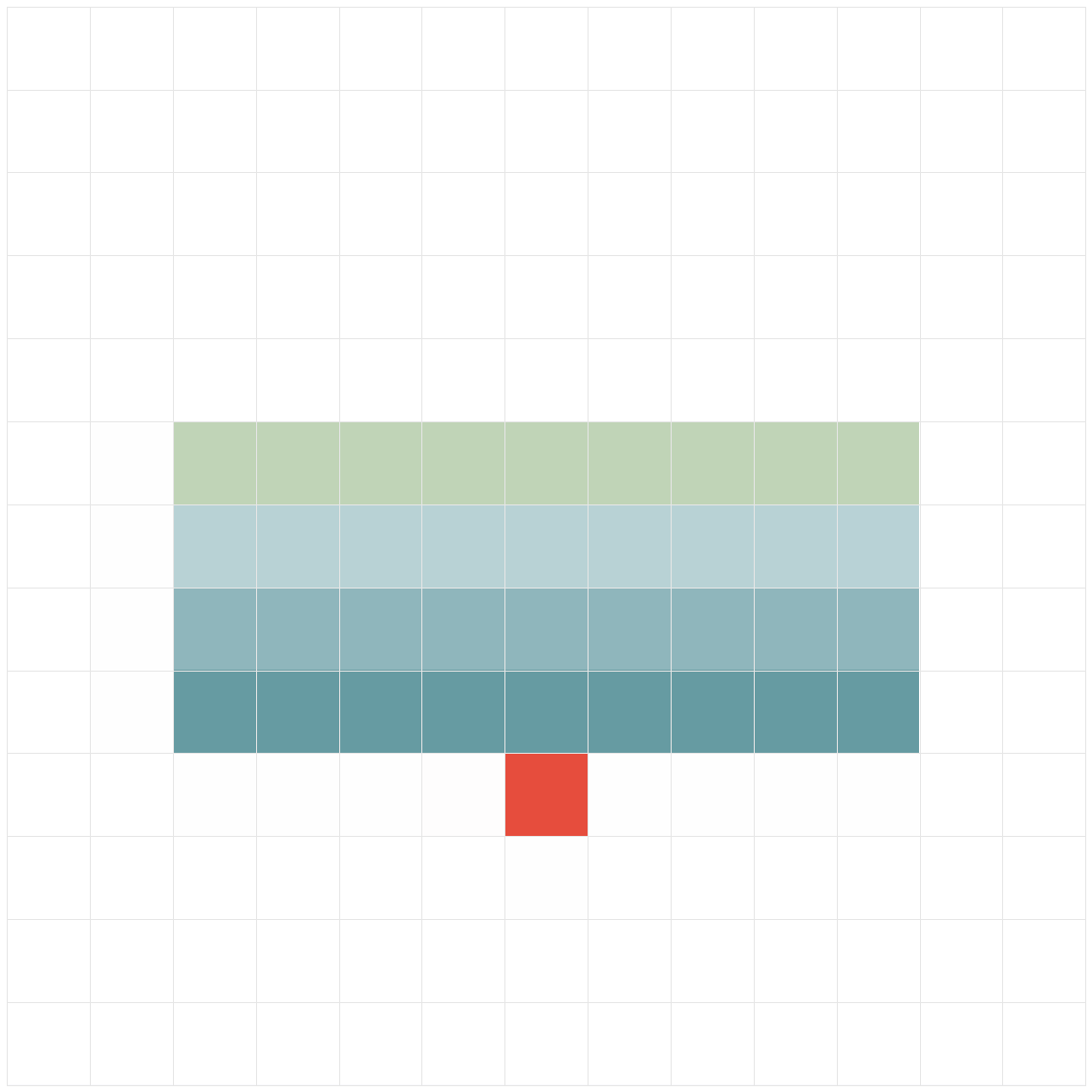}
~
\includegraphics[scale=0.08]{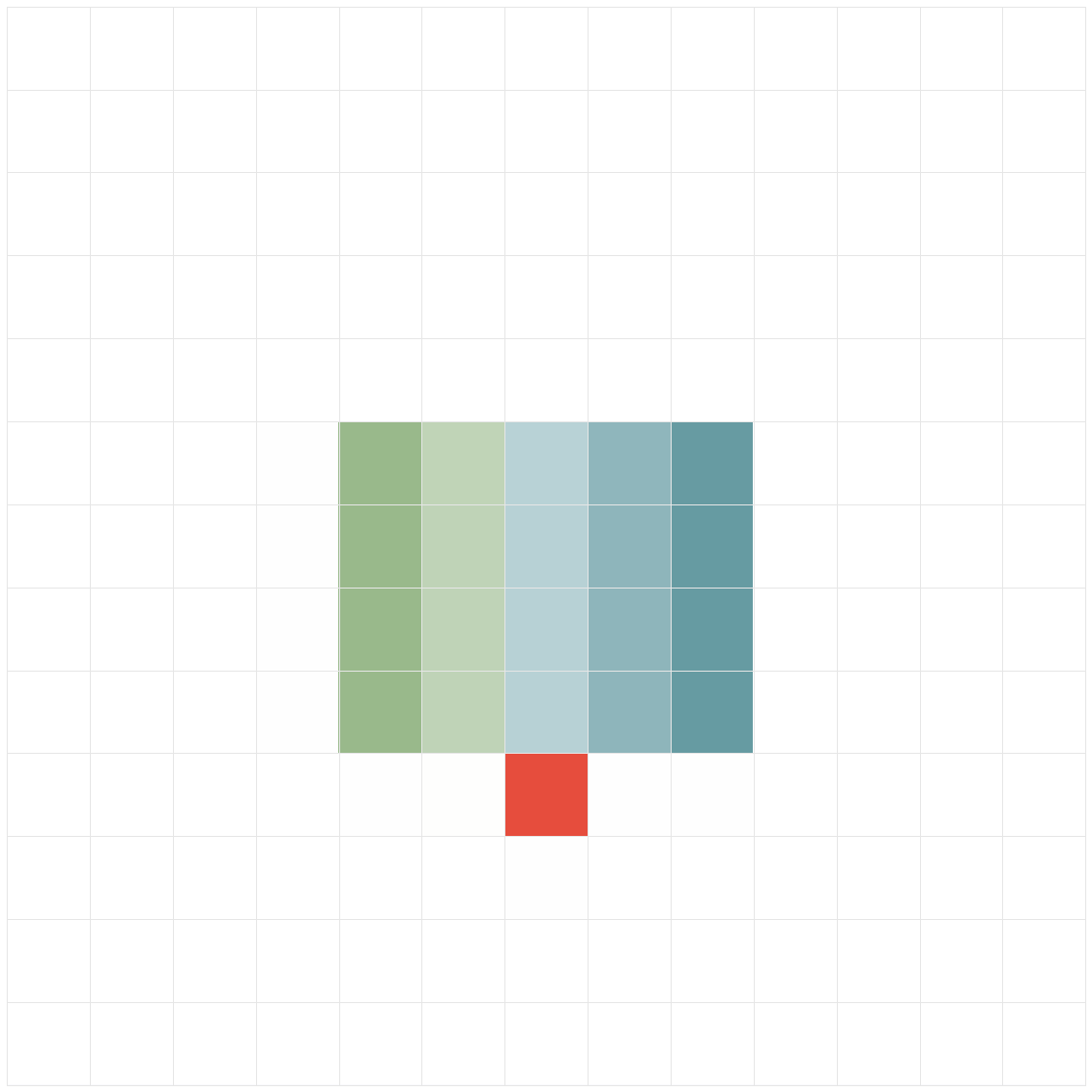}
\caption{Example max-pool context functions. }
\label{pic:max_pool_ctx}
\end{figure}

The next type of context we introduced were \emph{distance} context functions, some examples of which are depicted in Figure \ref{pic:distance_ctx}.
Once again the red square denotes the location of the current pixel.
Each non-white/non-red location is labeled with an index between 1 and the number of non-white/non-red locations; darker colours indicate smaller index values.
The distance context function returns the smallest index value whose pixel value is equal to 1.
This class of context allows one to measure distance to the nearest active pixel under various orderings of local pixel locations.

\begin{figure}[h!]
\centering
\includegraphics[scale=0.08]{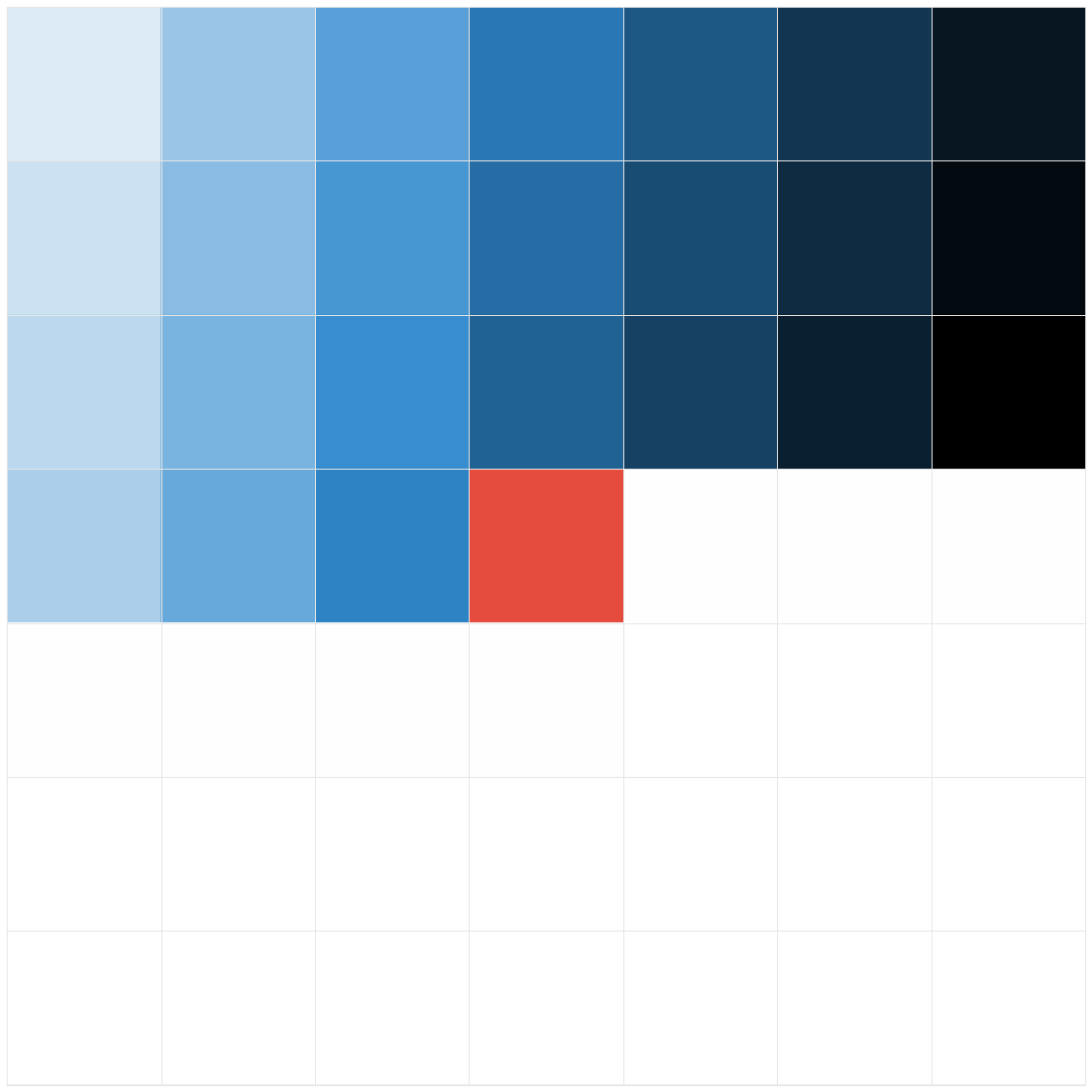}
~
\includegraphics[scale=0.08]{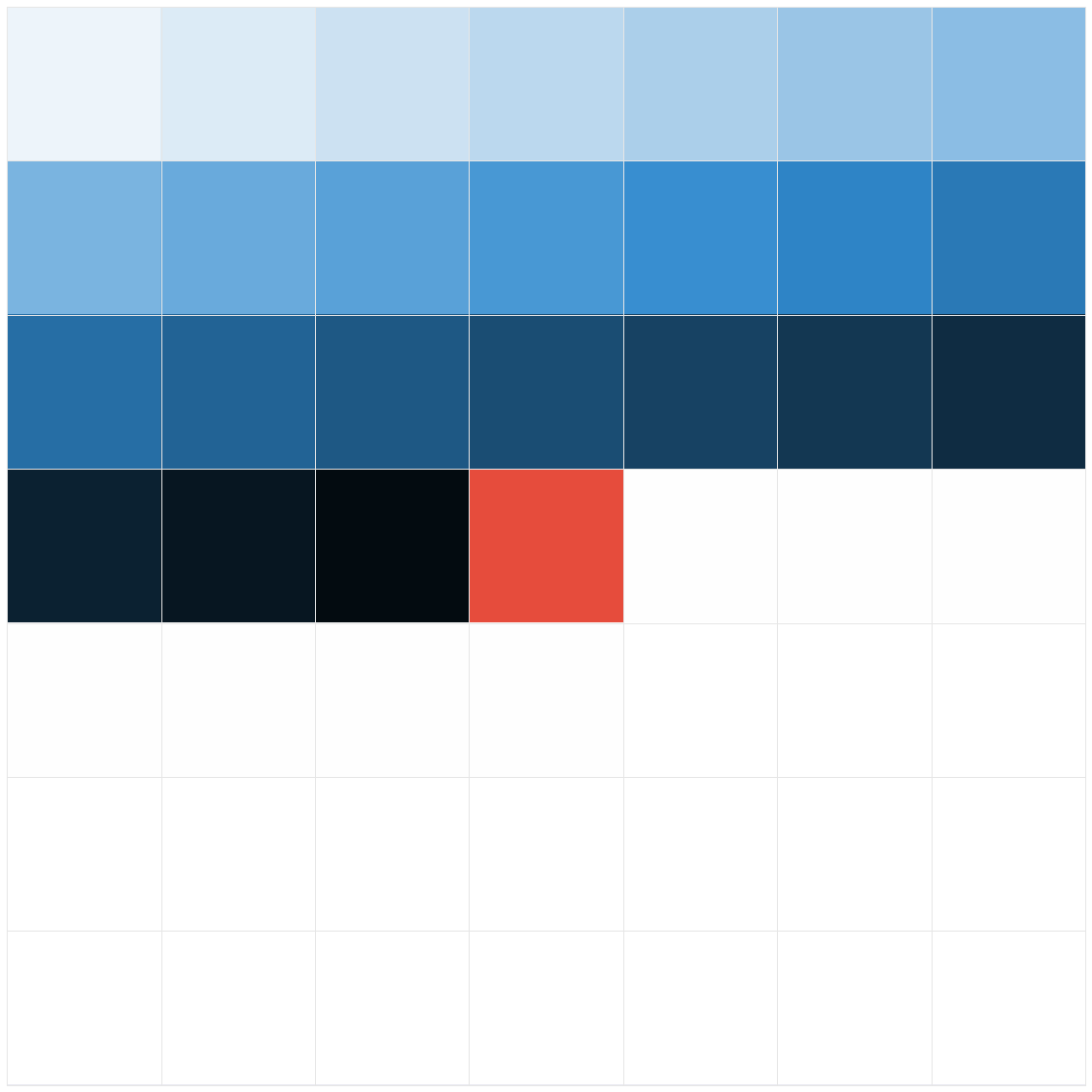}
~
\includegraphics[scale=0.08]{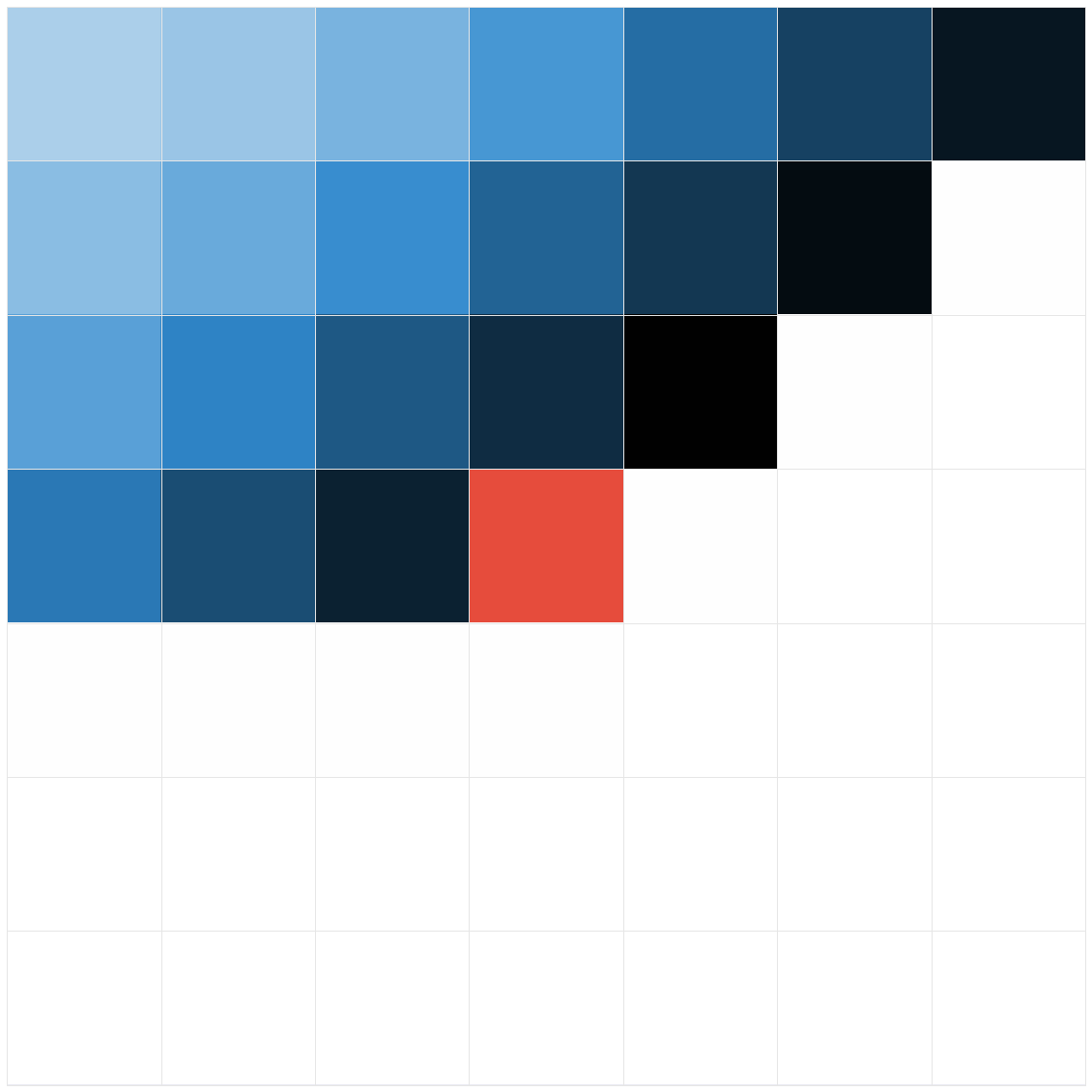}
~
\includegraphics[scale=0.08]{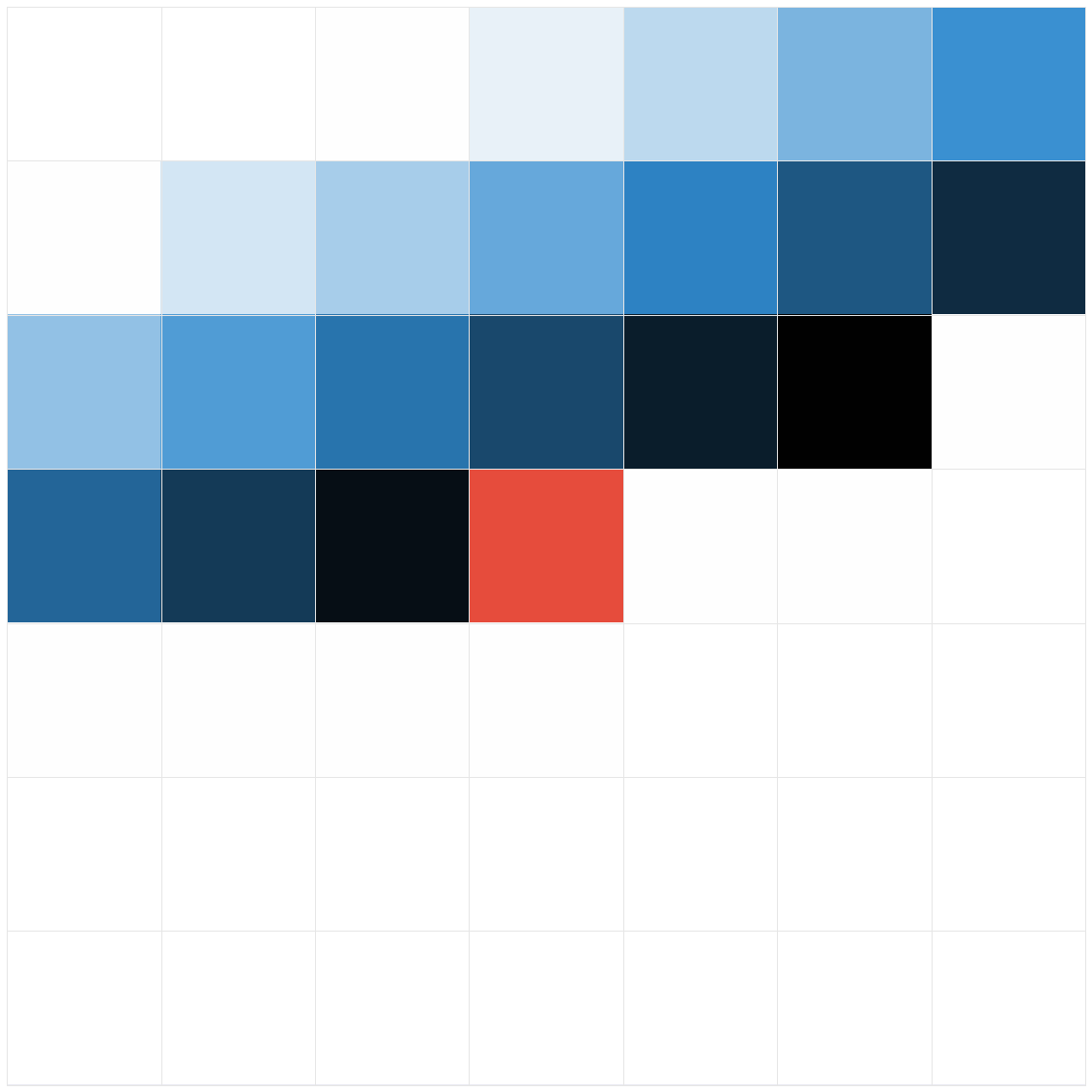}
\caption{Example distance context functions. }
\label{pic:distance_ctx}
\end{figure}

\paragraph{Network construction}

Each GLN used 4-layers, with the gating function for each neuron determined by a choice of skip-gram/max-pool/distance context function.
A set of 200 different context functions were defined, and these were randomly distributed across a network architecture whose shape was 35-60-35-70.
The learning rate for an example seen at time $t$ was set to $\min\{25/t, 0.005\}$.

\paragraph{Results}

Running the method purely online across a single pass of the data (we concatenated the training, validation and test sets into one contiguous stream of data) gave an average loss of $79.0$ nats per image across the test data.
To the best of our knowledge, this result matches the state of the art \citep{VanDenOord2016} of any batch trained density model which outputs exact probabilities, and is close to the estimated upper bounds reported for state of the art variational approaches \citep{bachman2016}.

\section{Relationship to the PAQ family compressors}

One of the main contributions of this paper is to justify and explain the empirical success of the PAQ family of compression programs.
At the time of writing, the best performing PAQ implementation in terms of general purpose compression performance is {\sc cmix}, an open-source project whose details are described by \cite{cmix} on his personal website.
{\sc cmix} uses a mixing network comprised of many gated geometric mixers, each of which use different skip-gram context functions to define the gating operations. 
While many of the core building blocks have been analyzed previously by \cite{Mattern16}, the reason for the empirical success of such locally trained mixing networks had until now remained somewhat of a mystery.
Our work shows that such architectures are a special cases of Gated Linear Networks, and that future improvements may result by exploring alternate no-regret learning methods, by using a broader class of context functions or by using our gated linear network formulation to enable efficient implementation on GPUs via vectorized matrix operations.
In principle the universality of GLNs suggests that the performance of PAQ-like approaches will continue to scale as hardware improves.
Furthermore, our MNIST density modeling results suggest that such methods can be competitive with state of the art deep learning approaches on domains beyond general purpose file compression.

\section{Conclusion}

We have introduced Gated Linear Networks, a family of architectures designed for online learning under the logarithmic loss.
Under appropriate conditions, such architectures are universal in the sense that they can can model any bounded Borel-measurable function; more significantly, 
they are guaranteed to find this solution using any appropriate choice of no-regret online optimization algorithm to locally adapt the weights of each neuron.
Initial experimental results suggest that the method warrants further investigation.

\acks{Thanks to A{\"a}ron van den Oord, Byron Knoll, Malcolm Reynolds, Kieran Milan, Simon Schmitt, Nando de Freitas, Chrisantha Fernando, Charles Blundell, Ian Osband, David Budden, Michael Bowling, Shane Legg and Demis Hassabis.}

\bibliography{xinet}

\begin{thebibliography}{50}
\providecommand{\natexlab}[1]{#1}
\providecommand{\url}[1]{\texttt{#1}}
\expandafter\ifx\csname urlstyle\endcsname\relax
  \providecommand{\doi}[1]{doi: #1}\else
  \providecommand{\doi}{doi: \begingroup \urlstyle{rm}\Url}\fi

\bibitem[Bachman(2016)]{bachman2016}
Philip Bachman.
\newblock An architecture for deep, hierarchical generative models.
\newblock In D.~D. Lee, M.~Sugiyama, U.~V. Luxburg, I.~Guyon, and R.~Garnett,
  editors, \emph{Advances in Neural Information Processing Systems 29}, pages
  4826--4834. Curran Associates, Inc., 2016.

\bibitem[Baldi and Hornik(1989)]{Baldi1989}
P.~Baldi and K.~Hornik.
\newblock Neural networks and principal component analysis: Learning from
  examples without local minima.
\newblock \emph{Neural Networks}, 2\penalty0 (1):\penalty0 53--58, January
  1989.
\newblock ISSN 0893-6080.
\newblock \doi{10.1016/0893-6080(89)90014-2}.

\bibitem[Balduzzi(2016)]{Balduzzi16}
David Balduzzi.
\newblock Deep online convex optimization with gated games.
\newblock \emph{CoRR}, abs/1604.01952, 2016.
\newblock URL \url{http://arxiv.org/abs/1604.01952}.

\bibitem[Bell et~al.(1990)Bell, Cleary, and Witten]{bell90}
T.~C. Bell, J.~G. Cleary, and I.~H. Witten.
\newblock \emph{Text Compression}.
\newblock Prentice Hall, Englewood Cliffs, NJ, 1990.

\bibitem[Bell and Arnold(1997)]{bell97}
Tim Bell and Ross Arnold.
\newblock A corpus for the evaluation of lossless compression algorithms.
\newblock \emph{Data Compression Conference}, 00:\penalty0 201, 1997.
\newblock ISSN 1068-0314.
\newblock \doi{doi.ieeecomputersociety.org/10.1109/DCC.1997.582019}.

\bibitem[Cesa-Bianchi and Lugosi(2006)]{Cesa-Bianchi2006}
Nicolo Cesa-Bianchi and Gabor Lugosi.
\newblock \emph{Prediction, Learning, and Games}.
\newblock Cambridge University Press, New York, NY, USA, 2006.
\newblock ISBN 0521841089.

\bibitem[Chen et~al.(2016)Chen, Kingma, Salimans, Duan, Dhariwal, Schulman,
  Sutskever, and Abbeel]{ChenKSDDSSA16}
Xi~Chen, Diederik~P. Kingma, Tim Salimans, Yan Duan, Prafulla Dhariwal, John
  Schulman, Ilya Sutskever, and Pieter Abbeel.
\newblock Variational lossy autoencoder.
\newblock \emph{CoRR}, abs/1611.02731, 2016.

\bibitem[Cormack and Horspool(1987)]{Cormack87}
G.~V. Cormack and R.~N.~S. Horspool.
\newblock Data compression using dynamic markov modelling.
\newblock \emph{The Computer Journal}, 30\penalty0 (6), 1987.

\bibitem[Foerster et~al.(2017)Foerster, Gilmer, Sohl-Dickstein, Chorowski, and
  Sussillo]{foerster17a}
Jakob~N. Foerster, Justin Gilmer, Jascha Sohl-Dickstein, Jan Chorowski, and
  David Sussillo.
\newblock Input switched affine networks: An {RNN} architecture designed for
  interpretability.
\newblock In Doina Precup and Yee~Whye Teh, editors, \emph{Proceedings of the
  34th International Conference on Machine Learning}, volume~70 of
  \emph{Proceedings of Machine Learning Research}, pages 1136--1145,
  International Convention Centre, Sydney, Australia, 06--11 Aug 2017. PMLR.

\bibitem[Genest and Zidek(1986)]{Genest86}
Christian Genest and James~V. Zidek.
\newblock Combining probability distributions: A critique and an annotated
  bibliography.
\newblock \emph{Statistical Science}, 1\penalty0 (1):\penalty0 114--135, 1986.

\bibitem[Ghosh and Wan(2017)]{deskew}
Dibya Ghosh and Alvin Wan, 2017.
\newblock URL \url{https://fsix.github.io/mnist/Deskewing.html}.

\bibitem[Gregor et~al.(2015)Gregor, Danihelka, Graves, Rezende, and
  Wierstra]{icml2015_gregor15}
Karol Gregor, Ivo Danihelka, Alex Graves, Danilo Rezende, and Daan Wierstra.
\newblock Draw: A recurrent neural network for image generation.
\newblock In David Blei and Francis Bach, editors, \emph{Proceedings of the
  32nd International Conference on Machine Learning (ICML-15)}, pages
  1462--1471. JMLR Workshop and Conference Proceedings, 2015.

\bibitem[Gulrajani et~al.(2016)Gulrajani, Kumar, Ahmed, Taiga, Visin,
  V{\'{a}}zquez, and Courville]{GulrajaniKATVVC16}
Ishaan Gulrajani, Kundan Kumar, Faruk Ahmed, Adrien~Ali Taiga, Francesco Visin,
  David V{\'{a}}zquez, and Aaron~C. Courville.
\newblock Pixelvae: {A} latent variable model for natural images.
\newblock \emph{CoRR}, abs/1611.05013, 2016.

\bibitem[Guthrie et~al.(2006)Guthrie, Allison, Liu, Guthrie, and
  Wilks]{Guthrie06}
David Guthrie, Ben Allison, W.~Liu, Louise Guthrie, and Yorick Wilks.
\newblock A closer look at skip-gram modelling.
\newblock In \emph{Proceedings of the Fifth international Conference on
  Language Resources and Evaluation (LREC-2006)}, Genoa, Italy, 2006.

\bibitem[Hazan(2016)]{Hazan16}
Elad Hazan.
\newblock Introduction to online convex optimization.
\newblock \emph{Foundations and Trends in Optimization}, 2\penalty0
  (3-4):\penalty0 157--325, 2016.
\newblock \doi{10.1561/2400000013}.

\bibitem[Hazan et~al.(2007)Hazan, Agarwal, and Kale]{HazanAK07}
Elad Hazan, Amit Agarwal, and Satyen Kale.
\newblock Logarithmic regret algorithms for online convex optimization.
\newblock \emph{Machine Learning}, 69\penalty0 (2-3):\penalty0 169--192, 2007.

\bibitem[Helgason(2011)]{Hel11}
Sigurdur Helgason.
\newblock The radon transform on r n.
\newblock In \emph{Integral Geometry and Radon Transforms}, pages 1--62.
  Springer, 2011.

\bibitem[Herbster and Warmuth(1998)]{Herbster1998}
Mark Herbster and Manfred~K. Warmuth.
\newblock Tracking the best expert.
\newblock \emph{Machine Learning}, 32\penalty0 (2):\penalty0 151--178, August
  1998.
\newblock ISSN 0885-6125.

\bibitem[Hinton(2002)]{hinton2002}
Geoffrey~E. Hinton.
\newblock Training products of experts by minimizing contrastive divergence.
\newblock \emph{Neural Computation}, 14\penalty0 (8):\penalty0 1771--1800,
  August 2002.
\newblock ISSN 0899-7667.

\bibitem[Hornik(1991)]{Hor91}
Kurt Hornik.
\newblock Approximation capabilities of multilayer feedforward networks.
\newblock \emph{Neural networks}, 4\penalty0 (2):\penalty0 251--257, 1991.

\bibitem[Hu and Shing(1982)]{hu1982}
T.~C. Hu and M.~T. Shing.
\newblock Computation of matrix chain products. part i.
\newblock \emph{SIAM Journal on Computing}, 11\penalty0 (2):\penalty0 362--373,
  1982.
\newblock \doi{10.1137/0211028}.

\bibitem[Hutter(2017)]{hutterprize}
Marcus Hutter.
\newblock Hutter prize, 2017.
\newblock URL \url{http://prize.hutter1.net/}.

\bibitem[Jacobs et~al.(1991)Jacobs, Jordan, Nowlan, and Hinton]{Jacobs1991}
Robert~A. Jacobs, Michael~I. Jordan, Steven~J. Nowlan, and Geoffrey~E. Hinton.
\newblock Adaptive mixtures of local experts.
\newblock \emph{Neural Comput.}, 3\penalty0 (1):\penalty0 79--87, 1991.
\newblock ISSN 0899-7667.
\newblock \doi{10.1162/neco.1991.3.1.79}.

\bibitem[Jun et~al.(2017)Jun, Orabona, Wright, and Willett]{JOWW17}
Kwang-Sung Jun, Francesco Orabona, Stephen Wright, and Rebecca Willett.
\newblock Improved strongly adaptive online learning using coin betting.
\newblock In \emph{Artificial Intelligence and Statistics}, pages 943--951,
  2017.

\bibitem[Knoll and de~Freitas(2012)]{knoll12}
B.~Knoll and N.~de~Freitas.
\newblock A machine learning perspective on predictive coding with paq8.
\newblock In \emph{Data Compression Conference (DCC)}, pages 377--386, April
  2012.

\bibitem[Knoll(2017)]{cmix}
Byron Knoll, 2017.
\newblock URL \url{http://www.byronknoll.com/cmix.html}.

\bibitem[Larochelle and Murray(2011{\natexlab{a}})]{larochelle11a}
Hugo Larochelle and Iain Murray.
\newblock The neural autoregressive distribution estimator.
\newblock In Geoffrey Gordon, David Dunson, and Miroslav Dudík, editors,
  \emph{Proceedings of the Fourteenth International Conference on Artificial
  Intelligence and Statistics}, volume~15 of \emph{Proceedings of Machine
  Learning Research}, pages 29--37, Fort Lauderdale, FL, USA, 11--13 Apr
  2011{\natexlab{a}}. PMLR.

\bibitem[Larochelle and Murray(2011{\natexlab{b}})]{larochelle2011}
Hugo Larochelle and Iain Murray.
\newblock The neural autoregressive distribution estimator.
\newblock \emph{Journal of Machine Learning Research (JMLR)}, 15:\penalty0
  29--37, 2011{\natexlab{b}}.

\bibitem[Lecun et~al.(1998)Lecun, Bottou, Bengio, and Haffner]{Lecun98}
Yann Lecun, Léon Bottou, Yoshua Bengio, and Patrick Haffner.
\newblock Gradient-based learning applied to document recognition.
\newblock In \emph{Proceedings of the IEEE}, pages 2278--2324, 1998.

\bibitem[Luo et~al.(2016)Luo, Agarwal, Cesa{-}Bianchi, and Langford]{LuoACL16}
Haipeng Luo, Alekh Agarwal, Nicol{\`{o}} Cesa{-}Bianchi, and John Langford.
\newblock Efficient second order online learning by sketching.
\newblock In \emph{Advances in Neural Information Processing Systems 29: Annual
  Conference on Neural Information Processing Systems 2016, December 5-10,
  2016, Barcelona, Spain}, pages 902--910, 2016.

\bibitem[Mahoney(2000)]{Mahoney2000}
Matthew Mahoney.
\newblock Fast text compression with neural networks.
\newblock \emph{AAAI}, 2000.

\bibitem[Mahoney(2005)]{Mahoney2005}
Matthew Mahoney.
\newblock Adaptive weighing of context models for lossless data compression.
\newblock \emph{Technical Report, Florida Institute of Technology CS}, 2005.

\bibitem[Mahoney(2013)]{Mahoney2013}
Matthew Mahoney.
\newblock \emph{Data Compression Explained}.
\newblock 2013.

\bibitem[Mattern(2012)]{Mattern12}
Christopher Mattern.
\newblock Mixing strategies in data compression.
\newblock In \emph{2012 Data Compression Conference, Snowbird, UT, USA, April
  10-12}, pages 337--346, 2012.

\bibitem[Mattern(2013)]{Mattern13}
Christopher Mattern.
\newblock Linear and geometric mixtures - analysis.
\newblock In \emph{2013 Data Compression Conference, {DCC} 2013, Snowbird, UT,
  USA, March 20-22, 2013}, pages 301--310, 2013.

\bibitem[Mattern(2016)]{Mattern16}
Christopher Mattern.
\newblock \emph{On Statistical Data Compression}.
\newblock PhD thesis, Technische Universit{\"{a}}t Ilmenau, Germany, 2016.

\bibitem[Meyn and Tweedie(2012)]{MT12}
Sean~P Meyn and Richard~L Tweedie.
\newblock \emph{Markov chains and stochastic stability}.
\newblock Springer Science \& Business Media, 2012.

\bibitem[Minsky and Papert(1969)]{minsky69perceptrons}
Marvin Minsky and Seymour Papert.
\newblock \emph{Perceptrons: An Introduction to Computational Geometry}.
\newblock MIT Press, Cambridge, MA, USA, 1969.

\bibitem[Ostrovski et~al.(2017)Ostrovski, Bellemare, van~den Oord, and
  Munos]{OstrovskiBOM17}
Georg Ostrovski, Marc~G. Bellemare, A{\"{a}}ron van~den Oord, and R{\'{e}}mi
  Munos.
\newblock Count-based exploration with neural density models.
\newblock In \emph{Proceedings of the 34th International Conference on Machine
  Learning, {ICML} 2017, Sydney, NSW, Australia, 6-11 August 2017}, pages
  2721--2730, 2017.

\bibitem[Rumelhart et~al.(1988)Rumelhart, Hinton, and Williams]{Rumelhart1988}
David~E. Rumelhart, Geoffrey~E. Hinton, and Ronald~J. Williams.
\newblock Neurocomputing: Foundations of research.
\newblock chapter Learning Representations by Back-propagating Errors, pages
  696--699. MIT Press, Cambridge, MA, USA, 1988.
\newblock ISBN 0-262-01097-6.

\bibitem[Saxe et~al.(2013)Saxe, McClelland, and Ganguli]{SaxeMG13}
Andrew~M. Saxe, James~L. McClelland, and Surya Ganguli.
\newblock Exact solutions to the nonlinear dynamics of learning in deep linear
  neural networks.
\newblock \emph{CoRR}, abs/1312.6120, 2013.

\bibitem[Schmidhuber and Heil(1996)]{schmidhuber96}
J.~Schmidhuber and S.~Heil.
\newblock Sequential neural text compression.
\newblock \emph{IEEE Transactions on Neural Networks}, 7\penalty0 (1):\penalty0
  142--146, Jan 1996.
\newblock ISSN 1045-9227.
\newblock \doi{10.1109/72.478398}.

\bibitem[Van Den~Oord et~al.(2016)Van Den~Oord, Kalchbrenner, and
  Kavukcuoglu]{VanDenOord2016}
A\"{a}ron Van Den~Oord, Nal Kalchbrenner, and Koray Kavukcuoglu.
\newblock Pixel recurrent neural networks.
\newblock In \emph{Proceedings of the 33rd International Conference on
  International Conference on Machine Learning - Volume 48}, ICML'16, pages
  1747--1756. JMLR.org, 2016.

\bibitem[van Erven et~al.(2007)van Erven, Grunwald, and de~Rooij]{ErvenGR07}
Tim van Erven, Peter Grunwald, and Steven de~Rooij.
\newblock Catching up faster in bayesian model selection and model averaging.
\newblock In \emph{Advances in Neural Information Processing Systems 20,
  Proceedings of the Twenty-First Annual Conference on Neural Information
  Processing Systems, Vancouver, British Columbia, Canada, December 3-6, 2007},
  pages 417--424, 2007.

\bibitem[Veness et~al.(2012)Veness, Ng, Hutter, and Bowling]{VSHB12}
Joel Veness, Kee~Siong Ng, Marcus Hutter, and Michael~H. Bowling.
\newblock Context tree switching.
\newblock In \emph{2012 Data Compression Conference, Snowbird, UT, USA, April
  10-12, 2012}, pages 327--336, 2012.

\bibitem[Veness et~al.(2015)Veness, Bellemare, Hutter, Chua, and
  Desjardins]{VenessBHCD15}
Joel Veness, Marc~G. Bellemare, Marcus Hutter, Alvin Chua, and Guillaume
  Desjardins.
\newblock Compress and control.
\newblock In \emph{Proceedings of the Twenty-Ninth {AAAI} Conference on
  Artificial Intelligence, January 25-30, 2015, Austin, Texas, {USA.}}, pages
  3016--3023, 2015.

\bibitem[Willems et~al.(1997)Willems, Shtarkov, and Tjalkens]{Willems97}
Frans Willems, Yuri Shtarkov, and Tjalling Tjalkens.
\newblock Reflections on "the context-tree weighting method: Basic properties".
\newblock In \emph{IEEE Information Theory Society Newsletter}, volume~47,
  1997.

\bibitem[Witten et~al.(1987)Witten, Neal, and Cleary]{Witten1987}
Ian~H. Witten, Radford~M. Neal, and John~G. Cleary.
\newblock Arithmetic coding for data compression.
\newblock \emph{Communications of the ACM}, 30\penalty0 (6):\penalty0 520--540,
  June 1987.
\newblock ISSN 0001-0782.
\newblock \doi{10.1145/214762.214771}.

\bibitem[Witten et~al.(1999)Witten, Moffat, and Bell]{Witten1999}
Ian~H. Witten, Alistair Moffat, and Timothy~C. Bell.
\newblock \emph{Managing Gigabytes (2nd Ed.): Compressing and Indexing
  Documents and Images}.
\newblock Morgan Kaufmann Publishers Inc., San Francisco, CA, USA, 1999.
\newblock ISBN 1-55860-570-3.

\bibitem[Zinkevich(2003)]{zinkevich03}
Martin Zinkevich.
\newblock Online convex programming and generalized infinitesimal gradient
  ascent.
\newblock In \emph{Machine Learning, Proceedings of the Twentieth International
  Conference {(ICML} 2003), August 21-24, 2003, Washington, DC, {USA}}, pages
  928--936, 2003.

\end{thebibliography}

\appendix

\section{Proof of Proposition \ref{prop:loss-prop}}
\label{app:loss-prop}

Part (1):
If $x_t=1$ then $\ell^\geo_t(w) = -\log \sigma(\logit(p_t) \cdot w)$ from Equation \ref{eq:geomix2}, hence using $\sigma'(x) = \sigma(x) [1 - \sigma(x)]$ and Equation \ref{eq:geomix2} once more gives
\begin{equation}\label{eq:loss_pdiff1}
\frac{\partial \ell^\geo_t}{\partial w_i} = [ \sigma(\logit(p_t) \cdot w) - 1 ] \logit(p_{t,i}) = (\geo_w(1 \; ; \; p_t) - x_t ) \logit(p_{t,i}).
\end{equation}
Similarly, if $x_t=0$ then $\ell^\geo_t(w) = -\log \left( 1 - \sigma\left(\logit(p_t) \cdot w\right) \right)$, and so 
\begin{equation}\label{eq:loss_pdiff2}
\frac{\partial \ell^\geo_t}{\partial w_i} =  \sigma(\logit(p_t) \cdot w) \logit(p_{t,i}) = (\geo_w(1 \; ; \; p_t) - x_t ) \logit(p_{t,i}).
\end{equation}
Hence $\nabla \ell^\geo_t(w) = (\geo_w(1 \; ; \; p_t) - x_t ) \logit(p_t)$.
\\

\noindent Part (2):
As $| (\geo_w(1 \; ; \; p_t) - x_t ) | \leq 1$ it holds that $\norm{\nabla \ell^\geo_t(w)}_2 \leq \norm{\logit(p)}_2$ for all $w \in \cC$.
\\

\noindent Part (3): the convexity of $\ell^\geo_t$ has been established already by \cite{Mattern13}.
\\

\noindent 
Part (4a): 
We make use of Lemma 4.1 in \cite{Hazan16}, which states that a twice differentiable function $ f : \mathbb{R}^m \to \mathbb{R}$ is $\alpha$-exp-concave at $x \in \mathbb{R}^m$ if and only if there exists a scalar $\alpha > 0$ such that $\nabla^2 f(x) - \alpha \nabla f(x) \nabla f(x)^\top$ is positive semi-definite.

We can calculate the Hessian of the loss directly from Equations~\ref{eq:loss_pdiff1} and \ref{eq:loss_pdiff2} by observing that
\begin{equation*}
\frac{\partial^2 \ell^\geo_t}{\partial w_j \, \partial w_i} = 
\logit(p_{t,i}) \logit(p_{t,j}) \, \sigma(\logit(p_t) \cdot w) [ 1 - \sigma(\logit(p_t) \cdot w)],
\end{equation*}
and so
\vspace{-1em}
\begin{align}
\label{eq:loss-lemma-hessian}
\nabla^2 \ell^\geo_t(w) &= 
\sigma(\logit(p_t) \cdot w) [ 1 - \sigma(\logit(p_t) \cdot w)] \, \logit(p_{t}) \logit(p_{t})^\top \\
&= \geo_w(1 \; ; \; p_t) [1 - \geo_w(1 \; ; \; p_t)] \, \logit(p_{t}) \logit(p_{t})^\top. \notag 
\end{align}
Furthermore, from Part (1), we have
\begin{equation*}
\nabla \ell^\geo_t(w) \, \nabla \ell^\geo_t(w)^\top = (\geo_w(1 \; ; \; p_t) - x_t )^2 \logit(p_t) \logit(p_t)^\top.
\end{equation*}
Therefore, letting $q = \geo_w(1 \; ; \; p_t)$, we have that
\begin{align*}
A: &= \nabla^2 \ell^\geo_t(w) - \alpha \nabla \ell^\geo_t(w) \, \nabla \ell^\geo_t(w)^\top \\
&=  \logit(p_{t}) \logit(p_{t})^\top [q (1-q) - \alpha (q - x_t)^2 ]
\end{align*}
Now if we could show $k := q (1-q) - \alpha (q - x_t)^2 \geq 0$, then we would have that $A$ is positive-semidefinite as $x^\top A x = k \, x^\top \logit(p_t) \logit(p_t)^\top x = k \, (x \cdot \logit(p_t))^2 \geq 0$ for any non-zero real vector $x$.
Considering the two cases $x_t=0$ and $x_t=1$ separately, for $k \geq 0$ to hold, we must have either $\alpha \leq (1-q)/q$ and $\alpha \leq q / (1-q)$; these conditions can be met independently of $x_t$ by choosing $\alpha = \epsilon_0 / (1-\epsilon_0)$, where $\epsilon_0$ is the smallest possible value of $\geo_w(1 \; ; \; p_t)$ for any $p_t$ or $w$.
Given the assumption that $p_t \in [\epsilon, 1 - \epsilon]^d$ for some $\epsilon \in (0,1/2)$, we conclude that $\ell^\geo_t(w)$ is $\alpha$-exp-concave with 
\begin{align*}
\alpha = \min_{w \in \cW} \sigma( w \cdot \logit(\epsilon)) 
= \min_{w \in \cW} \sigma \left( \prod_{i=1}^m \log \left( \frac{\epsilon}{1 - \epsilon} \right)^{w_i} \right)
= \sigma \left( \log\left( \frac{\epsilon}{1 - \epsilon} \right)^{\max_{w \in \cW} \norm{w}_1} \right)
.
\end{align*}
Also, notice that \cref{eq:loss-lemma-hessian} is clearly positive-semidefinite, confirming Part (3). 
\\

\noindent 
Part (4b): 
Applying Part (2), and using the assumption that $p_{t} \in [\epsilon, 1-\epsilon]^m$ we have
\begin{align*}
\norm{\nabla \ell(w)}_2 
&\leq \norm{\logit(p)}_2
= \sqrt{\sum_{i=1}^m \logit^2(p_{t,i})}
\leq \sqrt{m \log^2\left(\frac{1-\epsilon}{\epsilon}\right)} \\
&= \sqrt{m} \left( \log\left(1-\epsilon\right) - \log\left(\epsilon\right) \right)
\leq \sqrt{m} \log \left(\frac{1}{\epsilon} \right).
\end{align*}

\section{Proof of Theorem~\ref{thm:capacity}}
\label{sec:thm:capacity}

The proof of Theorem~\ref{thm:capacity} relies on several simple technical lemmas.
Recall for $p,q \in (0,1)$ that $\cD(p,q) = p \log (p/q) + (1-p) \log((1-p)/(1-q))$ is the relative entropy (or Kullback-Leibler divergence) between Bernoulli distributions with biases $p$ and $q$ respectively.

\begin{lemma}\label{lem:capacity-tech1}
Let $p \in [0,1]$ and $g_p(y) = \cD(p, \sigma(y))$, then:
\begin{enumerate}
\item $g_p'(\logit(q)) = q - p$.
\item $g_p''(y) = \exp(y) / (1 + \exp(y))^2 \in (0,1/4]$.
\item $g_p(\Delta + \logit(q)) \leq \cD(p,q) + \Delta (q-p) + \Delta^2 / 8$.
\end{enumerate}
\end{lemma}

\begin{proof}
The first and second parts are trivial.
The third follows from the fact that $g_p$ is everywhere twice differentiable for all $p \in [0,1]$, which implies that
\begin{align*}
g_p(x+\Delta) 
&\leq g_p(x) + \Delta g_p'(x) + \max_y g_p''(y) \frac{\Delta^2}{2} 
\leq g_p(x) + \Delta g_p'(x) + \frac{\Delta^2}{8}\,.
\end{align*}
Then choose $x = \logit(q)$ and note that $g_p(\logit(q)) = \cD(p, q)$.
\end{proof}

\begin{lemma}\label{lem:kl2}
Given $f,p,q \in (0,1)$,
\begin{align*}
\cD(f, q) - \cD(f, p) = \cD(p, q) + \frac{\partial}{\partial \alpha} \cD(f, \sigma((1 - \alpha) \logit(p) + \alpha \logit(q))) \Bigg|_{\alpha = 0}
\end{align*}
\end{lemma}

\begin{proof}
Note that 
\begin{align*}
\frac{\partial}{\partial \alpha} \cD(f, \sigma((1-\alpha) \logit(p) + \alpha \logit(q)))\Bigg|_{\alpha = 0} = (f - p) (\logit(q) - \logit(p))\,.
\end{align*}
The proof follows from the definition of $\cD(\cdot, \cdot)$.
\end{proof}

\begin{proof}[of Theorem \ref{thm:capacity}] 
\paragraph{Part (1)} 
Recall that for all layers $i$ (including $i = 0$) the output of the bias neuron in the $i$th layer is $p^*_{i0}(z) = \beta$.
Given non-bias neuron $(i,k)$ and context $a \in \cC$ let 
\begin{align*}
N_{ika}(n) = \{t \leq n : c_{ik}(z_t) = a\}
\end{align*}
be the set of rounds when the context for neuron $(i,k)$ is $a$.
Now define
\begin{align*}
w_{ika}^* &= \argmin_{w \in \cW_{ik}} \int_{c_{ik}^{-1}(a)} \cD(f(z), \sigma(w \cdot \logit(p^*_{i-1}(z)))) d\mu(z) \\
p^*_{ik}(z) &= \sigma(w_{ika}^* \cdot \logit(p^*_{i-1}(z)))\,.
\end{align*}
Note that $w_{ika}^*$ need not be unique, but the strict convexity of $\cD(f(z), \sigma(\cdot))$ for all $z$ ensures that
$p^*_{ik}$ does not depend on this choice.
We prove by induction that the following holds almost surely for all non-bias neurons $(i,k)$ 
\begin{align}
\lim_{n \to \infty} \frac{1}{n} \sum_{t=1}^n \sup_{z \in \supp(\mu)} \left|p_{ik}(z;w^{(t)}) - p_{ik}^*(z)\right| = 0\,.
\label{eq:capacity-induction}
\end{align}
Let $(i,k)$ be a non-bias neuron and $c = c_{ik}$ be its context function and assume the above holds for all preceding layers. 
Fix $a \in \cC$ and abbreviate
$N(n) = N_{ika}(n)$ and $q_t = \logit(p_{i-1}(z_t;w^{(t)}))$ be the input to the $i$th layer in the $t$th round $q^*(z) = \logit(p^*_{i-1}(z))$ and
\begin{align*}
\ell_t(w) 
&= x_t \log\left(\frac{1}{\sigma(w \cdot q_t))}\right) + (1 - x_t) \log\left(\frac{1}{1 - \sigma(w \cdot q_t)}\right)\,. \\
\ell_t^*(w)
&= x_t \log\left(\frac{1}{\sigma(w \cdot q^*(z_t))}\right) + (1 - x_t) \log\left(\frac{1}{1 - \sigma(w \cdot q^*(z_t))}\right)\,.
\end{align*}
Let $\ell^*(w) = \E[\ell_t^*(w) \mathds{1}_{N(n)}(t)]$, which does not depend on $t$ by the assumption that $(z_t,x_t)$ is independent and identically distributed. Furthermore,
\begin{align}
\ell^*(w) = \int_{c^{-1}(a)} \cD(f(z), \sigma(w\cdot q^*(z))) d\mu(z) + C\,,
\label{eq:ells}
\end{align}
where $C$ is a constant depending only on $\mu$ and $f$ and given by:
\begin{align*}
C = -\int_{c^{-1}(a)} \left(f(z) \log f(z) + (1 - f(z)) \log (1 - f(z))\right) d\mu(z)\,.
\end{align*}
Since $\cW_{ik}$ is bounded it follows that $q_t$ is bounded.
Therefore $\ell_t(u)$ is bounded and continuous in $u$ and $q_t$, which by the induction hypothesis means that 
\begin{align*}
\lim_{n\to\infty} \frac{1}{n} \sum_{t \in N(n)} \sup_{u \in \cW_{ik}} |\ell_t(u) - \ell^*_t(u)| = 0\,. 
\end{align*}
Since the weights are learned using an algorithm with sublinear regret it follows that
\begin{align*}
0 
\geq \lim_{n \to \infty} \min_{u \in \cW_{ik}} \frac{1}{n} \sum_{t\in N(n)}\left( \ell_t(w^{(t)}) - \ell_t(u)\right)   
= \lim_{n \to \infty} \min_{u \in \cW_{ik}} \frac{1}{n} \sum_{t\in N(n)}\left( \ell^*_t(w^{(t)}) - \ell^*_t(u)\right)\,.
\end{align*}
For each $u \in \cW_{ik}$ define martingale
\begin{align}
M_n(u) = \sum_{t \in N(n)} \left(\ell_t^*(w^{(t)}) - \ell_t^*(u) - (\ell^*(w^{(t)}) - \ell^*(u))\right)\,.
\label{eq:martingale}
\end{align}
By the martingale law of large numbers $\Prob{\lim_{n \to \infty} M_n(u) / n = 0} = 1$. 
Since $\ell_t^*(u)$ is uniformly Lipschitz in $u$ for all $t$, a union bound on a dense subset of $\cW_{ik}$ is enough to ensure
that with probability one it holds for all $u \in \cW_{ik}$ that
\begin{align*}
0 
&\geq \lim_{n \to \infty}\frac{1}{n} \sum_{t \in N(n)} \ell_t^*(w^{(t)})- \ell_t^*(u) 
= \lim_{n \to \infty}\frac{1}{n}\sum_{t \in N(n)} \left(\ell^*(w^{(t)}) - \ell^*(u)\right)\,.
\end{align*}
The conclusion follows from the definition of $\ell^*$ in \cref{eq:ells}.

\paragraph{Part (2)} We begin by noting that for any non-bias neuron $(i,k)$ it holds that
\begin{align*}
D_{ik} 
&= \int_{\R^d} \cD(f(z), p^*_{ik}(z)) d\mu(z) \\ 
&= \sum_{a \in \cC} \int_{c_{ik}^{-1}(a)} \cD(f(z), p^*_{ik}(z)) d\mu(z) \\
&= \sum_{a \in \cC} \min_{w \in \cW_{ik}} \int_{c_{ik}^{-1}(a)} \cD(f(z), \sigma(w \cdot \logit(p^*_{i-1}(z)))) d\mu(z) \\
&\leq \min_{w \in \cW_{ik}} \int_{\R^d} \cD(f(z), \sigma(w \cdot \logit(p^*_{i-1}(z)))) d\mu(z) \\
&\leq \min_{j} \int_{\R^d} \cD(f(z), p^*_{i-1,j}(z)) d\mu(z) 
= D_{i-1}^*\,,
\end{align*}
where the first and last equalities serve as definitions and in the last line we made use of the assumption that $e_j \in \cW_{ik}$ for each neuron $(i-1,j)$.
Therefore the expected loss of any neuron in the $i$th layer is smaller than that of the best neuron in the $(i-1)$th layer.
Since $p^*_{i0}(z) = \beta$ for all layers $i$ and $z \in \R^d$ we have
\begin{align}
\int_{\R^d} \cD(f(z), p^*_{i0}(z)) d\mu(z) 
&= \int_{\R^d} \cD(f(z), \beta) d\mu(z) 
\leq \cD(0, \beta) = -\log(\beta)\,.
\label{eq:D1s}
\end{align}
Next fix a layer $i$ and let $(i,k)$ and $(i-1,j)$ be neurons maximising  
\begin{align*}
B_i &= \sum_{a \in \cC} \left(\int_{c_{ik}^{-1}(a)} \left(f(z) - p^*_{i-1,j}(z)\right) d\mu(z)\right)^2 \in [0,1]\,. 
\end{align*}
Let $\Delta_a = \int_{c_{ik}^{-1}(a)} (f(z) - p^*_{i-1,j}(z)) d\mu(z)$ and $\tilde \Delta_a = \sign(\Delta_a) \min\{|\Delta_a|, \delta/4\}$.
Then by Lemma~\ref{lem:capacity-tech1} we have
\begin{align*}
D^*_i
&\leq D_{ik} 
=\sum_{a \in \cC} \min_{w \in \cW_{ik}} \int_{c_{ik}^{-1}(a)} \cD(f(z), \sigma(w \cdot p^*_{i-1}(z))) d\mu(z) \\
&\leq \sum_{a \in \cC} \int_{c_{ik}^{-1}(a)} \cD\left(f(z), \sigma\left(4\tilde \Delta_{a} + \logit(p^*_{i-1,j}(z))\right)\right) d\mu(z) \\
&\leq \sum_{a \in \cC} \int_{c_{ik}^{-1}(a)} \left(\cD(f(z), p^*_{i-1,k_{i-1}}(z)) + 4\tilde \Delta_{a} (p^*_{i-1,j}(z) - f(z)) + 2\tilde \Delta_{a}^2\right) d\mu(z) \\
&\leq \int_{\R^d} \cD(f(z), p^*_{i-1,j}(z)) d\mu(z) - 4 \sum_{a \in \cC} \Delta_a \tilde \Delta_a + 2 \sum_{a \in \cC} \tilde \Delta_{a}^2 \\
&\leq \int_{\R^d} \cD(f(z), p^*_{i-1,j}(z)) d\mu(z) - 2 \sum_{a \in \cC} \Delta_a \tilde \Delta_a \\ 
&\leq \int_{\R^d} \cD(f(z), p^*_{i-1,j}(z)) d\mu(z) - 2 \min\left\{1, \frac{\delta}{4}\right\} \sum_{a \in \cC} \Delta_a^2 \\ 
&= D_{i-1,j} - \min\left\{2, \frac{\delta}{2}\right\} B_i  
= D_{i-2}^* - \min\left\{2, \frac{\delta}{2}\right\} B_i\,. 
\end{align*}
Therefore by telescoping the sum over odd and even natural numbers and \cref{eq:D1s} we have
\begin{align*}
\sum_{i=1}^\infty B_i 
\leq \frac{2}{\min\left\{2, \frac{\delta}{2}\right\}} \log\left(\frac{1}{\beta}\right) = \max\left\{1, \frac{4}{\delta}\right\} \log\left(\frac{1}{\beta}\right)\,, 
\end{align*}
which implies the result.

\paragraph{Part (3)} 
Let $(i,k)$ and $(i-1,j)$ be two non-bias neurons and
abbreviate $p(z) = p^*_{ik}(z) \in \R$ and $q(z) = p_{i-1,j}^*(z) \in \R^{K_{i - 1}}$. 
Let
\begin{align*}
\epsilon = \int_{\R^d} \left(\cD(f(z), q(z)) - \cD(f(z), p(z))\right) d\mu(z)\,. 
\end{align*}
Then by Lemma~\ref{lem:kl2} and Pinsker's inequality 
\begin{align*}
\epsilon 
&= \int_{\R^d} \left(\cD(p(z),q(z)) + \frac{\partial}{\partial \alpha} \cD(f(z), \sigma((1 - \alpha) \logit(p(z)) + \alpha \logit(q(z))))\Bigg|_{\alpha=0}\right) d\mu(z) \\
&\geq \int_{\R^d} \cD(p(z), q(z)) d\mu(z) 
\geq 2 \int_{\R^d} (p(z) - q(z))^2 d\mu(z)\,,
\end{align*}
where the first inequality follows from the definition of $p = p^*_{ik}$, which ensures the derivative is nonnegative and the second follows from Pinsker's inequality.
Therefore for all layers $i$ and non-bias neurons $(i,k)$ and $(i-1,j)$ it holds that 
\begin{align*}
\int_{\R^d} \left(p^*_{ik}(z) - p^*_{i-1,j}(z)\right)^2 d\mu(z) 
&\leq \frac{1}{2} \int_{\R^d} \left(\cD(f(z), p^*_{i-1,j}(z)) - \cD(f(z), p^*_{i,k}(z))\right) d\mu(z) \\
&\leq \frac{1}{2} (D^*_{i-2} - D^*_i)\,.
\end{align*}
By telescoping the sum over odd and even $i$ there exists a $p^*_\infty$ such that 
\begin{align*}
\lim_{i\to \infty} \max_{k > 1} \int_{\R^d} (p^*_{ik}(z) - p^*_\infty(z))^2 d\mu(z) = 0\,.
\end{align*}
Since $p^*_{ik}$ is $\cF$-measurable for all neurons $(i,k)$ and pointwise convergence preserves measurability it follows that $p^*_\infty$ is also $\cF$-measurable.
\end{proof}

\section{Proof of Theorem~\ref{thm:capacity-markov}}
\label{sec:thm:capacity-markov}

We use the same notation as the proof of Theorem~\ref{thm:capacity} in the previous section.
Only the first part is different. Since the side information is no longer independent and identically distributed we cannot claim anymore 
that $\E[\ell_t^*(u)] = \ell^*(u)$. The idea is to use the stability of the learning procedure to partition the data into chunks of increasing size on which the
weights are slowly changing, but where the mixing of the Markov chain will eventually ensure that the empirical distribution converges to stationary.
Let $(S_i)_i$ be a random contiguous partition of $N(\infty) = \{t \in \N : c_{ik}(z_t) = a\}$, which means that $\max S_i < \min S_{i+1}$ for all $i$ and $\bigcup_{i=1}^\infty S_i = N(\infty)$. 
By the stability assumption we may choose $(S_i)_i$ such that $\lim_{i\to \infty} |S_i| = \infty$ and $\max_{s,t \in S_i} \norm{w^{(s)} - w^{(t)}}_\infty \leq \epsilon_i$ with
$\lim_{i\to\infty} \epsilon_i = 0$ for all possible data sequences. Then abusing notation by letting $w^{(i)} = w^{(\min S_i)}$ and letting $n_I = \sum_{i=1}^I |S_i|$ we have the following holding almost surely:
\begin{align*}
0 
&\geq \lim_{n\to\infty} \frac{1}{n} \sum_{t=1}^n \left(\ell_t^*(w^{(t)}) - \ell_t^*(u)\right) 
= \lim_{I\to\infty} \frac{1}{n_I} \sum_{i=1}^I \sum_{t \in S_i} \left(\ell_t^*(w^{(t)}) - \ell_t^*(u)\right) \\
&= \lim_{I\to\infty} \frac{1}{n_I} \sum_{i=1}^I \sum_{t \in S_i} \left(\ell_t^*(w^{(i)}) - \ell_t^*(u) + \ell_t^*(w^{(t)}) - \ell_t^*(w^{(i)})\right) \\
&\geq \lim_{I\to\infty} \frac{1}{n_I} \sum_{i=1}^I \sum_{t \in S_i} \left(\ell_t^*(w^{(i)}) - \ell_t^*(u) - \epsilon_i \right) \\
&= \lim_{I\to\infty} \frac{1}{n_I} \sum_{i=1}^I \sum_{t \in S_i} \left(\ell_t^*(w^{(i)}) - \ell_t^*(u)\right) 
= \lim_{I\to\infty} \frac{1}{n_I} \sum_{i=1}^I |S_i| \left(\ell^*(w^{(i)}) - \ell^*(u)\right) \\
&= \lim_{I\to\infty} \frac{1}{n_I} \sum_{i=1}^I \sum_{t \in S_i} \left(\ell^*(w^{(t)}) - \ell^*(u)\right) 
= \lim_{n\to\infty} \frac{1}{n} \sum_{t=1}^n \left(\ell^*(w^{(t)}) - \ell^*(u)\right)\,,
\end{align*}
where we used the convergence theorem for aperiodic $\phi$-irreducible Markov chains \cite[Chap.\ 10]{MT12}.
The proof is concluded in the same way as the proof of Theorem~\ref{thm:capacity}.

\section{Norms and topological arguments}
\label{app:topology}

Recall that $(\R^d, \cF, \mu)$ is a probability space with $\cF$ the Lebesgue $\sigma$-algebra and for $\cG$ a set of $\cF$-measurable functions from $\R^d$ to $\cC$ we define
\begin{align*}
\norm{h}_{\cG} = \sup_{c \in \cG} \sum_{a \in \cC} \left|\int_{c^{-1}(a)} h(x) d\mu(x)\right|\,.
\end{align*}
Of course $\norm{\cdot}_{\cG}$ is not a norm for all $\cG$. In the following we establish that it is a norm for a few natural choices.
First we note that $\norm{\cdot}_{\cG}$ is `almost' a norm for all $\cG$ as summarised in the following trivial lemma.

\begin{lemma}
Let $g,h: \R^d \to \R$ be $\mu$-integrable and $\cG \subset \cC \to \R^d$. Then
\begin{enumerate}
\item $\norm{g + h}_\cG \leq \norm{g}_\cG + \norm{h}_{\cG}$ 
\item $\norm{a g}_\cG = a \norm{g}_\cG$
\item $\norm{0}_{\cG} = 0$.
\end{enumerate}
\end{lemma}

Thus in order to show that $\norm{\cdot}_{\cG}$ is a norm it is only necessary to prove that $\norm{h}_{\cG} = 0$ implies that $h = 0$ $\mu$-almost-everywhere.
We are especially interested in countable $\cG$, since we can construct (infinite) networks that eventually use all contexts in such sets. The following lemma
allows us to connect Theorem~\ref{thm:capacity} to a claim about the asymptotic universality of a network.

\begin{lemma}
If $\cG =\{c_1,c_2,\ldots\}$ is countable and $\norm{\cdot}_{\cG}$ is a norm on the space of $\mu$-integrable $h$, then for $\cG_i = \{c_1,c_2,\ldots,c_i\}$ it holds that
$\lim_{i\to \infty} \norm{h}_{\cG_i} = \norm{h}_{\cG}$ for all $h$.
\end{lemma}

For the remainder we assume that $\mu$ is absolutely continuous with respect to the Lebesgue measure. 

\begin{lemma}
Let $\cG = \{\mathds{1}_{B_r(x)} : x \in \Q^d, r \in (0,\infty) \cap \Q\}$
and suppose that $h$ is bounded and measurable with $\norm{h}_{\cG} = 0$. Then $h = 0$ $\mu$-almost-everywhere.
\end{lemma}

\begin{proof}
Assume without loss of generality that $\sup_{x \in \R^d} |h(x)| \leq 1$.
Suppose that $|h|$ does not vanish $\mu$-almost everywhere. Then there exists an $\epsilon > 0$ and $\delta > 0$ such that 
$\mu(|h| > \epsilon) > \delta$. 
Let $A = \{x : |h(x)| > \epsilon\}$, which is measurable because $|h|$ is measurable.
Therefore by outer-regularity there exists an open $U \subseteq \R^d$ with $A \subset U$ and $\mu(U) \leq \mu(A) + \epsilon \delta / 2$. 
Therefore $\int_U h(x) d\mu(x) = \int_A h(x) d\mu(x) + \int_{U-A} h(x) d\mu(x) \geq \epsilon \delta / 2$.
Next write $U = \bigcup_i B_i$ where $(B_i)$ are disjoint open balls with rational-valued centers and radii. Then by assumption
$\int_U h(x) d\mu(x) = \sum_{i=1}^\infty \int_{B_i} h(x) d\mu(x) = 0$, which is a contradiction.
\end{proof}

\begin{lemma}
Let $\cB$ be a countable base for $\R^d$ and
$\cG = \{\mathds{1}_U : U \in \cB\}$ and $\norm{h}_{\cG} = 0$ for some bounded and $\cF$-measurable $h$. 
Then $h = 0$ $\mu$-almost-everywhere.
\end{lemma}

\begin{proof}
As above.
\end{proof}

\begin{lemma}\label{lem:half-space-norm}
Let $\cG = \{\mathds{1}_{H_{\nu,c}} : \nu \in \Q^d, \norm{\nu} = 1, c \in \Q\}$. If $h:\R^d \to \R$ is measurable and $\norm{h}_{\cG} = 0$, then $h = 0$ almost-everywhere.
\end{lemma}

\begin{proof}
Since $\Q$ is dense in $\R$ it follows from absolutely continuity of $\mu$ that if $\nu \in \R^d$ and $c \in \R$ 
and $\norm{h}_{\cG} = 0$, then $\int_{H} h d\mu = 0$ for all half-spaces $H$.
Therefore if $\lambda$ is the Lebesgue measure, then 
\begin{align*}
\hat h(\partial H) = \int_{\partial H} h \frac{d\mu}{d\lambda} d\lambda = 0
\end{align*}
for all hyperplanes $\partial H$. But $\hat h$ is the Radon transform of $h d\mu/d\lambda$, which is zero everywhere only if $h d\mu/d\lambda$ is zero almost everywhere
\citep{Hel11}. Therefore
$h \frac{d\mu}{d\lambda} \equiv 0$ and so $h \equiv 0$ almost everywhere.
\end{proof}

\section{Capacity for Half-Space Contexts with Two Layers}
\label{app:2-layer-1-dim}

We now argue that isolation networks can approximate continuous functions to arbitrary precision with just two layers where the
first layer is sufficiently wide and consists of half-space contexts and the second layer (output) has just one neuron and no context. 
For dimension one we give an explicit construction, while for higher dimensions we sketch the argument via the Radon transform inversion formula.
The interestingness of these results is tempered by the fact that the precision of the approximation can only be made arbitrarily small by 
having an arbitrarily wide network \textit{and} that the weights must also be permitted to be arbitrarily large. In practice we did not find two-layer
half-space contexts to be effective beyond the one-dimensional case.

\begin{theorem}
Let $f:\R \to [0,1]$ be continuous and $\mu$ have compact support and be absolutely continuous with respect to the Lebesgue measure. Then for any $\epsilon > 0$ there 
exists a two-layer network with half-space contexts in the first layer and a single non-bias
neuron in the second layer with a trivial context function such that $\norm{p_{21}^* - f}_\infty < \epsilon$, where $p_{21}^*(z)$ is given in Theorem~\ref{thm:capacity}.
\end{theorem}

\begin{proof}
By applying a mollifier, any continuous $f$ may be approximated to arbitrary precision by a differentiable function 
with range a subset of $(\epsilon, 1-\epsilon)$ for small $\epsilon > 0$.
From now on we simply assume that $f$ is differentiable.
For $b \in (0,1)$ let 
\begin{align*}
\ell(b) &= \logit\left(\frac{1}{\mu([b,1])}\int^1_b f(z) d\mu(z)\right) &
\bar \ell(b) &= \logit\left(\frac{1}{\mu([0,b])} \int^b_0 f(z) d\mu(z) \right)
\end{align*}
and $\ell = \int^1_0 f(z) d\mu(z)$. Because $f(z) \in (\epsilon, 1-\epsilon)$ it follows that $\ell(b)$ and $\bar\ell(b)$ are also bounded.
By the first part of Theorem~\ref{thm:capacity}, if neuron $(1,k)$ uses context function $c_{1k}(z) = \mathds{1}_{z \geq b_k}(z)$, then
\begin{align*}
\logit(p^*_{1k}(z)) = \begin{cases}
\ell(b_k) & \text{if } z \geq b_k \\
\bar \ell(b_k) & \text{otherwise}\,.
\end{cases}
\end{align*}
Since neuron $(2,1)$ has a trivial context function there exists an $a \in \cC$ such that $c_{21}(z) = a$ for all $z \in \R$.
Abbreviating $w_k = w_{21ak}$ we can write the output of neuron $(2,1)$ as
\begin{align}
p^*_{21}(z) 
&= \sigma\left(\sum_{k=0}^{K_1} w_k \logit(p^*_{1k}(z))\right) \nonumber \\
&= \sigma\left(w_0 + \sum_{k : b_k \leq z} w_k \ell(b_k) + \sum_{k: b_k > z} w_k \bar\ell(b_k)\right) \nonumber \\
&= \sigma\left(w_0 + \sum_{k : b_k \leq z} w_k (\ell(b_k) - \bar \ell(b_k)) + \sum_k w_k \bar\ell(b_k)\right)\,.\label{eq:finite-sum}
\end{align}
The main idea is to write an approximation of $f$ as an integral that in turn may be approximated by the sum above for sufficiently large $K_1$ and well chosen
$w_k$. Define
\begin{align*}
v(x) = \frac{f'(x)}{f(x)(1-f(x))} \left(\frac{1}{\ell(x) - \bar \ell(x)}\right)\,.
\end{align*}
By differentiating the logit function and the fundamental theorem of calculus and the assumption that $f$ is differentiable and $f(z) \in (\epsilon, 1-\epsilon)$ we have
\begin{align*}
f(z) 
&= \sigma\left(\logit(f(0)) + \int^z_0 \frac{f'(x)}{f(x)(1-f(x))}dx \right) \\
&= \sigma\left(\logit(f(0)) + \int^z_0 v(x)(\ell(x) - \bar \ell(x)) dx \right)\,.
\end{align*}
The result follows by approximating the integral above with the sum in \cref{eq:finite-sum}.
\end{proof}

The situation for higher dimensions is more complicated, but the same theorem may be proven for $d > 1$ by appealing 
to the Radon transform inversion formula \citep[Chap 1]{Hel11},
which says that any infinitely differentiable and compactly supported function  $g:\R^d \to \R$ may be represented by
\begin{align*}
g(z) = \int_S \tilde g(\eta, \ip{z, \eta}) d\eta\,,
\end{align*}
where $S = \{x \in \R^d : \norm{x}_2 = 1\}$ is the sphere and the integral is with respect to the uniform measure and $\tilde g:S \times \R \to \R$
is carefully chosen. The precise form of $\tilde g$ is quite complicated, but only depends on the Radon transform of $g$, which is the operator
$(Rg): S \times \R \to \R$ given by $(Rg)(\eta, x) = \int_{\partial H_{\eta,x}} f(z) dz$. This result means that sufficiently regular functions are entirely
determined by their integrals on hyperplanes.
It remains to note that for wide two-layer networks with arbitrarily large weights and half-space contexts, the output of neuron $(2,1)$ can
approximate the above integral to arbitrary precision. We leave the details for the interested reader.

\section{Proof of Equation \ref{eq:switching_prior_bound}}
\label{app:switch_prior_proof}

This result is the same as that of Lemma 1 in \citep{VSHB12}.
We provide the proof here for the sake of completeness. 

\begin{lemma}
If $s(\nu_{1:n}) = \sum_{k=2}^n \mathds{1}[\nu_k \neq \nu_{k-1}]$, then for all $\nu_{1:n} \in \mathcal{I}_n(\mathcal{M})$, we have
\begin{equation*}
-\log w_{\tau}(\nu_{1:n}) \leq \left( s(\nu_{1:n})+1 \right) \left( \log |\cM| + \log_2 n \right).
\end{equation*}

\end{lemma}  
\begin{proof}
Consider an arbitrary $\nu_{1:n} \in \mathcal{I}_n(\mathcal{M})$.
Letting $m = s(\nu_{1:n})$, from Equation \ref{eq:switch_prior_rec} we have that
\begin{eqnarray*}
-\log w_{\tau}(\nu_{1:n}) &=& \log|\cM| -\log  \prod\limits_{t=2}^n \left( \tfrac{t-1}{t} \mathds{1}[\nu_t = \nu_{t-1}] + \tfrac{1}{t (|\cM|-1)} \mathds{1}[\nu_t \neq \nu_{t-1}] \right) \\
&\leq& \log|\cM| -\log  \prod\limits_{t=2}^n \left( \tfrac{t-1}{t} \mathds{1}[\nu_t = \nu_{t-1}] + \tfrac{1}{n (|\cM|-1)} \mathds{1}[\nu_t \neq \nu_{t-1}] \right)\\
&\leq& \log|\cM| -\log \left(  n^{-m} (|\cM|-1)^{-m} \prod\limits_{t=2}^{n-m} \tfrac{t-1}{t} \right) \\
&=& \log|\cM| + m \log n + m \log (|\cM|-1) + \log(n-m) \\
&\leq& (m+1) [\log|\cM| + \log n].
\end{eqnarray*}
\end{proof}

\end{document}